\documentclass{article}

\usepackage[main, preprint]{neurips_2026}

\usepackage[utf8]{inputenc}
\usepackage[T1]{fontenc}
\usepackage{hyperref}
\hypersetup{colorlinks=true, citecolor=blue, linkcolor=blue, urlcolor=blue}       %
\usepackage{url}
\usepackage{booktabs}
\usepackage{amsfonts}
\usepackage{nicefrac}
\usepackage{microtype}
\usepackage{xcolor}

\usepackage[table]{xcolor}
\usepackage[english]{babel}

\usepackage{amsmath,amssymb}
\usepackage{bbm}
\DeclareMathOperator*{\argmax}{arg\,max}

\usepackage{graphicx}
\usepackage{tikz}
\usetikzlibrary{arrows.meta,positioning,calc,fit,backgrounds,decorations.pathreplacing,decorations.pathmorphing}

\usepackage{multirow}
\usepackage{tabularx}

\usepackage{algorithm}
\usepackage{algorithmic}

\usepackage{enumitem}
\usepackage{tcolorbox}
\tcbuselibrary{breakable, skins, theorems}
\usepackage{titletoc}
\usepackage{placeins}

\usepackage{amsthm}
\newtheorem{proposition}{Proposition}

\newtheorem{assumption}{Assumption}
\newtheorem{theorem}{Theorem}

\newenvironment{maintheorem}[1][]{%
  \begin{tcolorbox}[
    colback = blue!5,
    colframe = white,
    fonttitle = \bfseries,
    breakable = true]
  \begin{theorem}[#1]%
}{%
  \end{theorem}%
  \end{tcolorbox}
  \vspace{5pt}
}
\newtcolorbox{interpretation}{
  enhanced,
  colback = blue!2,
  colframe = blue!15,
  boxrule = 0.4pt,
  arc = 2pt,
  breakable = true,
  left = 6pt, right = 6pt, top = 4pt, bottom = 4pt,
  before upper = {\textit{\textcolor{blue!50!black}{Interpretation.}}\quad},
}

\newcommand{\projecturl}{https://naoto-iwase.github.io/prefix-consistency-page}
\newcommand{\codeurl}{https://github.com/naoto-iwase/prefix-consistency}
\newcommand{\dataurl}{https://doi.org/10.5281/zenodo.20082164}

\title{Reliable Chain-of-Thought via Prefix Consistency}

\author{%
  Naoto Iwase\textsuperscript{1} \quad
  Yuki Ichihara\textsuperscript{2,3} \quad
  Mohammad Atif Quamar\textsuperscript{3} \quad
  Junpei Komiyama\textsuperscript{3,4} \\[2pt]
  \textsuperscript{1}Nagoya University \quad
  \textsuperscript{2}Nara Institute of Science and Technology \\[1pt]
  \textsuperscript{3}Mohamed bin Zayed University of Artificial Intelligence \quad
  \textsuperscript{4}RIKEN AIP \\[2pt]
  \normalfont\small\texttt{naoto@iwase.dev} \quad
  \texttt{\{yuki.ichihara, mohammad.atif\}@mbzuai.ac.ae} \quad
  \texttt{junpei@komiyama.info} \\[2pt]
  {\textbf{Project Page:} \url{\projecturl}}
}

\begin{document}
\maketitle

\begin{abstract}
Large language models often improve accuracy on reasoning tasks by sampling multiple chain-of-thought (CoT) traces and aggregating them with majority voting, a test-time technique called self-consistency. When we truncate a CoT partway through and regenerate the remainder, we observe that traces with correct answers reproduce their original answer more often than traces with wrong answers. We use this difference as a reliability signal, \textbf{prefix consistency}, that weights each candidate answer by how often it reappears under regeneration. It requires no access to token log-probabilities or self-rating prompts. Across five reasoning models and four math and science benchmarks, prefix consistency is the best correctness predictor in most settings, and reweighting votes by this signal reaches the plateau accuracy of majority voting with up to $21\times$ fewer tokens (median $4.6\times$).\makeatletter\if@anonymous\else{} Our code is available at \url{\codeurl}.\fi\makeatother
\end{abstract}

\begin{figure*}[!ht]
\centering
\resizebox{0.97\textwidth}{!}{%
  \includegraphics[trim=0 6bp 0 7bp,clip]{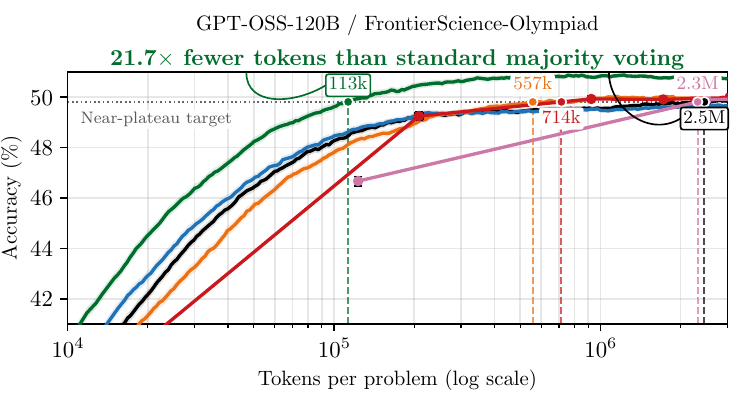}%
  \includegraphics[trim=0 6bp 0 7bp,clip]{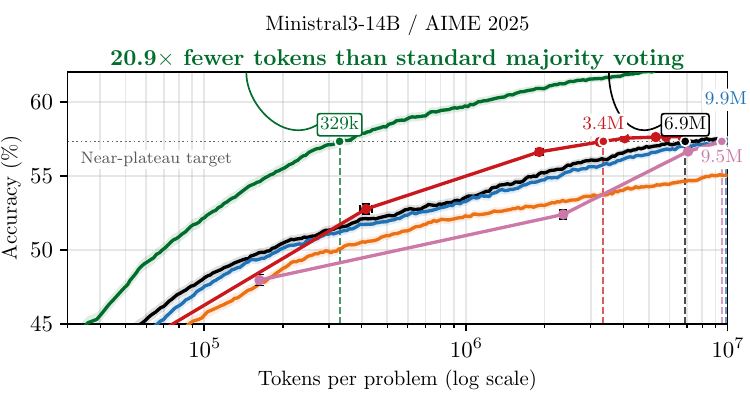}%
}\par
\vspace{6pt}
\includegraphics[width=\textwidth,trim=0 4bp 0 4bp,clip]{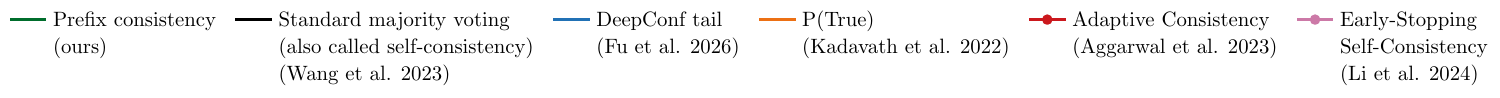}\par
\vspace{6pt}

\begin{minipage}[h]{\textwidth}
\centerline{\resizebox{\linewidth}{!}{\includegraphics[alt={A shared prefix branches into four regenerated solutions. All four reproduce the original answer in the consistent example, but only one does in the inconsistent example.}]{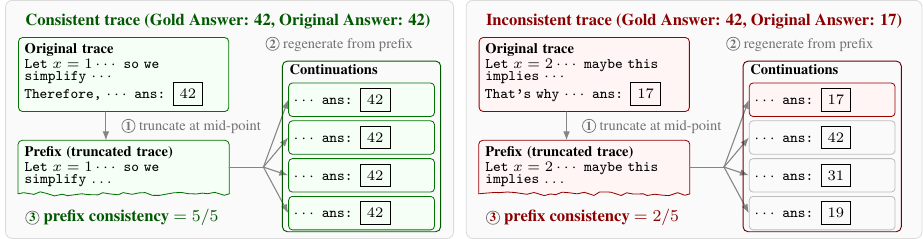}}}
\end{minipage}%

\caption{\textbf{Overview of prefix consistency.}
\emph{Top}: cost-equivalent accuracy of majority voting weighted by prefix consistency and baseline methods on two models and benchmark settings. Shaded bands and error bars show $\pm 2\times$ the Monte Carlo standard error. \emph{Bottom}: Overview of our proposed method (prefix consistency).
Correct reasoning traces exhibit greater reproducibility under regeneration.
}
\label{fig:regen-consistency-main}
\end{figure*}

\section{Introduction}
\label{sec:intro}
Large language models (LLMs) have shown strong reasoning ability when allowed to produce chain-of-thought (CoT) reasoning \citep{kojima2022zeroshot,wei2022chainofthought}.
Generating intermediate reasoning steps substantially improves performance on challenging tasks such as math \citep{zhou2023leasttomost,fu2023complexitybased}, scientific reasoning \citep{lu2022learntoexplain,wang2024tsciq}, and knowledge-intensive question answering \citep{trivedi-etal-2023-interleaving,wang2023knowledge}.

A simple and effective way to further improve the accuracy of the final answer is majority voting (MV, also known as self-consistency), which samples a diverse set of CoTs and returns the most frequent answer~\citep{wang2023selfconsistency}.
A limitation is that Standard MV treats all CoT outputs equally and still fails when the correct answer is in the minority.

To improve MV, the standard approach has been to use weighted majority voting (WMV).
WMV refines MV by weighting each generation according to its quality. The more reliable a generation is, the greater the signal it receives.
Existing WMV methods derive a per-sample reliability signal from the generated trace, including response probability~\citep{wang2023selfconsistency}, self-certainty~\citep{kang2025selfcertainty}, DeepConf~\citep{fu2025deepthinkconfidence}, verbalized confidence elicited in text~\citep{lin2022teaching,taubenfeld2025confidenceimprovesselfconsistency}, and P(True)~\citep{kadavath2022languagemodelsknowthey}.
Previous studies have demonstrated that these reliability-aware aggregation methods outperform Standard MV.
However, these signals often fail to separate correct from wrong traces on difficult problems, the regime where Standard MV most needs improvement (Figure~\ref{fig:signal_distributions}).

We introduce a novel reliability signal, \textbf{prefix consistency}, and incorporate it into WMV. This signal is motivated by the observation that correct reasoning traces tend to be more reproducible under regeneration than wrong ones.
We truncate each sample's CoT at a specified fraction and regenerate continuations from the prefix (Figure~\ref{fig:regen-consistency-main}).
Prefix consistency requires no access to token log-probabilities. Since regenerated answers also participate in voting, our method recovers correct answers absent from the initial samples.

Our contributions are:
\begin{enumerate}[leftmargin=15pt]
    \item We propose \emph{prefix consistency}, a reliability signal that truncates each sample's CoT and regenerates from the prefix, and use it to form prefix-consistency-weighted majority voting (PC-WMV). PC-WMV requires no access to token log-probabilities.
    \item Across 4 benchmarks and 5 model scales, prefix consistency outperforms existing WMV baselines (e.g.\ DeepConf, P(True), Self-certainty) as a correctness predictor (best macro-averaged AUROC on 15 out of 20 (model, benchmark) cells, .63--.80). On many problems with Pass@1 below 50\%, where Standard MV fails by default, prefix consistency still discriminates correct from wrong traces (Section~\ref{sec:what_drives}), leaving room for PC-WMV to find the correct answer.
    \item In cost-equivalent comparison against the primary WMV baselines (DeepConf tail, P(True), Self-certainty) and adaptive-stopping baselines (AC, ESC), PC-WMV is the most cost-efficient on the majority of the 20 (model, benchmark) settings, reaching Standard MV plateau at a median $4.6\times$ fewer tokens (up to $21\times$ vs.\ Standard MV and $10\times$ vs.\ AC sweep, Figure~\ref{fig:regen-consistency-main} and Table~\ref{tab:token_savings}).
\end{enumerate}

\section{Preliminary}
\label{sec:prelim}

We consider a benchmark $\mathcal{Q}$, a set of problems. For each problem $q \in \mathcal{Q}$, we have an answer space $\mathcal{A}$ and a correct answer $a^\star_q \in \mathcal{A}$.
Given $q$, an LLM generates a trace $y$, i.e., a sequence of tokens that represents a CoT followed by a final summary, from which we parse the final answer $a \in \mathcal{A}$. We write $(y, a) \sim \mathrm{LLM}(\cdot \mid q)$.
We write $\mathrm{Pass@1}_q := \Pr_{(y, a) \sim \mathrm{LLM}(\cdot \mid q)}[a = a^{\star}_{q}]$ for the per-problem single-sample success probability, and report the macro-average over $\mathcal{Q}$ as the benchmark-level Pass@1.
When the context is clear, we suppress $q$ in the notation (e.g., $a^{\star}$ instead of $a^{\star}_{q}$).

To improve the accuracy over Pass@1 at test time, we draw $N$ independent samples $\{(y_i, a_i)\}_{i=1}^{N}$ and aggregate them into a single output.
A standard method that aggregates the $N$ answers is \emph{majority voting} (MV, also known as self-consistency), which returns the most frequent answer:
\begin{equation}
\label{eq:mv}
    \hat a^{\mathrm{MV}}_N = \argmax_{a \in \mathcal{A}} \sum_{i=1}^{N} \mathbf{1}[a_i = a].
\end{equation}

We refer to this unweighted aggregator as \emph{Standard MV} (Eq.~\eqref{eq:mv}).
Standard MV treats all samples equally and fails when the correct answer is not the mode of the answer distribution, typically observed when an LLM faces challenging problems where Pass@1 accuracy is below 50\%.
A natural extension is \emph{weighted majority voting} (WMV), where each sample $i$ contributes a weighted vote $v_i(a) \geq 0$ for answer $a$:
\begin{equation}
\label{eq:wmv}
    \hat a^{\mathrm{WMV}}_N = \argmax_{a \in \mathcal{A}} \sum_{i=1}^{N} v_i(a).
\end{equation}

For sample $i$, let $\ell_i$ denote the model's token-level log-probabilities available to the signal. Prior WMV methods extract a confidence signal $s(y_i, \ell_i) \geq 0$ from the trace and apply a weighting function $w \colon \mathbb{R}_{\geq 0} \to \mathbb{R}_{\geq 0}$ to it:
\begin{equation}
\label{eq:wmv_prior}
    v_i(a) = w(s(y_i, \ell_i)) \cdot \mathbf{1}[a_i = a].
\end{equation}
Another class of WMV methods adopts verbalized signals that require no log-probability access ($\ell_i = \emptyset$), where $s$ depends only on the text of $y_i$.
Other methods (Self-certainty, DeepConf, Response probability) set $\ell_i$ to the per-token log-probabilities along $y_i$.
P(True) sets $\ell_i$ to the log-probability of the ``True'' token under a self-rating prompt.
Appendix~\ref{app:baseline_impl} gives the explicit form of $s$ and $w$ for each baseline.
However, such confidence signals often fail to separate correct traces from wrong traces on difficult problems.
We next introduce prefix consistency, a signal that requires no log-probability access ($\ell_i = \emptyset$).

\section{Prefix Consistency}
\label{sec:method}

\begin{figure}[t]
\centering
\includegraphics[width=0.32\textwidth]{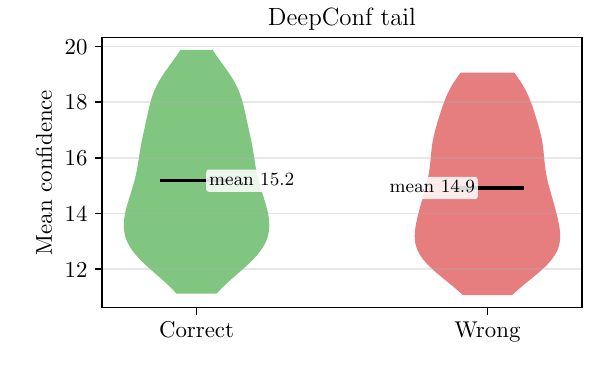}\hfill
\includegraphics[width=0.32\textwidth]{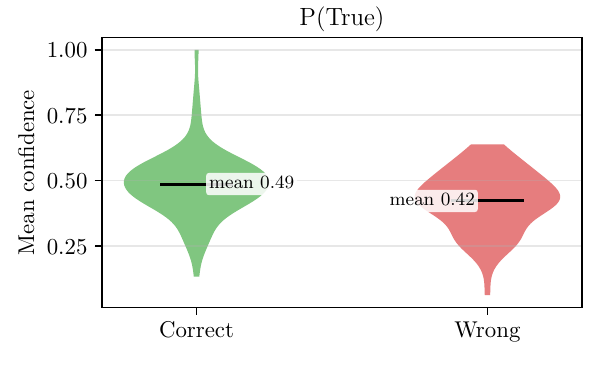}\hfill
\includegraphics[width=0.32\textwidth]{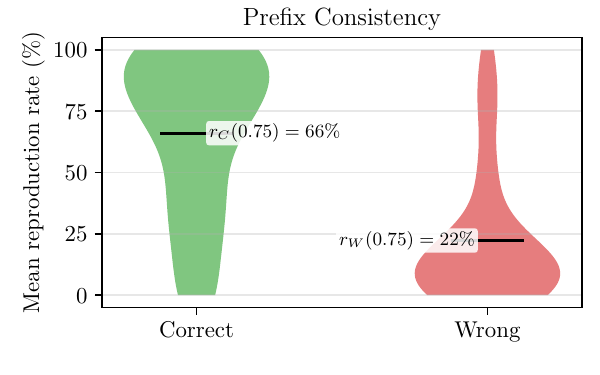}
\caption{\textbf{Per-problem reliability signal distributions for correct vs.\ wrong traces} on FrontierScience-Olympiad with GPT-OSS-120B. Each violin shows the distribution of per-problem mean signal values across the 73 problems that have at least one correct and one wrong trace. On this difficult benchmark (Pass@1 ${\approx}.34$), the baseline methods DeepConf tail and P(True) place correct and wrong traces in nearly overlapping ranges, while prefix consistency separates them. The same baselines have some discriminative ability on easier (model, benchmark) settings (Table~\ref{tab:auroc}). The point is that they fail precisely in the regime where reweighting matters most.}
\label{fig:signal_distributions}
\end{figure}

We propose prefix consistency, a reliability signal: truncate each sample's CoT at an intermediate point and regenerate its continuation, treating samples whose initial answer reappears as more reliable than those whose answer changes (Figure~\ref{fig:regen-consistency-main}).

\subsection{Prefix Consistency as a Reliability Signal}
\label{sec:prefix_consistency}

For each sample $(y_i, a_i)$, let $|y_i|$ denote the number of tokens in $y_i$.
We truncate $y_i$ after its first $\lceil \tau |y_i| \rceil$ tokens for a fixed fraction
$\tau \in (0,1)$, and regenerate $K$ continuations from this prefix, yielding the multiset:
\begin{equation}
\label{eq:multiset}
    A_i^{(\tau,K)} = \{a_i,\, \tilde a_{i,1}^{(\tau)}, \ldots, \tilde a_{i,K}^{(\tau)}\}.
\end{equation}
We refer to $A_i^{(\tau,K)}$ as the $i$-th \emph{group}.
For the following discussion, we focus on the case when $K=1$. We will write $A_i^{(\tau)}$ for $A_i^{(\tau,1)}$ and write $\tilde a_i^{(\tau)}$ for $\tilde a_{i,1}^{(\tau)}$. Extending this to arbitrary $K$ is straightforward.

The key empirical phenomenon is a \emph{reproduction-rate asymmetry}: a regenerated answer is more likely to match the initial answer when the initial answer is correct. Let $r_C(\tau)$ and $r_W(\tau)$ denote the probabilities of reproducing the initial answer, conditioned on whether it is correct or wrong:
\begin{equation}
\label{eq:rates_def}
    r_C(\tau) := \Pr[\tilde a_i^{(\tau)} = a_i \mid a_i = a^\star], \qquad
    r_W(\tau) := \Pr[\tilde a_i^{(\tau)} = a_i \mid a_i \neq a^\star]
\end{equation}
Across models and benchmarks, we consistently observe the following inequality (Table~\ref{tab:signal}):
\begin{equation}
\label{eq:asymmetry}
    r_C(\tau) > r_W(\tau).
\end{equation}
In other words, a correct initial answer is more likely to be reproduced under regeneration than a wrong initial answer. Figure~\ref{fig:signal_distributions} illustrates this asymmetry on FrontierScience-Olympiad with GPT-OSS-120B and contrasts it with two baseline signals (DeepConf tail and P(True)) that fail to distinguish between correct and wrong traces.

To exploit this observation, we score each candidate $a \in A_i^{(\tau)}$ by its reproducibility within the group:
\begin{equation}
\label{eq:ci_def}
    c_i^{(\tau)}(a) := \frac{|\{a' \in A_i^{(\tau)} : a' = a\}|}{2}.
\end{equation}
We denote the \emph{prefix consistency score} of $a$ in group $i$ by $c_i^{(\tau)}(a) \in \{0, 1/2, 1\}$.
Unlike conventional per-sample reliability signals (e.g., DeepConf tail, P(True)), which assign a single scalar to each sample's initial answer, $c_i^{(\tau)}(a)$ is defined for every candidate $a \in A_i^{(\tau)}$, including regenerated answers that did not appear among the initial answers.
\subsection{Prefix-Consistency-Weighted Majority Voting (Algorithm~\ref{alg:regen_wmv})}
\label{sec:wmv}

We set the WMV weight in Eq.~\eqref{eq:wmv} using Eq.~\eqref{eq:ci_def}:
\begin{equation}
\label{eq:pc_weight}
    v_i(a) \;=\; w\!\left(c_i^{(\tau)}(a)\right)
\end{equation}
where $w \colon [0,1] \to \mathbb{R}_{\geq 0}$ with $w(0) = 0$.

We refer to this method as \textbf{prefix-consistency-weighted majority voting (PC-WMV)} (Algorithm~\ref{alg:regen_wmv}).
Since Eq.~\eqref{eq:pc_weight} weights every distinct $a \in A_i^{(\tau)}$ rather than only the initial answer $a_i$, PC-WMV's aggregated vote $\sum_i v_i(a)$ can be positive for regenerated answers absent from the initial $N$ samples.
This is the operational consequence of using a per-candidate signal rather than a per-sample one.

We now demonstrate that this additional flexibility results in a clear advantage over Standard MV in situations where Standard MV is proven to fail. Standard MV fails when the correct answer occurs less frequently than a wrong answer. The following theorem, in the simplest case of a binary answer space, demonstrates how PC-WMV uses the reproduction-rate asymmetry to recover the correct answer in this regime.

\begin{algorithm}[t]
\small
\setlength{\belowcaptionskip}{0pt}
\caption{Prefix-Consistency-Weighted Majority Voting (PC-WMV)}
\label{alg:regen_wmv}
\begin{algorithmic}[1]
\REQUIRE Problem $q$, \#groups $N$, truncation fraction $\tau$, \#regenerations per group $K$, weighting $w$
\STATE $\mathrm{votes} \leftarrow \emptyset$ \COMMENT{map from answer to accumulated votes}
\FOR{$i = 1$ to $N$}
    \STATE $y_i \leftarrow \textsc{GenerateTrace}(q)$
    \STATE $a_i \leftarrow \textsc{ExtractAnswer}(y_i)$
    \STATE $y_i[{:}\lceil \tau |y_i| \rceil] \leftarrow \textsc{Truncate}(y_i, \tau)$ \COMMENT{shared prefix}
    \STATE $A_i^{(\tau,K)} \leftarrow [a_i]$ \COMMENT{multiset of answers for $i$-th group}
    \FOR{$k = 1$ to $K$}
        \STATE $\tilde{y}_{i,k}^{(\tau)} \leftarrow \textsc{ContinueFromPrefix}(q, y_i[{:}\lceil \tau |y_i| \rceil])$
        \STATE $\tilde{a}_{i,k}^{(\tau)} \leftarrow \textsc{ExtractAnswer}(\tilde{y}_{i,k}^{(\tau)})$
        \STATE append $\tilde{a}_{i,k}^{(\tau)}$ to $A_i^{(\tau,K)}$
    \ENDFOR
    \FOR{each distinct $a$ in $A_i^{(\tau,K)}$}
        \STATE $c_i^{(\tau,K)}(a) \leftarrow |\{a' \in A_i^{(\tau,K)} : a' = a\}| / (K{+}1)$ \COMMENT{general-$K$ form of Eq.~\eqref{eq:ci_def}}
        \STATE $\mathrm{votes}[a] \leftarrow \mathrm{votes}[a] + w(c_i^{(\tau,K)}(a))$
    \ENDFOR
\ENDFOR
\RETURN $\displaystyle\argmax_{a \in \mathcal{A}} \mathrm{votes}[a]$
\end{algorithmic}
\end{algorithm}

\begin{maintheorem}[PC-WMV strictly improves over Standard MV when $r_C(\tau) > r_W(\tau)$]
\label{thm:improvement}
\leavevmode\\
Fix $\tau \in (0, 1)$ and $K = 1$. Let $\mathcal A = \{a^\star, a'\}$, where $a^\star$ is the correct answer and $a'$ is the only wrong answer. Assume the $N$ groups are i.i.d., and let $\pi(a^\star) := \Pr[a_i = a^\star]$ denote the Pass@1. Assume $r_C(\tau) > r_W(\tau) > 0$. For any weighting function $w$ with $w(0) = 0$ and $w(1) > 0$, in the limit of $N \rightarrow \infty$, PC-WMV converges to $a^\star$ a.s.\ iff
\begin{equation}
\label{eq:improvement_condition}
    \pi(a^\star) > \frac{r_W(\tau)}{r_C(\tau) + r_W(\tau)}.
\end{equation}
In contrast, Standard MV converges to $a^\star$ a.s.\ if and only if $\pi(a^\star) > \tfrac12$. Therefore, PC-WMV converges to $a^\star$ on the interval where Standard MV does not:
\begin{equation}
\label{eq:improvement_region}
    \pi(a^\star) \;\in\; \left(\frac{r_W(\tau)}{r_C(\tau) + r_W(\tau)},\; \tfrac12\right],
\end{equation}
of width
\begin{equation}
\label{eq:improvement_width}
    \frac{D(\tau)}{2\,(r_C(\tau) + r_W(\tau))} \qquad (D(\tau) := r_C(\tau) - r_W(\tau)).
\end{equation}
\end{maintheorem}
The formal proof of Theorem~\ref{thm:improvement} is in Appendix~\ref{app:binary}.

\begin{interpretation}
The key quantity in our analysis is $D(\tau)$, which we call the \emph{discrimination gap}.
Theorem~\ref{thm:improvement} considers the simple case of two answer candidates, correct and wrong: when $D(\tau) > 0$, PC-WMV recovers the correct answer when Standard MV does not, provided that Pass@1 falls within the interval of Eq.~\eqref{eq:improvement_region}.
At fixed $r_C(\tau)+r_W(\tau)$, a larger $D(\tau)$ widens the region in which PC-WMV outperforms Standard MV.
Our results (Table~\ref{tab:signal}) show that $D(\tau)$ is substantially larger than $0$ across models and benchmarks, so PC-WMV is effective in practice.
\end{interpretation}

\paragraph{Hyperparameters.}
Prefix consistency has two hyperparameters: the truncation fraction $\tau \in (0, 1)$ and the number of regenerations per group $K \in \mathbb{N}$. PC-WMV adds a third, the weighting function $w \colon [0, 1] \to \mathbb{R}_{\geq 0}$ with $w(0) = 0$.

\section{Experiments}
\label{sec:experiments}

We conduct experiments on science (FrontierScience-Olympiad~\citep{openai2025frontierscience}) and math (HMMT Feb~2026, AIME~2025, Brumo~2025~\citep{balunovic2025matharena}) datasets.
We evaluate on five reasoning LLMs: GPT-OSS-120B, GPT-OSS-20B~\citep{openai2025gptoss}, Nemotron3-30B (\texttt{NVIDIA-Nemotron-3-Nano-30B-A3B-BF16})~\citep{nvidia2025nemotron3nano}, Nemotron2-9B (\texttt{NVIDIA-Nemotron-Nano-9B-v2})~\citep{nvidia2025nemotronnano2}, and Ministral3-14B (\texttt{Ministral-3-14B-Reasoning-2512})~\citep{mistral2026ministral3}.

We compare our proposed methods against Standard MV, three primary WMV baselines (Self-certainty~\citep{kang2025selfcertainty}, DeepConf tail~\citep{fu2025deepthinkconfidence}, and P(True)~\citep{kadavath2022languagemodelsknowthey,taubenfeld2025confidenceimprovesselfconsistency}), and two adaptive-stopping rules over MV (Adaptive Consistency, AC~\citep{aggarwal-etal-2023-lets}, and Early-Stopping Self-Consistency, ESC~\citep{li2024escape}).
The details of these methods are documented in Appendix~\ref{app:baseline_impl}.

\paragraph{Hyperparameters of prefix consistency.}
Unless otherwise specified, we fix the truncation fraction at $\tau = 0.75$. The results in the main paper use $K = 1$ throughout. We accordingly suppress $\tau$ in the notation and write $c_i(a), r_C, r_W, D$ for $c_i^{(\tau)}(a), r_C(\tau), r_W(\tau), D(\tau)$.

\subsection{Prefix Consistency as a Correctness Predictor}
\label{sec:signal}

We report $\overline{\mathrm{AUROC}}$, the macro-averaged AUROC over problems with at least one correct and one wrong initial sample (formal definition in Appendix~\ref{app:auroc_eval}).

Note that some previous work~\citep{xiong2024llmsexpressuncertainty,fadeeva2024factchecking} adopted AUROC pooled across problems, which differs from $\overline{\mathrm{AUROC}}$. However,
as argued in \citet{taubenfeld2025confidenceimprovesselfconsistency}, such an AUROC pooled across problems conflates within-problem discrimination with cross-problem score-difficulty correlation, and only the former predicts whether confidence-weighted self-consistency improves over Standard MV. They also report that calibration metrics such as Expected Calibration Error and Brier score are similarly unsuitable. In their data, the best-calibrated source (verbalized binary) gave the smallest improvement while the strongest method (P(True) therein) was only moderately calibrated.
At vote time, every WMV method (including PC-WMV and all baselines) compares scores only among samples from the same problem, and thus we consider $\overline{\mathrm{AUROC}}$ to better measure discriminative ability between correct traces and wrong traces.

Table~\ref{tab:signal} reports the discrimination gap $D = r_C - r_W$ for prefix consistency: $D > 0$ on every (model, benchmark) cell, confirming the asymmetry $r_C > r_W$.
Table~\ref{tab:auroc} reports $\overline{\mathrm{AUROC}}$ for prefix consistency against the WMV baselines. Prefix consistency has the highest $\overline{\mathrm{AUROC}}$ on $15$ of $20$ cells (typically around $0.7$), separating correct from wrong traces more clearly than the baselines.
Baselines' $\overline{\mathrm{AUROC}}$ often hovers near $0.5$ (= random\footnote{$\overline{\mathrm{AUROC}} = 0.5$ is the value attained by a random signal that is uninformative about correctness.}) on harder cells, where their scores differ little between correct and wrong samples, reaching $\sim 0.7$ only on some easier cells.

\begin{table}
\centering
\caption{\textbf{Reproduction rates $r_C$, $r_W$ and discrimination gap $D = r_C - r_W$ for prefix consistency (larger $D$ is better).} Macro-averaged over problems with at least one correct and one wrong initial sample. $r_C \geq r_W$ holds on every (model, benchmark) cell, and a larger $D$ predicts a larger PC-WMV advantage over Standard MV (Theorem~\ref{thm:improvement}).}
\label{tab:signal}
\footnotesize
\setlength{\tabcolsep}{3pt}
\begin{tabular}{l ccc ccc ccc ccc ccc}
\toprule
& \multicolumn{3}{c}{GPT-OSS-120B} & \multicolumn{3}{c}{GPT-OSS-20B} & \multicolumn{3}{c}{Nemotron3-30B} & \multicolumn{3}{c}{Nemotron2-9B} & \multicolumn{3}{c}{Ministral3-14B} \\
\cmidrule(lr){2-4} \cmidrule(lr){5-7} \cmidrule(lr){8-10} \cmidrule(lr){11-13} \cmidrule(lr){14-16}
Benchmark & $r_C$ & $r_W$ & $D$ & $r_C$ & $r_W$ & $D$ & $r_C$ & $r_W$ & $D$ & $r_C$ & $r_W$ & $D$ & $r_C$ & $r_W$ & $D$ \\
\midrule
FrontierScience-Olympiad & 66.0 & 22.3 & 43.7 & 55.4 & 14.0 & 41.4 & 50.9 & 24.7 & 26.2 & 45.0 & 37.5 & 7.4 & 56.4 & 26.0 & 30.4 \\
HMMT Feb~2026 & 87.8 & 48.1 & 39.7 & 78.3 & 29.7 & 48.6 & 80.0 & 40.4 & 39.6 & 79.3 & 49.6 & 29.7 & 74.5 & 32.0 & 42.4 \\
AIME~2025 & 95.5 & 61.7 & 33.7 & 87.8 & 42.6 & 45.2 & 74.7 & 38.2 & 36.5 & 69.5 & 43.2 & 26.3 & 77.3 & 26.1 & 51.1 \\
Brumo~2025 & 83.8 & 56.7 & 27.1 & 82.8 & 39.0 & 43.7 & 75.6 & 15.3 & 60.3 & 81.2 & 54.7 & 26.6 & 76.2 & 24.4 & 51.8 \\
\bottomrule
\end{tabular}
\vspace{2pt}
\par\raggedright\footnotesize
\end{table}

\begin{table}
\centering
\caption{\textbf{$\overline{\mathrm{AUROC}}$ for correctness discrimination (higher is better).} Macro-averaged $\overline{\mathrm{AUROC}}$ per (model, benchmark), with the best per column in bold.}
\label{tab:auroc}
\footnotesize
\setlength{\tabcolsep}{1.5pt}
\resizebox{\textwidth}{!}{%
\begin{tabular}{l cccc cccc cccc cccc cccc}
\toprule
& \multicolumn{4}{c}{GPT-OSS-120B} & \multicolumn{4}{c}{GPT-OSS-20B} & \multicolumn{4}{c}{Nemotron3-30B} & \multicolumn{4}{c}{Nemotron2-9B} & \multicolumn{4}{c}{Ministral3-14B} \\
\cmidrule(lr){2-5} \cmidrule(lr){6-9} \cmidrule(lr){10-13} \cmidrule(lr){14-17} \cmidrule(lr){18-21}
Signal & FSci & HMMT & AIME & Brumo & FSci & HMMT & AIME & Brumo & FSci & HMMT & AIME & Brumo & FSci & HMMT & AIME & Brumo & FSci & HMMT & AIME & Brumo \\
\midrule
Prefix consistency & \textbf{.719} & .698 & .669 & \textbf{.636} & \textbf{.707} & \textbf{.743} & \textbf{.726} & \textbf{.719} & \textbf{.631} & \textbf{.698} & \textbf{.682} & \textbf{.801} & .537 & \textbf{.648} & .631 & .633 & \textbf{.652} & \textbf{.712} & \textbf{.756} & \textbf{.759} \\
Self-certainty & .570 & .664 & .641 & .549 & .545 & .570 & .560 & .484 & .496 & .414 & .520 & .655 & .458 & .605 & .537 & .580 & .318 & .359 & .408 & .342 \\
DeepConf bottom-10\% & .525 & .640 & .601 & .486 & .525 & .580 & .565 & .481 & .474 & .413 & .514 & .607 & .434 & .515 & .529 & .515 & .342 & .331 & .395 & .347 \\
DeepConf tail & .565 & \textbf{.728} & \textbf{.703} & .606 & .583 & .659 & .645 & .557 & .508 & .588 & .465 & .755 & .494 & .628 & .665 & .641 & .349 & .404 & .421 & .358 \\
Verbal 0--100 & .516 & .552 & .505 & .436 & .516 & .574 & .561 & .511 & .524 & .489 & .606 & .532 & .571 & .563 & .584 & .548 & .505 & .525 & .525 & .450 \\
P(True) & .568 & .550 & .551 & .527 & .587 & .667 & .592 & .583 & .493 & .532 & .578 & .564 & \textbf{.620} & .635 & \textbf{.724} & \textbf{.708} & .469 & .454 & .441 & .477 \\
\bottomrule
\end{tabular}%
}%
\vspace{2pt}
\par\raggedright\footnotesize Abbreviations: FSci = FrontierScience-Olympiad, HMMT = HMMT Feb~2026, AIME = AIME~2025, Brumo = Brumo~2025.
\end{table}

\subsection{Weighted Majority Voting Results}
\label{sec:wmv_results}

We compare PC-WMV against existing WMV methods under the same generated-token budget.

\paragraph{Weighting variants.}
We use the power family $w^{(n)}(c) = c^n$ for $n \in \{1, 2, 3\}$, denoted PC-linear, PC-quadratic, and PC-cubic, where the ``PC'' prefix abbreviates prefix consistency.
Under $K = 1$, $c_i(a) \in \{0, 1/2, 1\}$. Thus, a candidate that is reproduced under regeneration ($c = 1$) receives weight $1$, while a candidate that appears in only one of the two traces ($c = 1/2$) receives weight $1/2^n$.
The weight ratio between a reproduced candidate and a single-trace one is therefore $2^n : 1$, that is, $2 : 1$ for linear, $4 : 1$ for quadratic, and $8 : 1$ for cubic. The larger $n$ is, the more pronounced the relative weight given to reproduced answers.

\paragraph{Cost-equivalent evaluation.}
We measure inference cost by the total number of generated tokens, treating each generated token as equally expensive, and log-probability access as free.
For each model and benchmark, we first generate $N{=}128$ initial samples per problem ($N{=}64$ for Ministral3-14B), from which all methods sample with replacement under a common token budget $B$. See Appendix~\ref{app:eval_protocol} for the pool construction, trial design, and confidence-interval definition.

Table~\ref{tab:wmv} reports accuracy under fixed token budgets (250k, 1M, and 5M tokens) across the three models and four benchmarks.
At 1M tokens, prefix consistency matches or exceeds all baselines on the more difficult benchmarks (FrontierScience-Olympiad, HMMT Feb~2026), while the advantage is smaller on AIME~2025 where Standard MV already achieves high accuracy.
The improvements are consistent across weighting functions (PC-linear, PC-quadratic, and PC-cubic). PC-cubic provides the greatest improvement for the most difficult problems.
Per-model tables with the full set of baselines (DeepConf variants, Response probability, verbalized confidence, P(True)) for all five models, including the two not shown above (GPT-OSS-20B, Nemotron2-9B), are reported in Appendix~\ref{app:all_baselines}.

\begin{table}
\centering
\caption{\textbf{Weighted majority voting accuracy at fixed token budget $B$ (higher is better).} Each method's accuracy at $B \in \{250\text{k}, 1\text{M}, 5\text{M}\}$ tokens, with the best per (model, $B$) column in bold.}
\label{tab:wmv}
\scriptsize
\setlength{\tabcolsep}{2.2pt}
\renewcommand{\arraystretch}{0.95}
\resizebox{\textwidth}{!}{%
\begin{tabular}{ll ccc ccc ccc}
\toprule
& & \multicolumn{3}{c}{GPT-OSS-120B} & \multicolumn{3}{c}{Nemotron3-30B} & \multicolumn{3}{c}{Ministral3-14B} \\
\cmidrule(lr){3-5} \cmidrule(lr){6-8} \cmidrule(lr){9-11}
Benchmark & Method & $B{=}$250k & $B{=}$1M & $B{=}$5M & $B{=}$250k & $B{=}$1M & $B{=}$5M & $B{=}$250k & $B{=}$1M & $B{=}$5M \\
\midrule
\multirow{7}{*}{\shortstack[l]{FrontierScience-Olympiad}} & Standard MV & .493$_{\pm.001}$ & .495$_{\pm.001}$ & .499$_{\pm.001}$ & .404$_{\pm.002}$ & .468$_{\pm.002}$ & .490$_{\pm.001}$ & .137$_{\pm.001}$ & .141$_{\pm.001}$ & .142$_{\pm.001}$ \\
 & Self-certainty & .494$_{\pm.001}$ & .495$_{\pm.001}$ & .497$_{\pm.001}$ & .402$_{\pm.002}$ & .468$_{\pm.002}$ & .491$_{\pm.001}$ & .130$_{\pm.001}$ & .135$_{\pm.001}$ & .137$_{\pm.001}$ \\
 & DeepConf tail & .493$_{\pm.001}$ & .495$_{\pm.001}$ & .497$_{\pm.001}$ & .405$_{\pm.002}$ & .470$_{\pm.002}$ & .494$_{\pm.001}$ & .130$_{\pm.001}$ & .134$_{\pm.001}$ & .137$_{\pm.001}$ \\
 & P(True) & .493$_{\pm.001}$ & .500$_{\pm.001}$ & .500$_{\pm.000}$ & .394$_{\pm.003}$ & .465$_{\pm.002}$ & .492$_{\pm.002}$ & .137$_{\pm.001}$ & .143$_{\pm.001}$ & .146$_{\pm.001}$ \\
 & PC-linear & .495$_{\pm.001}$ & .496$_{\pm.001}$ & .496$_{\pm.001}$ & .429$_{\pm.002}$ & .481$_{\pm.002}$ & .494$_{\pm.001}$ & .149$_{\pm.002}$ & .160$_{\pm.001}$ & .162$_{\pm.001}$ \\
 & PC-quadratic & .502$_{\pm.001}$ & .504$_{\pm.001}$ & .503$_{\pm.001}$ & \textbf{.433}$_{\pm.002}$ & \textbf{.486}$_{\pm.002}$ & \textbf{.504}$_{\pm.002}$ & .156$_{\pm.002}$ & .170$_{\pm.001}$ & .173$_{\pm.001}$ \\
 & PC-cubic & \textbf{.506}$_{\pm.001}$ & \textbf{.508}$_{\pm.001}$ & \textbf{.507}$_{\pm.001}$ & \textbf{.433}$_{\pm.002}$ & \textbf{.486}$_{\pm.002}$ & .503$_{\pm.001}$ & \textbf{.158}$_{\pm.002}$ & \textbf{.174}$_{\pm.001}$ & \textbf{.181}$_{\pm.001}$ \\
\midrule
\multirow{7}{*}{\shortstack[l]{HMMT Feb~2026}} & Standard MV & .708$_{\pm.003}$ & .745$_{\pm.003}$ & .763$_{\pm.001}$ & .736$_{\pm.004}$ & .777$_{\pm.003}$ & .802$_{\pm.002}$ & .413$_{\pm.004}$ & \textbf{.440}$_{\pm.003}$ & .458$_{\pm.002}$ \\
 & Self-certainty & .713$_{\pm.003}$ & .749$_{\pm.003}$ & .768$_{\pm.002}$ & .734$_{\pm.004}$ & .779$_{\pm.003}$ & \textbf{.804}$_{\pm.002}$ & .409$_{\pm.004}$ & .429$_{\pm.003}$ & .438$_{\pm.002}$ \\
 & DeepConf tail & .720$_{\pm.003}$ & .758$_{\pm.003}$ & .783$_{\pm.002}$ & .740$_{\pm.004}$ & .779$_{\pm.003}$ & .801$_{\pm.002}$ & .408$_{\pm.004}$ & .427$_{\pm.003}$ & .436$_{\pm.002}$ \\
 & P(True) & .706$_{\pm.003}$ & .743$_{\pm.003}$ & .784$_{\pm.002}$ & .736$_{\pm.004}$ & .779$_{\pm.003}$ & \textbf{.804}$_{\pm.002}$ & .404$_{\pm.003}$ & .429$_{\pm.003}$ & .433$_{\pm.002}$ \\
 & PC-linear & .717$_{\pm.003}$ & .751$_{\pm.002}$ & .764$_{\pm.001}$ & \textbf{.742}$_{\pm.004}$ & .785$_{\pm.003}$ & .802$_{\pm.002}$ & \textbf{.418}$_{\pm.004}$ & \textbf{.440}$_{\pm.003}$ & .444$_{\pm.002}$ \\
 & PC-quadratic & .725$_{\pm.003}$ & .764$_{\pm.002}$ & .782$_{\pm.001}$ & \textbf{.742}$_{\pm.004}$ & \textbf{.787}$_{\pm.003}$ & .803$_{\pm.002}$ & .415$_{\pm.004}$ & \textbf{.440}$_{\pm.003}$ & .447$_{\pm.002}$ \\
 & PC-cubic & \textbf{.727}$_{\pm.003}$ & \textbf{.771}$_{\pm.002}$ & \textbf{.786}$_{\pm.001}$ & \textbf{.742}$_{\pm.004}$ & \textbf{.787}$_{\pm.003}$ & \textbf{.804}$_{\pm.002}$ & .414$_{\pm.004}$ & \textbf{.440}$_{\pm.003}$ & \textbf{.461}$_{\pm.002}$ \\
\midrule
\multirow{7}{*}{\shortstack[l]{AIME~2025}} & Standard MV & .901$_{\pm.002}$ & .901$_{\pm.001}$ & .900$_{\pm.000}$ & .935$_{\pm.002}$ & \textbf{.966}$_{\pm.000}$ & \textbf{.967}$_{\pm.000}$ & .508$_{\pm.003}$ & .535$_{\pm.003}$ & .569$_{\pm.002}$ \\
 & Self-certainty & .904$_{\pm.002}$ & .903$_{\pm.001}$ & .900$_{\pm.000}$ & .935$_{\pm.002}$ & \textbf{.966}$_{\pm.000}$ & \textbf{.967}$_{\pm.000}$ & .505$_{\pm.003}$ & .533$_{\pm.003}$ & .565$_{\pm.002}$ \\
 & DeepConf tail & .906$_{\pm.002}$ & .905$_{\pm.001}$ & .900$_{\pm.000}$ & \textbf{.937}$_{\pm.002}$ & \textbf{.966}$_{\pm.000}$ & \textbf{.967}$_{\pm.000}$ & .506$_{\pm.003}$ & .533$_{\pm.003}$ & .565$_{\pm.002}$ \\
 & P(True) & .902$_{\pm.002}$ & .909$_{\pm.001}$ & .906$_{\pm.001}$ & .930$_{\pm.003}$ & .964$_{\pm.001}$ & \textbf{.967}$_{\pm.000}$ & .495$_{\pm.003}$ & .523$_{\pm.003}$ & .542$_{\pm.002}$ \\
 & PC-linear & .906$_{\pm.002}$ & .913$_{\pm.002}$ & .909$_{\pm.001}$ & .932$_{\pm.003}$ & .964$_{\pm.001}$ & \textbf{.967}$_{\pm.000}$ & .556$_{\pm.003}$ & .583$_{\pm.003}$ & .600$_{\pm.002}$ \\
 & PC-quadratic & .912$_{\pm.002}$ & .922$_{\pm.002}$ & .930$_{\pm.002}$ & .932$_{\pm.003}$ & .963$_{\pm.001}$ & \textbf{.967}$_{\pm.000}$ & \textbf{.565}$_{\pm.003}$ & .595$_{\pm.003}$ & .616$_{\pm.002}$ \\
 & PC-cubic & \textbf{.913}$_{\pm.002}$ & \textbf{.926}$_{\pm.002}$ & \textbf{.941}$_{\pm.002}$ & .932$_{\pm.003}$ & .963$_{\pm.001}$ & \textbf{.967}$_{\pm.000}$ & .564$_{\pm.003}$ & \textbf{.597}$_{\pm.003}$ & \textbf{.621}$_{\pm.002}$ \\
\midrule
\multirow{7}{*}{\shortstack[l]{Brumo~2025}} & Standard MV & .801$_{\pm.003}$ & .821$_{\pm.002}$ & .833$_{\pm.000}$ & .896$_{\pm.003}$ & .928$_{\pm.002}$ & .931$_{\pm.001}$ & .669$_{\pm.004}$ & .691$_{\pm.003}$ & .699$_{\pm.003}$ \\
 & Self-certainty & .805$_{\pm.003}$ & .826$_{\pm.002}$ & .833$_{\pm.000}$ & .918$_{\pm.003}$ & .954$_{\pm.002}$ & .963$_{\pm.001}$ & .659$_{\pm.004}$ & .682$_{\pm.003}$ & .686$_{\pm.002}$ \\
 & DeepConf tail & .810$_{\pm.003}$ & .831$_{\pm.001}$ & .833$_{\pm.000}$ & .920$_{\pm.003}$ & .955$_{\pm.002}$ & .965$_{\pm.001}$ & .659$_{\pm.004}$ & .681$_{\pm.003}$ & .683$_{\pm.002}$ \\
 & P(True) & .794$_{\pm.003}$ & .821$_{\pm.003}$ & .844$_{\pm.002}$ & .886$_{\pm.003}$ & .917$_{\pm.002}$ & .929$_{\pm.001}$ & .632$_{\pm.004}$ & .658$_{\pm.003}$ & .661$_{\pm.003}$ \\
 & PC-linear & .810$_{\pm.003}$ & .830$_{\pm.002}$ & .833$_{\pm.000}$ & \textbf{.933}$_{\pm.003}$ & \textbf{.961}$_{\pm.002}$ & \textbf{.968}$_{\pm.001}$ & .675$_{\pm.004}$ & .705$_{\pm.003}$ & .723$_{\pm.002}$ \\
 & PC-quadratic & .815$_{\pm.003}$ & .837$_{\pm.002}$ & .842$_{\pm.001}$ & .932$_{\pm.003}$ & .959$_{\pm.002}$ & \textbf{.968}$_{\pm.001}$ & .684$_{\pm.004}$ & .719$_{\pm.003}$ & .740$_{\pm.002}$ \\
 & PC-cubic & \textbf{.818}$_{\pm.003}$ & \textbf{.842}$_{\pm.002}$ & \textbf{.857}$_{\pm.001}$ & .929$_{\pm.003}$ & .953$_{\pm.002}$ & .964$_{\pm.001}$ & \textbf{.686}$_{\pm.004}$ & \textbf{.727}$_{\pm.003}$ & \textbf{.759}$_{\pm.003}$ \\
\bottomrule
\end{tabular}%
}%
\end{table}

\subsection{Token Efficiency}
\label{sec:token_efficiency}

Section~\ref{sec:wmv_results} reported accuracy under fixed budget constraints.
We next compare how many tokens each method needs to reach the same target accuracy.

Table~\ref{tab:token_savings} shows the token-efficiency ratio $B_{\mathrm{method}} / B_{\mathrm{MV}}$, where $B_X$ is the budget method $X$ needs to reach the target accuracy Pass@1 $+\, \alpha \times$ (Standard MV plateau $-$ Pass@1) for $\alpha \in \{75\%, 90\%, 99\%\}$. The Standard MV plateau is Standard MV's bootstrap-saturated accuracy on the $N$-sample pool (Appendix~\ref{app:cost_accuracy_eval}), so $\alpha$ interpolates between Pass@1 ($\alpha{=}0$) and this plateau ($\alpha{=}1$). A ratio $<1$ indicates a cost advantage for $X$ over Standard MV at the target. For example, PC-cubic, shown as ``Prefix consistency'' in Figure~\ref{fig:regen-consistency-main}, achieves a token-efficiency ratio of approximately $0.05$, corresponding to a $21\times$ reduction in tokens. The headline numbers in this paper (median $4.6\times$, up to $21\times$ vs.\ Standard MV, up to $10\times$ vs.\ AC sweep) are computed at $\alpha{=}99\%$ across all $20$ (model, benchmark) cells: Table~\ref{tab:token_savings} together with Table~\ref{tab:token_savings_gpt_oss_20b_all} (GPT-OSS-20B) and Table~\ref{tab:token_savings_nemotron2_9b_all} (Nemotron2-9B) in Appendix~\ref{app:all_baselines}.

\begin{table}
\centering
\caption{\textbf{Token efficiency ratio $B_{\mathrm{method}} / B_{\mathrm{MV}}$ at target accuracy $\alpha$ between Pass@1 and the Standard MV plateau (smaller is better).} $B_X$ is the budget method $X$ needs to reach Pass@1 $+\, \alpha \times$ (Standard MV plateau $-$ Pass@1) for $\alpha \in \{75\%, 90\%, 99\%\}$. Cells $<$1 indicate more cost-efficient than Standard MV. ``N/A'' indicates the method's plateau is below the target. Best method per column in bold.}
\label{tab:token_savings}
\scriptsize
\setlength{\tabcolsep}{2.2pt}
\renewcommand{\arraystretch}{0.95}
\resizebox{\textwidth}{!}{%
\begin{tabular}{ll ccc ccc ccc}
\toprule
& & \multicolumn{3}{c}{GPT-OSS-120B} & \multicolumn{3}{c}{Nemotron3-30B} & \multicolumn{3}{c}{Ministral3-14B} \\
\cmidrule(lr){3-5} \cmidrule(lr){6-8} \cmidrule(lr){9-11}
Benchmark & Method & $\alpha{=}75\%$ & $\alpha{=}90\%$ & $\alpha{=}99\%$ & $\alpha{=}75\%$ & $\alpha{=}90\%$ & $\alpha{=}99\%$ & $\alpha{=}75\%$ & $\alpha{=}90\%$ & $\alpha{=}99\%$ \\
\midrule
\multirow{10}{*}{\shortstack[l]{FrontierScience-Olympiad}} & Pass@1 / Standard MV plateau & .338 / .500 &  &  & .295 / .493 &  &  & .091 / .142 &  &  \\
 & Standard MV budget ($B_{\mathrm{MV}}$) & 45k & 107k & 2.5M & 486k & 1.3M & 6.0M & 118k & 249k & 924k \\
\cmidrule(l){2-11}
 & Self-certainty & 0.89$_{\pm0.04}$$\times$ & 0.87$_{\pm0.06}$$\times$ & N/A & 1.04$_{\pm0.06}$$\times$ & 1.01$_{\pm0.11}$$\times$ & 0.85$_{\pm0.19}$$\times$ & 1.62$_{\pm0.20}$$\times$ & $>$10$\times$ & N/A \\
 & DeepConf tail & 0.91$_{\pm0.05}$$\times$ & 0.91$_{\pm0.07}$$\times$ & N/A & 0.97$_{\pm0.06}$$\times$ & 0.88$_{\pm0.08}$$\times$ & 0.56$_{\pm0.14}$$\times$ & 1.54$_{\pm0.18}$$\times$ & $>$10$\times$ & N/A \\
 & P(True) & 1.23$_{\pm0.05}$$\times$ & 1.20$_{\pm0.10}$$\times$ & 0.23$_{\pm0.09}$$\times$ & 1.14$_{\pm0.06}$$\times$ & 1.10$_{\pm0.11}$$\times$ & 0.75$_{\pm0.18}$$\times$ & 1.10$_{\pm0.12}$$\times$ & 1.01$_{\pm0.27}$$\times$ & 0.48$_{\pm0.30}$$\times$ \\
 & AC sweep & 1.93$_{\pm0.09}$$\times$ & 1.54$_{\pm0.12}$$\times$ & 0.29$_{\pm0.15}$$\times$ & 1.24$_{\pm0.07}$$\times$ & 0.87$_{\pm0.09}$$\times$ & 0.49$_{\pm0.13}$$\times$ & 2.48$_{\pm0.40}$$\times$ & 4.28$_{\pm1.45}$$\times$ & N/A \\
 & ESC sweep & 2.74$_{\pm0.08}$$\times$ & 5.59$_{\pm0.70}$$\times$ & 0.94$_{\pm0.53}$$\times$ & 2.06$_{\pm0.26}$$\times$ & 3.62$_{\pm0.41}$$\times$ & 2.23$_{\pm0.58}$$\times$ & 4.53$_{\pm1.55}$$\times$ & $>$10$\times$ & $>$10$\times$ \\
 & PC-linear & 0.70$_{\pm0.03}$$\times$ & 0.63$_{\pm0.05}$$\times$ & 3.88$_{\pm2.37}$$\times$ & 0.67$_{\pm0.04}$$\times$ & 0.54$_{\pm0.05}$$\times$ & 0.40$_{\pm0.09}$$\times$ & \textbf{0.38$_{\pm0.04}$$\times$} & 0.26$_{\pm0.07}$$\times$ & 0.10$_{\pm0.07}$$\times$ \\
 & PC-quadratic & \textbf{0.65$_{\pm0.03}$$\times$} & \textbf{0.52$_{\pm0.04}$$\times$} & 0.06$_{\pm0.02}$$\times$ & \textbf{0.61$_{\pm0.03}$$\times$} & \textbf{0.48$_{\pm0.04}$$\times$} & 0.24$_{\pm0.06}$$\times$ & \textbf{0.38$_{\pm0.04}$$\times$} & \textbf{0.24$_{\pm0.06}$$\times$} & \textbf{0.08$_{\pm0.05}$$\times$} \\
 & PC-cubic & \textbf{0.65$_{\pm0.03}$$\times$} & \textbf{0.52$_{\pm0.04}$$\times$} & \textbf{0.05$_{\pm0.02}$$\times$} & \textbf{0.61$_{\pm0.03}$$\times$} & \textbf{0.48$_{\pm0.04}$$\times$} & \textbf{0.22$_{\pm0.05}$$\times$} & \textbf{0.38$_{\pm0.04}$$\times$} & \textbf{0.24$_{\pm0.06}$$\times$} & \textbf{0.08$_{\pm0.05}$$\times$} \\
\midrule
\multirow{10}{*}{\shortstack[l]{HMMT Feb~2026}} & Pass@1 / Standard MV plateau & .589 / .760 &  &  & .708 / .810 &  &  & .270 / .463 &  &  \\
 & Standard MV budget ($B_{\mathrm{MV}}$) & 364k & 858k & 1.8M & 1.6M & 4.2M & 8.5M & 258k & 1.2M & 7.4M \\
\cmidrule(l){2-11}
 & Self-certainty & 0.74$_{\pm0.07}$$\times$ & 0.85$_{\pm0.07}$$\times$ & 0.87$_{\pm0.11}$$\times$ & 0.85$_{\pm0.15}$$\times$ & 0.89$_{\pm0.14}$$\times$ & 0.93$_{\pm0.12}$$\times$ & 1.29$_{\pm0.17}$$\times$ & N/A & N/A \\
 & DeepConf tail & 0.63$_{\pm0.09}$$\times$ & 0.72$_{\pm0.08}$$\times$ & 0.57$_{\pm0.07}$$\times$ & 0.82$_{\pm0.15}$$\times$ & 0.97$_{\pm0.19}$$\times$ & N/A & 1.29$_{\pm0.17}$$\times$ & N/A & N/A \\
 & P(True) & 1.05$_{\pm0.14}$$\times$ & 1.17$_{\pm0.14}$$\times$ & 0.94$_{\pm0.09}$$\times$ & 0.86$_{\pm0.15}$$\times$ & 0.90$_{\pm0.12}$$\times$ & 0.82$_{\pm0.10}$$\times$ & 1.33$_{\pm0.14}$$\times$ & N/A & N/A \\
 & AC sweep & \textbf{0.42$_{\pm0.05}$$\times$} & 0.63$_{\pm0.10}$$\times$ & N/A & \textbf{0.51$_{\pm0.09}$$\times$} & \textbf{0.34$_{\pm0.05}$$\times$} & \textbf{0.24$_{\pm0.03}$$\times$} & 1.52$_{\pm0.28}$$\times$ & 1.21$_{\pm0.28}$$\times$ & \textbf{0.60$_{\pm0.20}$$\times$} \\
 & ESC sweep & 1.96$_{\pm0.33}$$\times$ & 6.66$_{\pm1.29}$$\times$ & N/A & 4.33$_{\pm1.24}$$\times$ & 3.12$_{\pm0.46}$$\times$ & 1.88$_{\pm0.26}$$\times$ & 8.79$_{\pm1.33}$$\times$ & 4.75$_{\pm1.03}$$\times$ & 2.05$_{\pm0.75}$$\times$ \\
 & PC-linear & 0.70$_{\pm0.08}$$\times$ & 0.79$_{\pm0.08}$$\times$ & 0.88$_{\pm0.11}$$\times$ & 0.62$_{\pm0.11}$$\times$ & 0.86$_{\pm0.16}$$\times$ & N/A & \textbf{0.85$_{\pm0.11}$$\times$} & 1.11$_{\pm0.28}$$\times$ & N/A \\
 & PC-quadratic & 0.52$_{\pm0.05}$$\times$ & 0.49$_{\pm0.04}$$\times$ & 0.41$_{\pm0.04}$$\times$ & 0.56$_{\pm0.10}$$\times$ & 0.80$_{\pm0.15}$$\times$ & 1.10$_{\pm0.15}$$\times$ & 0.91$_{\pm0.11}$$\times$ & 1.01$_{\pm0.26}$$\times$ & N/A \\
 & PC-cubic & 0.52$_{\pm0.05}$$\times$ & \textbf{0.46$_{\pm0.04}$$\times$} & \textbf{0.35$_{\pm0.04}$$\times$} & 0.59$_{\pm0.10}$$\times$ & 0.79$_{\pm0.13}$$\times$ & 0.95$_{\pm0.14}$$\times$ & 1.03$_{\pm0.12}$$\times$ & 0.98$_{\pm0.20}$$\times$ & 0.67$_{\pm0.20}$$\times$ \\
\midrule
\multirow{10}{*}{\shortstack[l]{AIME~2025}} & Pass@1 / Standard MV plateau & .789 / .900 &  &  & .902 / .967 &  &  & .345 / .576 &  &  \\
 & Standard MV budget ($B_{\mathrm{MV}}$) & 56k & 99k & 189k & 396k & 589k & 1.2M & 371k & 2.0M & 6.9M \\
\cmidrule(l){2-11}
 & Self-certainty & 0.81$_{\pm0.06}$$\times$ & 0.81$_{\pm0.07}$$\times$ & 0.66$_{\pm0.11}$$\times$ & 0.97$_{\pm0.06}$$\times$ & 0.96$_{\pm0.06}$$\times$ & 0.84$_{\pm0.12}$$\times$ & 1.21$_{\pm0.15}$$\times$ & 1.20$_{\pm0.12}$$\times$ & N/A \\
 & DeepConf tail & \textbf{0.68$_{\pm0.04}$$\times$} & \textbf{0.70$_{\pm0.06}$$\times$} & 0.56$_{\pm0.09}$$\times$ & 0.95$_{\pm0.06}$$\times$ & 0.95$_{\pm0.06}$$\times$ & 0.89$_{\pm0.13}$$\times$ & 1.22$_{\pm0.15}$$\times$ & 1.20$_{\pm0.12}$$\times$ & 1.43$_{\pm0.26}$$\times$ \\
 & P(True) & 1.14$_{\pm0.08}$$\times$ & 1.13$_{\pm0.09}$$\times$ & 1.01$_{\pm0.16}$$\times$ & 1.12$_{\pm0.07}$$\times$ & 1.15$_{\pm0.07}$$\times$ & 1.20$_{\pm0.20}$$\times$ & 1.93$_{\pm0.22}$$\times$ & N/A & N/A \\
 & AC sweep & 1.03$_{\pm0.12}$$\times$ & 1.06$_{\pm0.26}$$\times$ & 2.98$_{\pm0.82}$$\times$ & \textbf{0.21$_{\pm0.01}$$\times$} & \textbf{0.18$_{\pm0.01}$$\times$} & \textbf{0.21$_{\pm0.05}$$\times$} & 0.91$_{\pm0.11}$$\times$ & 0.54$_{\pm0.06}$$\times$ & 0.49$_{\pm0.10}$$\times$ \\
 & ESC sweep & 0.96$_{\pm0.06}$$\times$ & 1.38$_{\pm0.57}$$\times$ & $>$10$\times$ & 0.28$_{\pm0.01}$$\times$ & 0.24$_{\pm0.05}$$\times$ & 0.59$_{\pm0.18}$$\times$ & 4.44$_{\pm0.83}$$\times$ & 2.41$_{\pm0.22}$$\times$ & 1.38$_{\pm0.26}$$\times$ \\
 & PC-linear & 0.87$_{\pm0.05}$$\times$ & 0.73$_{\pm0.06}$$\times$ & 0.58$_{\pm0.10}$$\times$ & 1.13$_{\pm0.08}$$\times$ & 1.21$_{\pm0.08}$$\times$ & 1.29$_{\pm0.24}$$\times$ & \textbf{0.25$_{\pm0.03}$$\times$} & 0.11$_{\pm0.01}$$\times$ & 0.07$_{\pm0.01}$$\times$ \\
 & PC-quadratic & 0.86$_{\pm0.05}$$\times$ & \textbf{0.70$_{\pm0.06}$$\times$} & \textbf{0.53$_{\pm0.08}$$\times$} & 1.22$_{\pm0.09}$$\times$ & 1.33$_{\pm0.10}$$\times$ & 1.44$_{\pm0.24}$$\times$ & 0.26$_{\pm0.03}$$\times$ & \textbf{0.09$_{\pm0.01}$$\times$} & \textbf{0.05$_{\pm0.01}$$\times$} \\
 & PC-cubic & 0.86$_{\pm0.05}$$\times$ & 0.71$_{\pm0.06}$$\times$ & \textbf{0.53$_{\pm0.08}$$\times$} & 1.23$_{\pm0.08}$$\times$ & 1.39$_{\pm0.10}$$\times$ & 1.67$_{\pm0.26}$$\times$ & 0.27$_{\pm0.03}$$\times$ & \textbf{0.09$_{\pm0.01}$$\times$} & \textbf{0.05$_{\pm0.01}$$\times$} \\
\midrule
\multirow{10}{*}{\shortstack[l]{Brumo~2025}} & Pass@1 / Standard MV plateau & .713 / .833 &  &  & .816 / .932 &  &  & .421 / .698 &  &  \\
 & Standard MV budget ($B_{\mathrm{MV}}$) & 274k & 934k & 3.3M & 342k & 651k & 2.5M & 92k & 265k & 1.8M \\
\cmidrule(l){2-11}
 & Self-certainty & 0.73$_{\pm0.18}$$\times$ & 0.76$_{\pm0.10}$$\times$ & 0.59$_{\pm0.09}$$\times$ & 0.45$_{\pm0.05}$$\times$ & 0.41$_{\pm0.05}$$\times$ & 0.14$_{\pm0.07}$$\times$ & 1.35$_{\pm0.11}$$\times$ & 1.47$_{\pm0.22}$$\times$ & N/A \\
 & DeepConf tail & 0.50$_{\pm0.12}$$\times$ & 0.47$_{\pm0.06}$$\times$ & 0.35$_{\pm0.06}$$\times$ & 0.40$_{\pm0.05}$$\times$ & 0.39$_{\pm0.07}$$\times$ & 0.14$_{\pm0.07}$$\times$ & 1.35$_{\pm0.11}$$\times$ & 1.56$_{\pm0.23}$$\times$ & N/A \\
 & P(True) & 1.46$_{\pm0.36}$$\times$ & 1.07$_{\pm0.11}$$\times$ & 0.60$_{\pm0.11}$$\times$ & 1.40$_{\pm0.14}$$\times$ & 2.05$_{\pm0.31}$$\times$ & 3.39$_{\pm1.62}$$\times$ & 2.33$_{\pm0.22}$$\times$ & N/A & N/A \\
 & AC sweep & \textbf{0.32$_{\pm0.09}$$\times$} & \textbf{0.26$_{\pm0.03}$$\times$} & 0.20$_{\pm0.04}$$\times$ & \textbf{0.27$_{\pm0.03}$$\times$} & \textbf{0.18$_{\pm0.03}$$\times$} & 0.12$_{\pm0.07}$$\times$ & 1.42$_{\pm0.14}$$\times$ & 1.85$_{\pm0.32}$$\times$ & 1.20$_{\pm0.75}$$\times$ \\
 & ESC sweep & 3.19$_{\pm1.47}$$\times$ & 6.11$_{\pm0.65}$$\times$ & N/A & 0.30$_{\pm0.03}$$\times$ & 0.34$_{\pm0.06}$$\times$ & 0.25$_{\pm0.21}$$\times$ & 2.39$_{\pm0.38}$$\times$ & 7.42$_{\pm2.14}$$\times$ & 7.95$_{\pm4.31}$$\times$ \\
 & PC-linear & 0.56$_{\pm0.15}$$\times$ & 0.51$_{\pm0.07}$$\times$ & 0.39$_{\pm0.06}$$\times$ & 0.34$_{\pm0.04}$$\times$ & 0.28$_{\pm0.04}$$\times$ & \textbf{0.09$_{\pm0.05}$$\times$} & 0.93$_{\pm0.07}$$\times$ & 0.80$_{\pm0.12}$$\times$ & 0.32$_{\pm0.19}$$\times$ \\
 & PC-quadratic & 0.44$_{\pm0.11}$$\times$ & 0.37$_{\pm0.04}$$\times$ & 0.19$_{\pm0.03}$$\times$ & 0.33$_{\pm0.03}$$\times$ & 0.28$_{\pm0.04}$$\times$ & 0.10$_{\pm0.05}$$\times$ & \textbf{0.86$_{\pm0.06}$$\times$} & \textbf{0.61$_{\pm0.09}$$\times$} & 0.20$_{\pm0.12}$$\times$ \\
 & PC-cubic & 0.44$_{\pm0.11}$$\times$ & 0.32$_{\pm0.04}$$\times$ & \textbf{0.16$_{\pm0.02}$$\times$} & 0.35$_{\pm0.04}$$\times$ & 0.30$_{\pm0.04}$$\times$ & 0.11$_{\pm0.06}$$\times$ & \textbf{0.86$_{\pm0.06}$$\times$} & 0.62$_{\pm0.09}$$\times$ & \textbf{0.18$_{\pm0.11}$$\times$} \\
\bottomrule
\end{tabular}%
}%
\end{table}

PC-cubic is more cost-efficient than Standard MV on $11$ out of $12$ model and benchmark settings at $\alpha{=}99\%$. The strongest savings are $0.05\times$ on both FrontierScience-Olympiad with GPT-OSS-120B and AIME~2025 with Ministral3-14B, and $0.08\times$ on FrontierScience-Olympiad with Ministral3-14B. All three correspond to cells with large $D$ (Table~\ref{tab:signal}).

The savings track the discrimination gap $D$ (Table~\ref{tab:signal}): pairs with the largest $D$ yield the largest reductions. PC-cubic offers little advantage over Standard MV on two cells (it underperforms on Nemotron3-30B AIME~2025 and only marginally beats Standard MV on Nemotron3-30B HMMT Feb~2026), both of which have a small Pass@1-to-plateau gap of at most $.11$ ($.902 \to .967$ and $.708 \to .810$), leaving little room above Pass@1 for any reweighting irrespective of $D$. The practical penalty in this regime, where Pass@1 is close to Standard MV plateau, is correspondingly small.

\paragraph{Comparison with adaptive-stopping baselines.}
PC-cubic is competitive with or better than AC sweep on most cells and outperforms ESC sweep at $\alpha{=}99\%$ on every (model, benchmark) cell except Nemotron3-30B on AIME~2025, despite being a non-adaptive reweighting of the same initial pool that AC and ESC consume sequentially. The advantage is most pronounced on large-$D$ cells (FrontierScience-Olympiad on every model, and most Ministral3-14B benchmarks), e.g.\ PC-cubic at $0.05\times$ vs.\ AC sweep at $0.29\times$ at $\alpha{=}99\%$ on GPT-OSS-120B FrontierScience-Olympiad.

AC substantially outperforms PC only on the two cells noted above where Pass@1 is close to Standard MV plateau (Nemotron3-30B on AIME~2025 and HMMT Feb~2026), since the aggregated vote on wrong answers is small and AC's early stop alone bounds the cost. However, the evaluated AC sweep does not reach the $\alpha{=}99\%$ target on $2$ out of $12$ (model, benchmark) cells (marked ``N/A'' in Table~\ref{tab:token_savings}). PC-cubic reaches the target on all $12$.

ESC stops at the first fixed-size window of samples that all share the same answer, which is too strict on benchmarks where wrong answers are diverse, so ESC either stops well after AC or fails to stop within the budget. PC and adaptive stopping act on orthogonal axes: PC reweights votes while AC and ESC decide when to stop sampling. A hybrid that votes by PC weights and stops by the AC rule would combine AC's cost bound on easy cells with PC's accuracy on difficult ones (left to future work).

\paragraph{Remark on the cost accounting.}
Treating log-probability access as free is implementation-dependent but holds for our vLLM setup. This favors baselines that read log-probabilities of the initial trace without generating extra tokens (Self-certainty, DeepConf, Response probability), while prefix consistency spends budget on the regeneration tokens. Even so, the PC-WMV has significant advantage. Imposing any cost for log-probability retrieval only widens the margin.

Additional results on token-efficiency ratios are reported per model in Appendix~\ref{app:all_baselines}.

\subsection{How Problem Difficulty Affects the Discrimination Gap \texorpdfstring{$D$}{D}}
\label{sec:what_drives}

The discrimination gap $D$ determines where prefix consistency improves WMV (Theorem~\ref{thm:improvement}), so we now study how $D$ varies with problem difficulty, indexed by Pass@1.

Figure~\ref{fig:pass1_vs_rates} plots per-problem $r_C$ and $r_W$ as a function of Pass@1 for three of the five models (the remaining two are in Appendix~\ref{app:pass1_vs_rates_glm}), stratified by category. Across all five models, $r_C$ rises with problem easiness, $r_W$ depends on the model and category but only weakly on Pass@1, and $D = r_C - r_W$ inherits both. We discuss $r_C$ and $r_W$ in turn below.

First, $r_C$ (solid lines) increases with Pass@1 for both categories. Pass@1 here indexes problem easiness within a fixed model, and on easier problems, the correct answer is more reliably reproduced under regeneration. Logistic generalized linear model (GLM) slopes $\beta(r_C)$ on $\text{logit}(r) = \beta_0 + \beta \cdot \text{Pass@1}$ range from $+2.7$ to $+5.1$ across the six (model, category) pairs shown in Figure~\ref{fig:pass1_vs_rates}, all significantly positive (cluster-bootstrap $p < 0.005$, see Appendix~\ref{app:pass1_vs_rates_glm}).

Second, $r_W$ (dashed lines) depends more weakly on Pass@1 than $r_C$: $|\beta(r_W)| \leq 1.14$ across the six pairs, and among four out of six curves, we cannot statistically reject $\beta = 0$ at a $2\sigma$-equivalent confidence level. For the remaining two non-zero slopes, $\beta$ is $+1.14$ (GPT-OSS-120B Math) and $-0.75$ (Ministral3-14B Math), with opposite signs, both smaller in magnitude than the smallest $r_C$ slope.
The level of $r_W$ varies by category, sitting at ${\sim}$15--40\% on Science and ${\sim}$25--60\% on Math. Math errors may reflect internally consistent miscalculations and science errors may reflect more diffuse knowledge gaps, but this is a hypothesis: our results establish $r_C > r_W$ as a behavioral regularity, and a mechanistic verification (e.g., calculation-heavy vs.\ concept-heavy subsets) is left to future work.

The finding that $D$ widens on easier cells may appear to conflict with the ``savings track $D$'' claim of Section~\ref{sec:token_efficiency}, since easier cells also have a smaller gap between Pass@1 and Standard MV plateau. The two reconcile by noting that PC-WMV's advantage over Standard MV depends on both $D$ and the gap above Pass@1: a large $D$ pays off only when there is room to reweight votes, which is why the two cells where PC-cubic offers little advantage over Standard MV are precisely the cells with Pass@1 concentrated within ${\sim}.10$ of Standard MV plateau (Nemotron3-30B on AIME~2025 and HMMT Feb~2026).

\begin{figure}[t]
\centering
\resizebox{\textwidth}{!}{%
  \includegraphics[trim=0 5bp 0 5bp,clip]{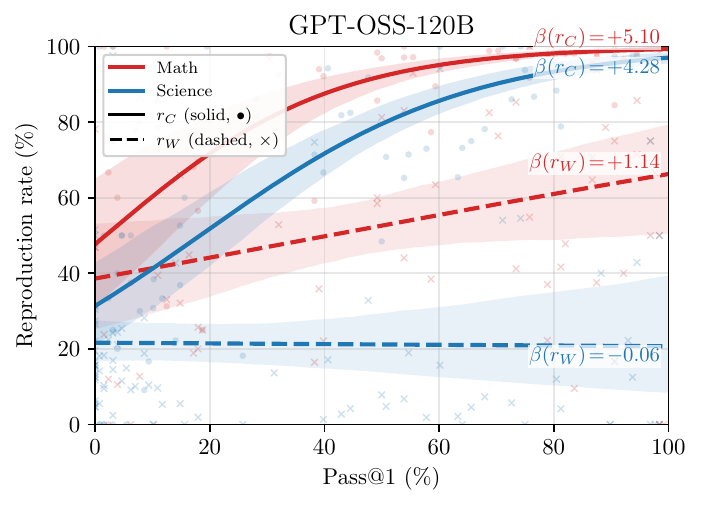}%
  \includegraphics[trim=0 5bp 0 5bp,clip]{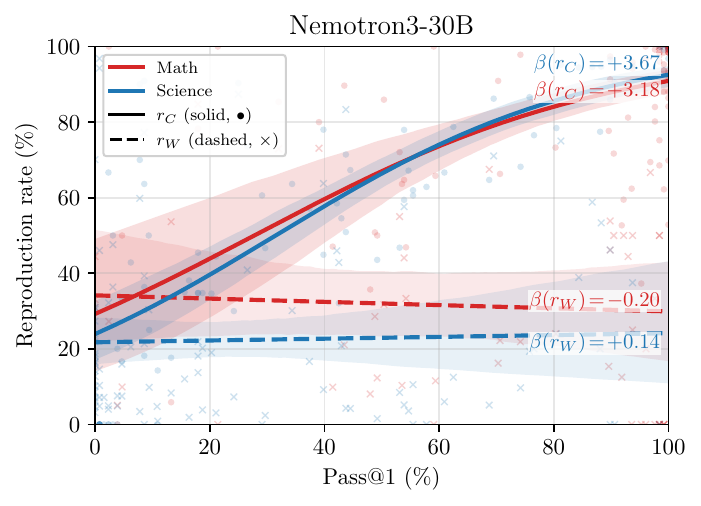}%
  \includegraphics[trim=0 5bp 0 5bp,clip]{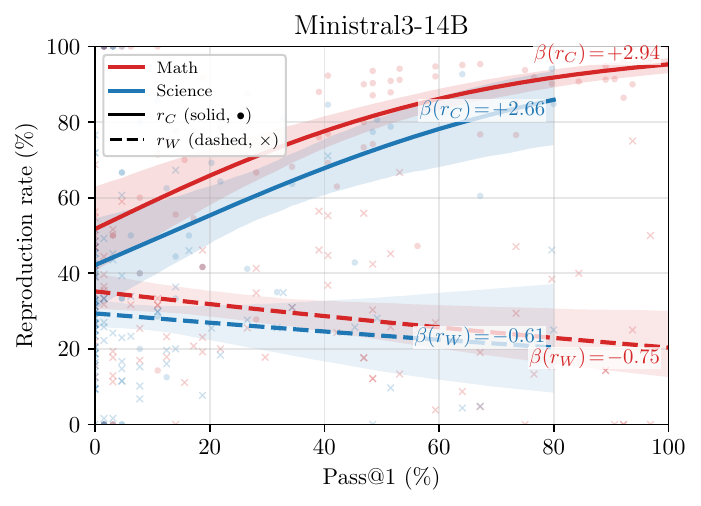}%
}
\caption{\textbf{Per-problem reproduction rates vs.\ Pass@1.} One panel per model. Solid ($\bullet$): $r_C$. Dashed ($\times$): $r_W$. Both are per-category logistic-regression fits with shaded pointwise cluster-bootstrap confidence intervals (CIs) over problems and per-problem scatter overlays; curves are labeled with the GLM slope $\beta$. $r_C > r_W$ holds across the full Pass@1 range, including below 50\%. Science: FrontierScience-Olympiad. Math: HMMT Feb~2026 $\cup$ AIME~2025 $\cup$ Brumo~2025.}
\label{fig:pass1_vs_rates}
\end{figure}

\section{Conclusion}
\label{sec:conclusion}

We introduced prefix consistency, a reliability signal for weighted majority voting that truncates each CoT and checks whether answers reproduce under regeneration.
Across benchmarks, prefix consistency is a stronger correctness predictor than existing baselines, and PC-WMV improves upon existing weighted majority voting methods under cost-equivalent comparison, especially on more difficult benchmarks.
Our analysis highlights the discrimination gap $D = r_C - r_W$ as the key quantity governing when the method helps: PC-WMV is most effective when $D > 0$ and Pass@1 leaves a meaningful gap below Standard MV plateau.
Regeneration stability is thus a practically useful test-time signal for aggregating votes, not merely a descriptive property of chain-of-thought.

\begin{ack}
J. Komiyama was supported by the MBZUAI Start-up Fund [BF0121].
\end{ack}

\bibliographystyle{plainnat}
\bibliography{references}

\appendix
\startcontents[appendix]

\section*{Appendix Contents}
\renewcommand{\contentsname}{Appendix Contents}
\setcounter{tocdepth}{2}
\begingroup
\small
\linespread{0.9}\selectfont
\printcontents[appendix]{}{1}{}
\endgroup
\clearpage

\section{Related Work}
\label{sec:related}

\subsection{Test-Time Scaling}
\label{sec:related_scaling}

\paragraph{Sequential and parallel scaling.}
Test-time compute scaling comes in two paradigms. Sequential scaling extends a single reasoning chain, for example by appending continuation tokens to lengthen CoT~\citep{muennighoff-etal-2025-s1}. Parallel scaling generates multiple independent solutions and aggregates them: self-consistency is the canonical instance, while tree search~\citep{yao2023treethoughtsdeliberateproblem} branches over candidate steps and process-reward verifiers rerank parallel samples~\citep{lightman2024lets}. \citet{parashar2025sys2bench} benchmarked a range of inference-time techniques across reasoning and planning tasks and reported that no single method is consistently best.
Our method sits between the two: it uses parallel samples but introduces a structured perturbation (truncation + regeneration) that probes the robustness of each sample's reasoning.
Beyond aggregation, the same perturbation could serve as a quality estimate for partial reasoning states inside richer search procedures such as Graph-of-Thoughts~\citep{besta2024graphofthoughts}, which extends Tree-of-Thoughts to a directed graph with refinement and aggregation operations.

\paragraph{Majority voting.}
\citet{wang2023selfconsistency} introduced self-consistency for CoT reasoning: sample multiple reasoning traces and take the majority vote over extracted answers.
A line of work targets the cost of self-consistency by stopping early. \citet{aggarwal-etal-2023-lets} introduced Adaptive Consistency, which evaluates a Beta-binomial posterior over the running top-1 and top-2 answer counts and stops once the posterior margin of the top answer crosses a threshold. \citet{li2024escape} introduced Early-Stopping Self-Consistency, which inspects fixed-size windows of samples for unanimity and stops at the first unanimous window.
\citet{sharma2025sequentialedge} showed that sequential, entropy-aware voting can outperform parallel self-consistency at matched compute, highlighting the importance of cost-equivalent comparisons.
\citet{komiyama2026boinf} analyzed the asymptotic behavior of majority voting (best-of-$\infty$) and proposed a Bayesian nonparametric stopping rule.
Our work is complementary: we assess the quality of each vote through regeneration weighting, which is compatible with the adaptive stopping rules above.

\paragraph{Confidence-based weighting.}
\citet{kang2025selfcertainty} proposed \emph{self-certainty}, a trace-level quality score defined as the mean KL divergence of the per-token output distribution from uniform, and used it to rerank best-of-$N$ candidates.
\citet{fu2025deepthinkconfidence} then extended self-certainty.
In addition to the global trace mean (equivalent to self-certainty), they introduced a sliding-window group confidence and proposed the use of their bottom 10\% (bottom-10\% group confidence) and the mean over the final tokens of a trace (tail confidence) as alternative trace-level scores for confidence-weighted majority voting.
\citet{taubenfeld2025confidenceimprovesselfconsistency} demonstrated that combining self-assessed confidence (most effectively via P(True), the model's own probability that its answer is correct) with self-consistency improves accuracy, though this still relies on the model's introspective calibration.
\citet{xiong2024llmsexpressuncertainty} studied whether LLMs can express uncertainty and found that verbalized confidence is often poorly calibrated.
\citet{scalena2025eagerentropyawaregenerationadaptive} used per-token entropy to allocate compute adaptively during generation, a token-level compute-allocation complement to our sample-level reliability estimation.
All of these signals rely on the model's own log-probabilities or verbalized self-assessment. We show that such confidence signals are unreliable on difficult problems and propose an alternative based on prefix self-reproduction that requires neither.

\paragraph{Prefix continuation as a per-sample signal.}
Several recent methods regenerated from a truncated reasoning prefix or extracted a structural signal from a CoT and used it to select or weight samples.
\citet{hammoud2025beyondlastanswer} proposed SubthoughtReasoner, which segments a single greedy trace into sequential subthoughts at linguistic cues and takes the mode over per-subthought regenerations as the final answer (single-trace refinement, +13pp on AIME~2024 vs.\ the same trace's last answer).
\citet{faria2025sampledontsearchrethinking} also regenerated from intermediate states, but for test-time alignment via Metropolis-Hastings (building on quality-aware machine-translation sampling~\citep{faria2024quest}) rather than answer aggregation.
\citet{zhu2024pathconsistency} iteratively promoted a high-confidence prefix among $N$ initial partial answers via Beta-gated agreement on the extracted-answer distribution and resampled subsequent branches conditioned on it (up to 40.5\% latency reduction at matched accuracy).
\citet{jindal2026the} clustered the first 256 tokens of $N$ initial samples and expanded only the dominant cluster (up to 60\% token reduction at matched accuracy).
\citet{feng2025whatcharacterizes} scored each of $N{=}64$ traces by its Failed-Step Fraction (the proportion of abandoned reasoning branches) and selected the lowest-FSF trace, yielding 5\% to 13\% improvements over random selection on AIME~2025.
The latter three collapsed $N$ samples onto a subset (a single best prefix, the dominant cluster, or one selected trace) and therefore inherited the failure mode of Standard MV when the subset misses the correct answer: \citet{jindal2026the} reported a 10pp drop on AIME~2025 (76.7\% to 66.7\%) where the largest prefix cluster need not contain the correct answer.
PC-WMV moves in the converse direction: it keeps all $N$ samples and \emph{reweights} the vote by each sample's prefix self-reproduction, so a minority sample whose prefix continuation reproduces its own answer is upweighted rather than discarded.
The two directions are complementary, since prefix-cluster pruning could feed PC-WMV reweighting on the surviving traces.
Under cost-equivalent comparison, SubthoughtReasoner does not consistently improve over Standard MV in our experiments, filling the comparison left open by \citet{hammoud2025beyondlastanswer}.

\subsection{Chain-of-Thought}
\label{sec:related_dynamics}

\paragraph{Internal answer determination.}
Several lines of evidence suggest that LLMs determine their answer internally before the visible CoT concludes.
\citet{lanham2023measuringfaithfulness} found that truncating a CoT mid-reasoning and forcing an early answer often does not change the prediction, especially for larger models.
\citet{boppana2026reasoningtheater} confirmed this with attention probes: on easy problems, the internal answer is determined well before the visible reasoning ends.
\citet{zhang2025reasoningmodelsknow} showed that linear probes on hidden states can classify the correctness of the future final answer at intermediate CoT stages.
These observations are consistent with the $r_C > r_W$ asymmetry we measure.

\paragraph{Error propagation.}
In CoT reasoning, errors compound across steps: once an error is introduced, subsequent reasoning builds on it~\citep{gan2025rethinkingexternalslowthinking}.
\citet{kim2025metastabledynamicschainofthought} modeled CoT as a metastable Markov process on a reasoning graph, where dense intra-cluster (easy) and sparse inter-cluster (difficult) steps induce timescale separation. Our $r_C$ and $r_W$ are analogous in spirit to within-cluster persistence rates, although we measure them at the answer level under regeneration rather than at the reasoning-step level along a single trace.

\paragraph{Reasoning robustness.}
Several works have studied the robustness of LLM reasoning.
\citet{yu2025benchmarkingreasoningrobustness} found that reasoning models can be brittle to minor perturbations in the input.
\citet{vonrecum2026robust} systematically evaluated seven intervention types on open-weight reasoning LLMs and found that robustness degrades more when interventions occur early in the CoT.
\citet{jiang2025robustanswers} showed that correct answers persist even when reasoning logic is perturbed, suggesting a decoupling between answer stability and reasoning faithfulness.
Our work uses this asymmetry between correct and wrong reasoning traces as a practical signal for answer aggregation.

\paragraph{Overthinking and length scaling.}
Reasoning models tend to allocate compute disproportionately to problem difficulty, which limits the improvements from extending a single CoT.
\citet{chen2025overthinking} reported that o1-like models generate up to 1953\% more tokens than non-thinking models on trivial arithmetic, and reach the correct answer in the first generated solution in over 92\% of cases while later solutions still account for roughly 40\% of tokens.
\citet{pu2025thoughtterminator} introduced DUMB500 to quantify overthinking on easy problems and proposed ThoughtTerminator, a training-free decoding-time termination scheme that reduces overthinking tokens by 76\% to 98\% with minimal accuracy loss.
\citet{aggarwal2026optimalthinkingbench} formalized the trade-off as a joint over- and underthinking benchmark and found that no current model balances the two: even o3 reaches only 71.1\% on their unified score.
These results motivate aggregating multiple short samples rather than extending a single long trace, which is the setting our method targets.

\section{Limitations and Future Work}\label{sec:limitation}
The effectiveness of prefix consistency depends on the discrimination gap $D = r_C - r_W$. When wrong answers are themselves stable under regeneration, $r_W$ remains high, and the advantage of PC-WMV over Standard MV becomes small (e.g., on very difficult problems). Likewise, when the correct answer is reproduced only weakly from the chosen prefix, $r_C$ remains low, and the signal becomes less informative. In this sense, the method is most effective in regimes where regeneration behavior meaningfully separates correct from wrong traces.

The practical advantage of PC-WMV also depends on the available room above Pass@1. When Pass@1 is already close to Standard MV plateau, the aggregated vote on wrong answers is small, so even a strong reliability signal yields only limited accuracy improvements. This explains the small-gap settings in our experiments where PC-WMV does not outperform Standard MV (or simple verbalized-confidence baselines) by a large margin (Appendix~\ref{app:all_baselines}).

Although prefix consistency does not require token log-probabilities, it is not without cost: it incurs additional inference cost through prefix truncation and regeneration. This trade-off may be less favorable in deployments where log-probability access is readily available and inexpensive. Relatedly, the method introduces hyperparameters such as the truncation fraction $\tau$, the number of regenerations $K$, and the weighting function $w$, and in this work, we use fixed defaults rather than selecting them adaptively for each model or task.

A further practical limitation is that our method assumes access to an explicit Chain-of-Thought trace that can be truncated and continued from a prefix. This assumption does not hold for many frontier closed models and commercial APIs, including systems where the full internal reasoning trace is hidden or only summarized. In such settings, prefix consistency cannot be applied directly, even if the model can generate correct final answers. Extending the method to settings without visible CoT, for example by using intermediate summaries, structured scratchpads, or other externally exposed reasoning states, is an important direction for future work.

More broadly, our results establish prefix consistency as a useful behavioral signal, but not as a mechanistic explanation of why correct traces are more reproducible than wrong ones. Understanding the origin of the observed asymmetry $r_C > r_W$ remains future work. Finally, our evaluation is limited to reasoning-oriented models and math/science-style benchmarks, so the extent to which the same phenomenon holds for other domains, model families, or API settings remains to be established.

As shown in Section~\ref{sec:what_drives}, $r_C$ rises with problem easiness, $r_W$ depends on the model and category but only weakly on Pass@1, and $D = r_C - r_W$ inherits both. This pattern is consistent with reinforcement learning with verifiable rewards (RLVR) encouraging reasoning that survives verification \citep{wen2025rlvr}. State-of-the-art LLMs are designed to increase the mass of correct reasoning paths via widening the discrimination gap $D$. We do not test this connection directly, and whether the same pattern holds for non-RLVR-trained reasoners is left to future work.

\section{Asymptotic Analysis}
\label{app:theory}

We analyze the asymptotic convergence of PC-WMV, prove Theorem~\ref{thm:improvement} along the way, and verify the required assumptions empirically.
\subsection{Notation and Useful Tools}
\label{app:notation}
Assume that the answer space $\mathcal A$ is finite throughout this appendix. Fix $\tau \in (0, 1)$ throughout, and omit it from subscripts: e.g., write $\tilde a_i$ for $\tilde a_i^{(\tau)}$ and $c_i(a)$ for $c_i^{(\tau)}(a)$. The score takes values $c_i(a) \in \{0, \tfrac12, 1\}$.

We define the transition probability as follows:
\begin{equation}
\label{eq:transition_kernel}
    T(b \rightarrow a) \;:=\; \Pr\!\bigl[\tilde a_i = a \,\big|\, a_i = b\bigr].
\end{equation}
The reproduction rates of Section~\ref{sec:prefix_consistency} can be written in terms of $T$:
\begin{equation}
\label{eq:rates_from_kernel}
    r_C \;=\; T(a^\star \rightarrow a^\star),
    \qquad
    r_W \;=\; \mathbb E\!\left[T(a_i \rightarrow a_i) \,\big|\, a_i \neq a^\star\right].
\end{equation}
Let $\pi$ and $\pi^{\rightarrow}$ be the marginal distributions of $a_i$ and $\tilde a_i$, respectively. Here $\pi(a^\star)$ is the same per-problem Pass@1 introduced in Section~\ref{sec:prelim} and used in Theorem~\ref{thm:improvement}. We keep the $\pi$ notation in this appendix so the population-level expressions in Eq.~\eqref{eq:pi_rho} below align with the proof.
\begin{equation}
\label{eq:pi_rho}
    \pi(a) \;:=\; \Pr[a_i = a],
    \qquad
    \pi^{\rightarrow}(a) \;:=\; \Pr[\tilde a_i = a] \;=\; \sum_{b \in \mathcal A} \pi(b)\, T(b \rightarrow a).
\end{equation}
All quantities above ($T$, $r_C$, $r_W$, $\pi$, $\pi^{\rightarrow}$) are defined per problem. The results in this section are statements about a fixed problem $q$, and the i.i.d.\ assumption across $i$ refers to independent samples within $q$.

For completeness, we restate the definitions of the MV and PC-WMV votes.
Let $w \colon [0,1] \to \mathbb{R}_{\ge 0}$ satisfy $w(0)=0$.
For each $a \in \mathcal A$, define the aggregated PC-WMV score over $N$ groups by
\begin{equation*}
V_N^{(w)}(a) := \sum_{i=1}^{N} w\bigl(c_i(a)\bigr),
\end{equation*}
where each term is the per-group PC-WMV vote given in Eq.~\eqref{eq:pc_weight}.
We then define the PC-WMV estimator and MV voting methods by
\begin{equation}
\label{eq:estimators}
\hat a^{\mathrm{PC}}_N \;:=\; \argmax_{a \in \mathcal A} V_N^{(w)}(a),
\qquad
\hat a^{\mathrm{MV}}_N \;:=\; \argmax_{a \in \mathcal A}\, \frac{1}{N}\sum_{i=1}^{N}\mathbf{1}\{a_i=a\}.
\end{equation}

We next identify the population objective associated with the PC-WMV estimator.

\begin{proposition}[Population objective]
\label{prop:population_objective}
Assume that the pairs $(a_i,\tilde a_i)$ are i.i.d.\ across $i$, and that, conditional on $a_i$, the variable $\tilde a_i$ is drawn from $T(a_i \rightarrow \cdot)$.
Then, for every $a \in \mathcal A$,
\begin{equation}
\label{eq:population_objective}
\frac{1}{N}V_N^{(w)}(a)
\;\xrightarrow{\mathrm{a.s.}}\;
\Phi_w(a)
\;:=\;
\sum_{b \in \mathcal A}
\pi(b)\,
\mathbb E_{Z \sim \mathrm{Bern}(T(b \rightarrow a))}\!\left[
w\!\left(\frac{\mathbf 1\{a=b\}+Z}{2}\right)
\right].
\end{equation}
Moreover, if $\Phi_w$ has a unique maximizer, then $\hat a_N^{\mathrm{PC}}$ converges to that maximizer almost surely.
\end{proposition}

\begin{proof}
For each fixed $a \in \mathcal A$, the random variables $\{w(c_i(a))\}_{i=1}^N$ are i.i.d., so the strong law of large numbers yields
\begin{equation*}
\frac{1}{N}V_N^{(w)}(a)
\;\xrightarrow{\mathrm{a.s.}}\;
\mathbb E[w(c_i(a))].
\end{equation*}
Conditional on $a_i=b$, we have $\mathbf 1\{\tilde a_i=a\} \sim \mathrm{Bern}(T(b \rightarrow a))$, and
\begin{equation*}
2c_i(a)=\mathbf 1\{a=b\}+\mathbf 1\{\tilde a_i=a\}.
\end{equation*}
Hence,
\begin{equation*}
\mathbb E[w(c_i(a))\mid a_i=b]
=
\mathbb E_{Z \sim \mathrm{Bern}(T(b \rightarrow a))}\!\left[
w\!\left(\frac{\mathbf 1\{a=b\}+Z}{2}\right)
\right].
\end{equation*}
Taking expectation with respect to $b \sim \pi$ gives Eq.~\eqref{eq:population_objective}. For the final claim, $\frac{1}{N} V_N^{(w)}(a) \to \Phi_w(a)$ a.s.\ for each $a$ in the finite set $\mathcal A$, so when $\Phi_w$ has a unique maximizer $a^\star$, eventually $\hat a_N^{\mathrm{PC}} = \argmax_a V_N^{(w)}(a) = a^\star$ almost surely.
\end{proof}

We next derive an explicit decomposition of $\Phi_w$ that exposes its dependence on the marginal $\pi(a)$ and the self-transition $T(a \rightarrow a)$.

\begin{proposition}[Exact decomposition of $\Phi_w$]
\label{prop:exact_decomposition}
For every $a \in \mathcal A$,
\begin{equation}
\label{eq:exact_decomposition}
    \Phi_w(a) \;=\; w(\tfrac12)\,\bigl[\pi(a) + \pi^{\rightarrow}(a)\bigr] \;+\; \lambda_w\, \pi(a)\, T(a \rightarrow a),
    \qquad \lambda_w \;:=\; w(1) - 2\, w(\tfrac12).
\end{equation}
\end{proposition}

\begin{proof}
Fix $a \in \mathcal A$.
Since $c_i(a)$ is the average of the two indicators
\begin{equation*}
\mathbf{1}\{a_i=a\}
\qquad\text{and}\qquad
\mathbf{1}\{\tilde a_i=a\},
\end{equation*}
it takes only the three values $0$, $\tfrac12$, and $1$.

More precisely, $c_i(a)=1$ if and only if both $a_i=a$ and $\tilde a_i=a$ hold.
Thus,
\begin{equation*}
\Pr\bigl(c_i(a)=1\bigr)=\pi(a)\,T(a \rightarrow a).
\end{equation*}

Next, $c_i(a)=\tfrac12$ if and only if exactly one of the two events
$\{a_i=a\}$ and $\{\tilde a_i=a\}$ occurs.
Therefore,
\begin{equation*}
\Pr\bigl(c_i(a)=\tfrac12\bigr)=\Pr(a_i=a,\tilde a_i\neq a)+\Pr(a_i\neq a,\tilde a_i=a).
\end{equation*}
The first term is
\begin{equation*}
\Pr(a_i=a,\tilde a_i\neq a)=\pi(a)\bigl(1-T(a \rightarrow a)\bigr).
\end{equation*}
For the second term, recall that $\pi^{\rightarrow}(a)=\Pr(\tilde a_i=a)$, so
\begin{equation*}
\Pr(a_i\neq a,\tilde a_i=a)=\Pr(\tilde a_i=a)-\Pr(a_i=a,\tilde a_i=a)=
\pi^{\rightarrow}(a)-\pi(a)T(a \rightarrow a).
\end{equation*}
Combining the two expressions yields
\begin{equation*}
\Pr\bigl(c_i(a)=\tfrac12\bigr)=\pi(a)+\pi^{\rightarrow}(a)-2\pi(a)T(a \rightarrow a).
\end{equation*}

Finally, since $w(0)=0$ and $c_i(a)\in\{0,\tfrac12,1\}$,
\begin{equation*}
\Phi_w(a)=\mathbb{E}[w(c_i(a))]=
w(1)\Pr\bigl(c_i(a)=1\bigr)+
w(\tfrac12)\Pr\bigl(c_i(a)=\tfrac12\bigr).
\end{equation*}
Substituting the above probabilities, we obtain
\begin{equation*}
\Phi_w(a)
=w(1)\pi(a)T(a \rightarrow a)+
w(\tfrac12)\bigl[\pi(a)+\pi^{\rightarrow}(a)-2\pi(a)T(a \rightarrow a)\bigr].
\end{equation*}
Rearranging terms gives
\begin{equation*}
\Phi_w(a)
=w(\tfrac12)\bigl[\pi(a)+\pi^{\rightarrow}(a)\bigr]+\bigl(w(1)-2w(\tfrac12)\bigr)\pi(a)T(a \rightarrow a),
\end{equation*}
which is exactly Eq.~\eqref{eq:exact_decomposition}.
\end{proof}

\subsection{Asymptotic Convergence of PC-WMV}
\label{app:convergence}

Throughout this subsection we assume the i.i.d.\ setup of Proposition~\ref{prop:population_objective}.

By Proposition~\ref{prop:exact_decomposition}, applied to $a^\star$ and to any wrong $a$ and subtracting,
\begin{equation}
\label{eq:margin_decomposition}
\begin{aligned}
\Phi_w(a^\star) - \Phi_w(a) \;=\;& \underbrace{w(\tfrac12)\,\bigl[(\pi(a^\star)+\pi^{\rightarrow}(a^\star)) - (\pi(a)+\pi^{\rightarrow}(a))\bigr]}_{\text{pooled-mass term}} \\
&+ \underbrace{\lambda_w\,\bigl[\pi(a^\star)\,T(a^\star \rightarrow a^\star) - \pi(a)\,T(a \rightarrow a)\bigr]}_{\text{self-reproduction term}}.
\end{aligned}
\end{equation}
By Proposition~\ref{prop:population_objective}, identifying $a^\star$ as the population maximizer requires this difference to be positive for every wrong $a$. The two assumptions below control the two terms separately.

\begin{assumption}[Self-reproduction dominance]
\label{ass:self}
For every $a \neq a^\star$,
\begin{equation}
\label{eq:self}
\pi(a^\star)\,T(a^\star \rightarrow a^\star) > \pi(a)\,T(a \rightarrow a).
\end{equation}
\end{assumption}

\paragraph{Interpretation.}
$\pi(a^\star)\,T(a^\star \rightarrow a^\star)$ is the population mass of groups whose initial answer is correct and is reproduced correctly, and $\pi(a)\,T(a \rightarrow a)$ is the corresponding mass for the wrong answer $a$ that is reproduced as itself. Assumption~\ref{ass:self} requires that correct self-reproduction dominate every individual wrong self-reproduction. Equivalently, it makes the self-reproduction term in the decomposition strictly positive whenever $\lambda_w > 0$.

\begin{assumption}[Pooled-mass dominance]
\label{ass:pool}
For every $a \neq a^\star$,
\begin{equation}
\label{eq:pool}
\pi(a^\star)+\pi^{\rightarrow}(a^\star) > \pi(a)+\pi^{\rightarrow}(a).
\end{equation}
\end{assumption}

\paragraph{Interpretation.}
$\pi(a)+\pi^{\rightarrow}(a)$ is the total population mass on candidate $a$, combining initial occurrences and regeneration arrivals. Assumption~\ref{ass:pool} requires that $a^\star$ have the largest such total. By the same decomposition, it makes the pooled-mass term strictly positive whenever $w(\tfrac12)>0$.

\begin{theorem}[Asymptotic convergence of PC-WMV]
\label{thm:convergence}
Let $w$ satisfy $w(0)=0$, $w(1)>0$, and $\lambda_w := w(1)-2w(\tfrac12) \ge 0$. Then:
\begin{enumerate}[leftmargin=*,label=(\alph*)]
    \item If Assumptions~\ref{ass:self} and \ref{ass:pool} hold, then $\hat a_N^{\mathrm{PC}} \to a^\star$ almost surely.
    \item If $w(\tfrac12)=0$, then Assumption~\ref{ass:pool} is unnecessary: Assumption~\ref{ass:self} alone implies $\hat a_N^{\mathrm{PC}} \to a^\star$ almost surely.
\end{enumerate}
\end{theorem}

\begin{proof}
Fix any wrong $a \neq a^\star$. By Assumption~\ref{ass:self}, the self-reproduction bracket of Eq.~\eqref{eq:margin_decomposition} is strictly positive, so the self-reproduction term is nonnegative and is strictly positive whenever $\lambda_w > 0$.

For part (a), Assumption~\ref{ass:pool} makes the pooled-mass bracket of Eq.~\eqref{eq:margin_decomposition} strictly positive, so the pooled-mass term is nonnegative and is strictly positive whenever $w(\tfrac12) > 0$. Since $w(1) = \lambda_w + 2w(\tfrac12) > 0$, at least one of $\lambda_w$ and $w(\tfrac12)$ is strictly positive, so $\Phi_w(a^\star) > \Phi_w(a)$ for every wrong $a$. Hence $a^\star$ is the unique maximizer of $\Phi_w$, and Proposition~\ref{prop:population_objective} yields $\hat a_N^{\mathrm{PC}} \to a^\star$ almost surely.

For part (b), if $w(\tfrac12) = 0$ the pooled-mass term vanishes identically and $\lambda_w = w(1) > 0$, so the strict positivity of the self-reproduction term suffices to give $\Phi_w(a^\star) > \Phi_w(a)$ for every wrong $a$. Thus $\hat a_N^{\mathrm{PC}} \to a^\star$ almost surely without Assumption~\ref{ass:pool}.
\end{proof}

\subsection{Proof of Theorem~\ref{thm:improvement}}
\label{app:binary}

We prove Theorem~\ref{thm:improvement} by specializing the exact margin to the binary case $\mathcal A = \{a^\star, a'\}$, where $a^\star$ is the correct answer and $a'$ is the only wrong answer (so $\pi(a') = 1 - \pi(a^\star)$). In this case, the two assumptions of Theorem~\ref{thm:convergence} reduce to the same condition and the exact margin takes a one-bracket form, from which the asymptotic boundary $\pi(a^\star) > r_W/(r_C+r_W)$ follows.

\paragraph{The two assumptions coincide.}
With the only wrong answer being $a'$ and $\pi(a') = 1 - \pi(a^\star)$, Assumption~\ref{ass:self} becomes the binary boundary
\begin{equation}
\label{eq:binary_boundary}
\pi(a^\star)\,T(a^\star \rightarrow a^\star) > (1-\pi(a^\star))\,T(a' \rightarrow a').
\end{equation}
For Assumption~\ref{ass:pool}, expand $\pi^{\rightarrow}$ via Eq.~\eqref{eq:pi_rho}, using $T(a^\star \rightarrow a') = 1 - T(a^\star \rightarrow a^\star)$ and $T(a' \rightarrow a^\star) = 1 - T(a' \rightarrow a')$ (each row of $T$ sums to one):
\begin{align*}
\pi^{\rightarrow}(a^\star) &= \pi(a^\star)\,T(a^\star \rightarrow a^\star) + (1-\pi(a^\star))\bigl(1 - T(a' \rightarrow a')\bigr), \\
\pi^{\rightarrow}(a') &= \pi(a^\star)\bigl(1 - T(a^\star \rightarrow a^\star)\bigr) + (1-\pi(a^\star))\,T(a' \rightarrow a').
\end{align*}
Adding $\pi(a^\star)$ and $\pi(a') = 1 - \pi(a^\star)$ and simplifying,
\begin{align*}
\pi(a^\star) + \pi^{\rightarrow}(a^\star) &= 1 + \bigl[\pi(a^\star)\,T(a^\star \rightarrow a^\star) - (1-\pi(a^\star))\,T(a' \rightarrow a')\bigr], \\
\pi(a') + \pi^{\rightarrow}(a') &= 1 - \bigl[\pi(a^\star)\,T(a^\star \rightarrow a^\star) - (1-\pi(a^\star))\,T(a' \rightarrow a')\bigr].
\end{align*}
Subtracting,
\begin{equation*}
\bigl[\pi(a^\star)+\pi^{\rightarrow}(a^\star)\bigr] - \bigl[\pi(a')+\pi^{\rightarrow}(a')\bigr] = 2\,\bigl[\pi(a^\star)\,T(a^\star \rightarrow a^\star) - (1-\pi(a^\star))\,T(a' \rightarrow a')\bigr].
\end{equation*}
Hence Assumptions~\ref{ass:self} and~\ref{ass:pool} reduce to the same condition Eq.~\eqref{eq:binary_boundary}. When $\lambda_w \geq 0$, Theorem~\ref{thm:convergence}\,(a) gives $\hat a_N^{\mathrm{PC}} \to a^\star$ a.s.\ whenever Eq.~\eqref{eq:binary_boundary} holds. The exact margin below establishes the result without this restriction on $\lambda_w$.

\paragraph{The exact margin.}
Applying Eq.~\eqref{eq:margin_decomposition} with $a = a'$, the self-reproduction term contributes $\lambda_w$ times the quantity $\pi(a^\star)\,T(a^\star \rightarrow a^\star) - (1-\pi(a^\star))\,T(a' \rightarrow a')$. By the identity above, the pooled-mass term contributes $2w(\tfrac12)$ times the same bracket. Their sum is $w(1) = \lambda_w + 2w(\tfrac12)$ times the bracket:
\begin{equation}
\label{eq:single_wrong_margin}
\Phi_w(a^\star) - \Phi_w(a') = w(1)\,\bigl[\pi(a^\star)\,T(a^\star \rightarrow a^\star) - (1-\pi(a^\star))\,T(a' \rightarrow a')\bigr].
\end{equation}
The choice of $w$ thus enters only through $w(1)$.

\begin{proof}
By Eq.~\eqref{eq:single_wrong_margin} and $w(1) > 0$, $\Phi_w(a^\star) > \Phi_w(a')$ if and only if
\begin{equation*}
\pi(a^\star)\,T(a^\star \rightarrow a^\star) > (1-\pi(a^\star))\,T(a' \rightarrow a').
\end{equation*}
By Proposition~\ref{prop:population_objective}, a positive margin implies $\hat a_N^{\mathrm{PC}} \to a^\star$ almost surely, while a negative margin implies convergence to $a'$. At equality, a group's signed vote is $+w(1)$ or $-w(1)$ with equal positive probabilities, and zero otherwise. Its cumulative sum is a non-degenerate lazy symmetric random walk, which changes sign infinitely often almost surely, so the estimator does not converge to $a^\star$. Identifying $r_C := T(a^\star \rightarrow a^\star)$ and $r_W := T(a' \rightarrow a')$ via Eq.~\eqref{eq:rates_from_kernel} rewrites the threshold as
\begin{equation*}
\pi(a^\star) > \frac{r_W}{r_C + r_W}.
\end{equation*}
Standard MV's vote count $\frac{1}{N}\sum_{i=1}^{N} \mathbf 1\{a_i = a\}$ converges to $\pi(a)$ a.s.\ by the strong law of large numbers. At $\pi(a^\star)=\tfrac12$, its signed vote count is a symmetric random walk and does not converge to either answer. Thus Standard MV converges to $a^\star$ almost surely if and only if $\pi(a^\star) > \tfrac12$. Hence PC-WMV converges where Standard MV does not on the interval $r_W/(r_C+r_W) < \pi(a^\star) \le \tfrac12$, which has positive length if and only if $r_C > r_W$.
\end{proof}

\subsection{Empirical Verification of Assumptions}
\label{app:verification}

Table~\ref{tab:assumption} checks Assumption~\ref{ass:self} (A1) and Assumption~\ref{ass:pool} (A2) empirically on problems with at least one correct and one wrong initial sample. This is the same observed-sample inclusion rule used for $\overline{\mathrm{AUROC}}$ (Section~\ref{sec:signal}, Appendix~\ref{app:auroc_eval}), and permits estimation of both reproduction rates. All-correct or all-wrong observed pools are excluded because one empirical reproduction rate is unavailable, not because the population probability is known to equal an endpoint. The table further reports the per-problem indicator $\Delta_{w^{(n)}} := \min_{a \neq a^\star}\bigl[\Phi_{w^{(n)}}(a^\star) - \Phi_{w^{(n)}}(a)\bigr]$ for $w^{(n)}(c) = c^n$, $n \in \{1, 2, 3\}$ ($\Delta_{w^{(n)}} > 0$ is equivalent to $a^\star$ being the unique maximizer of $\Phi_{w^{(n)}}$).

$\Pr[\mathrm{A1}]$ is non-negligible across a wide range of benchmarks and models, so $\mathrm{A1}$ is not a rare or pathological event in practice. Conditional on $\mathrm{A1}$, $\Pr[\mathrm{A2}\mid \mathrm{A1}]$ is at least $87.5\%$ on every cell and reaches $100\%$ on $40\%$ of them, so the joint occurrence of $\mathrm{A1}$ and $\mathrm{A2}$ is common. For linear weighting ($n=1$), $\lambda_w=0$, so $\Delta_{w^{(1)}}>0$ is exactly $\mathrm{A2}$, explaining the identical columns. For quadratic and cubic weighting ($n \in \{2,3\}$), the last two columns show that the positive gap occurs with high probability conditional on $\mathrm{A1}$, even when $\mathrm{A2}$ is not required. These findings indicate that the assumptions reflect patterns that arise frequently in the experimental data.

\begin{table}[H]
\centering
\caption{\textbf{Empirical verification of Theorem~\ref{thm:convergence}'s assumptions} on $\mathcal{Q}'$, checking $\mathrm{A1}$ and $\mathrm{A2}$ for every wrong $a$. All probabilities are macro-averaged: each problem in $\mathcal{Q}'$ contributes one indicator per event. Theorem~\ref{thm:convergence} predicts that $a^\star$ uniquely maximizes $\Phi_w$ under $\mathrm{A1}\cap\mathrm{A2}$ for every convex nonnegative $w$ with $w(0)=0$ and $w(1)>0$. The last three columns report $\Pr[\Delta_{w^{(n)}} > 0 \mid \mathrm{A1}]$ for $w^{(n)}(c) = c^n$, $n \in \{1, 2, 3\}$. At $n=1$, $\Delta_{w^{(1)}}>0$ is exactly $\mathrm{A2}$; only $n=2,3$ can relax that condition under $\mathrm{A1}$.}
\label{tab:assumption}
\small
\resizebox{\textwidth}{!}{%
\begin{tabular}{@{}llrrrrrr@{}}
\toprule
\multirow{2}{*}{Benchmark} & \multirow{2}{*}{Model} & \multirow{2}{*}{$|\mathcal{Q}'|$} & \multirow{2}{*}{$\Pr[\mathrm{A1}]$} & \multirow{2}{*}{$\Pr[\mathrm{A2} \mid \mathrm{A1}]$} & \multicolumn{3}{c}{$\Pr[\Delta_{w^{(n)}} > 0 \mid \mathrm{A1}]$} \\
\cmidrule(lr){6-8}
 & & & & & $n{=}1$ & $n{=}2$ & $n{=}3$ \\
\midrule
\multirow{5}{*}{FrontierScience-Olympiad} & GPT-OSS-120B & 73 & 64.4 & 93.6 & 93.6 & 97.9 & 97.9 \\
 & GPT-OSS-20B & 74 & 70.3 & 96.2 & 96.2 & 98.1 & 100.0 \\
 & Nemotron3-30B & 80 & 61.3 & 93.9 & 93.9 & 98.0 & 100.0 \\
 & Nemotron2-9B & 59 & 35.6 & 90.5 & 90.5 & 90.5 & 100.0 \\
 & Ministral3-14B & 47 & 36.2 & 94.1 & 94.1 & 100.0 & 100.0 \\
\midrule
\multirow{5}{*}{HMMT Feb~2026} & GPT-OSS-120B & 24 & 79.2 & 94.7 & 94.7 & 100.0 & 100.0 \\
 & GPT-OSS-20B & 24 & 83.3 & 100.0 & 100.0 & 100.0 & 100.0 \\
 & Nemotron3-30B & 20 & 80.0 & 100.0 & 100.0 & 100.0 & 100.0 \\
 & Nemotron2-9B & 23 & 56.5 & 100.0 & 100.0 & 100.0 & 100.0 \\
 & Ministral3-14B & 23 & 69.6 & 87.5 & 87.5 & 93.8 & 100.0 \\
\midrule
\multirow{5}{*}{AIME~2025} & GPT-OSS-120B & 23 & 100.0 & 91.3 & 91.3 & 95.7 & 100.0 \\
 & GPT-OSS-20B & 26 & 92.3 & 100.0 & 100.0 & 100.0 & 100.0 \\
 & Nemotron3-30B & 17 & 100.0 & 100.0 & 100.0 & 100.0 & 100.0 \\
 & Nemotron2-9B & 23 & 73.9 & 100.0 & 100.0 & 100.0 & 100.0 \\
 & Ministral3-14B & 23 & 73.9 & 94.1 & 94.1 & 100.0 & 100.0 \\
\midrule
\multirow{5}{*}{Brumo~2025} & GPT-OSS-120B & 21 & 81.0 & 94.1 & 94.1 & 94.1 & 100.0 \\
 & GPT-OSS-20B & 23 & 95.7 & 95.5 & 95.5 & 95.5 & 100.0 \\
 & Nemotron3-30B & 20 & 95.0 & 100.0 & 100.0 & 100.0 & 100.0 \\
 & Nemotron2-9B & 20 & 70.0 & 100.0 & 100.0 & 100.0 & 100.0 \\
 & Ministral3-14B & 26 & 88.5 & 91.3 & 91.3 & 95.7 & 100.0 \\
\bottomrule
\end{tabular}%
}
\end{table}

\clearpage
\section{Additional Experiments}
\label{app:additional}

\subsection{ROC Curves for Correctness Predictors}
\label{app:roc_grid}

Tables~\ref{tab:signal} and~\ref{tab:auroc} report $\overline{\mathrm{AUROC}}$ as a single number per (model, benchmark, signal). Figure~\ref{fig:signal_roc_grid} visualizes the underlying ROC curves on the same 5 model $\times$ 4 benchmark grid, overlaying prefix consistency and the WMV baselines per panel. Curves are macro-averaged: each problem's ROC is interpolated onto a common false-positive-rate grid and then averaged over the same problem subset $\mathcal Q'$ used in $\overline{\mathrm{AUROC}}$ (problems with at least one correct and one wrong initial sample, Appendix~\ref{app:auroc_eval}).

The visual pattern matches the $\overline{\mathrm{AUROC}}$ numbers: prefix consistency's curve sits at or above the baselines on FrontierScience-Olympiad and HMMT Feb~2026 in 8 of the 10 (model, benchmark) settings across the five models, while on AIME~2025 and Brumo~2025 the gap shrinks as the baselines' curves move closer to PC. The two cells where PC is not visibly above are Nemotron2-9B FrontierScience-Olympiad (where P(True) leads, consistent with PC's smallest discrimination gap $D$ in Table~\ref{tab:signal}) and GPT-OSS-120B HMMT Feb~2026 (where DeepConf tail edges PC by $.03$ in $\overline{\mathrm{AUROC}}$).

\begin{figure}[!htbp]
\centering
\includegraphics[width=0.78\textwidth]{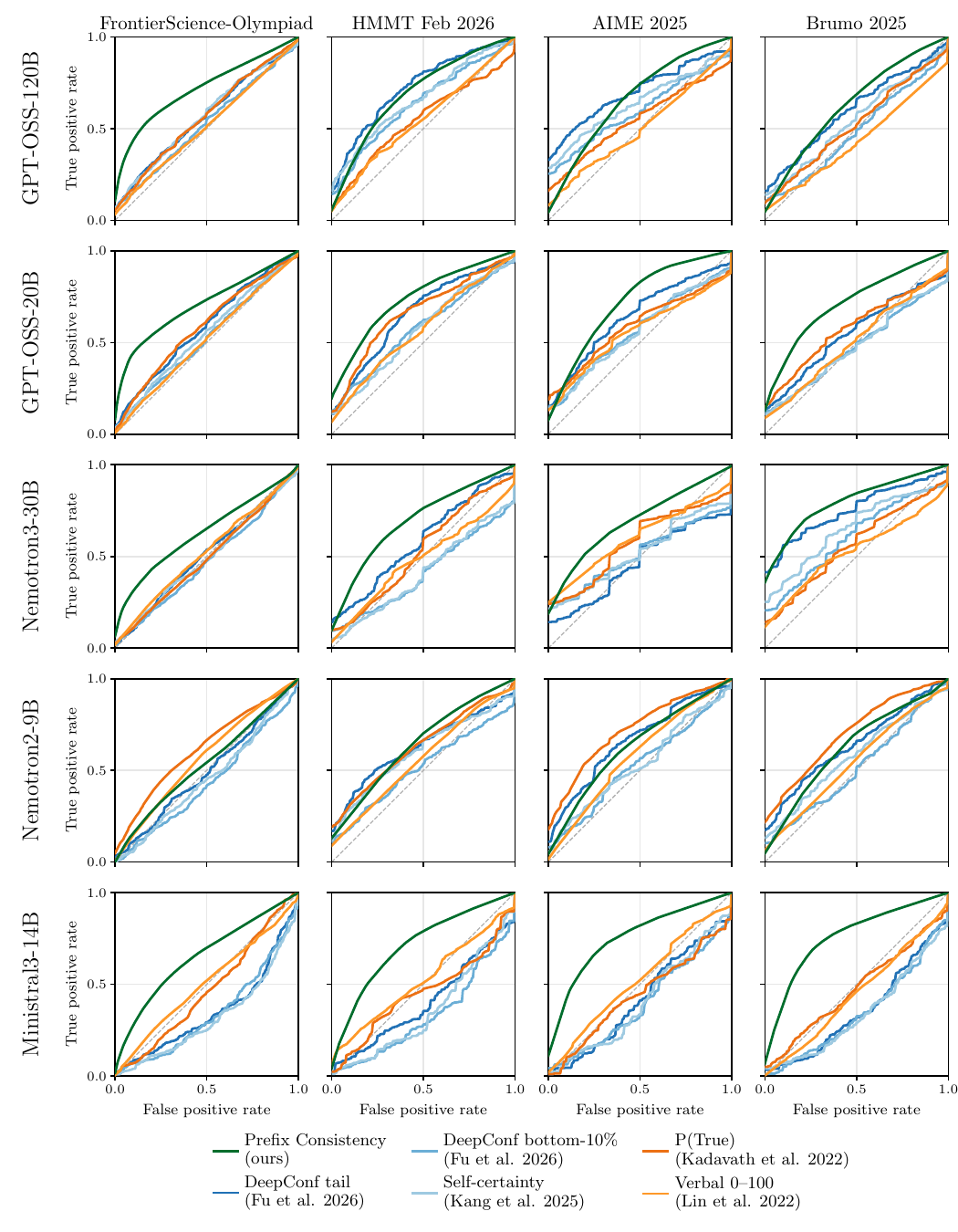}
\caption{\textbf{Macro ROC curves for prefix consistency and the WMV baselines, all 20 (model, benchmark) cells.} Rows are models, columns are benchmarks. Each curve is the per-problem ROC averaged on a common false-positive-rate grid, restricted to problems with at least one correct and one wrong initial sample. The dashed diagonal is the random-classifier baseline.}
\label{fig:signal_roc_grid}
\end{figure}

\FloatBarrier
\subsection{Cost--Accuracy Curves}
\label{app:cost_accuracy_curves}

Figure~\ref{fig:cost_accuracy_all} extends the cost--accuracy analysis of Figure~\ref{fig:regen-consistency-main} to all 20 (model, benchmark) cells, complementing the fixed-budget tables in Appendix~\ref{app:all_baselines}. Across nearly all 20 cells PC-cubic (green) sits above the other baselines once the budget exceeds the per-problem regeneration cost, and reaches Standard MV plateau at a noticeably smaller budget.

On FrontierScience-Olympiad, prefix consistency provides stronger correctness discrimination than every competing signal for four of the five models (Table~\ref{tab:auroc}). In all four, PC-cubic approaches the Standard MV plateau with substantially fewer tokens. The exception is Nemotron2-9B, which has the smallest discrimination gap (Table~\ref{tab:signal}) and no token savings.

The benefits also extend to settings where conventional confidence signals rank incorrect answers above correct ones. On Ministral3-14B, Self-certainty, DeepConf tail, and P(True) all have below-chance AUROC across the four benchmarks (Table~\ref{tab:auroc}). Nevertheless, PC-cubic improves over Standard MV and these baselines on FrontierScience-Olympiad, AIME~2025, and Brumo~2025, despite spending part of the token budget on regeneration.

\begin{figure}[!htbp]
\centering
\includegraphics[width=\textwidth]{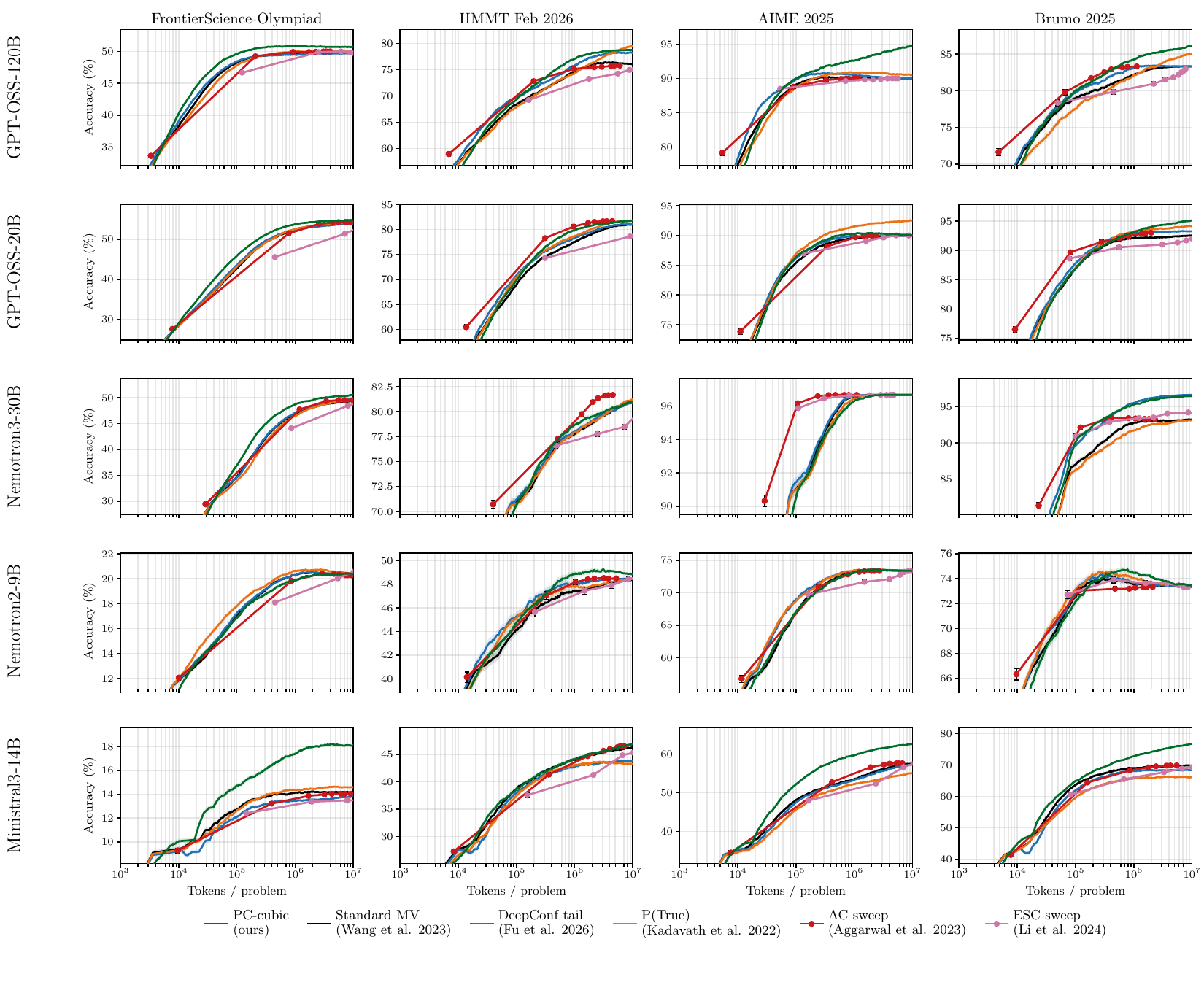}
\caption{\textbf{Cost--accuracy curves for all 20 (model, benchmark) cells.} Rows are models, columns are benchmarks. Y-axes are auto-zoomed to the Pass@1--plateau range. Shaded bands and AC/ESC error bars show $\pm 2\times$ the Monte Carlo standard error of accuracy (vertical) or cost (horizontal).}
\label{fig:cost_accuracy_all}
\end{figure}

\FloatBarrier
\subsection{Full Baseline Comparison}
\label{app:all_baselines}

Figure~\ref{fig:cost_accuracy_full} extends Figure~\ref{fig:cost_accuracy_all} to the full baseline set plus the oracle upper bounds. The three PC variants form a tight green band near the plateau across nearly every cell, while the DeepConf variants and especially their filtered variants spread out widely.

\begin{figure}[!htb]
\centering
\includegraphics[width=\textwidth]{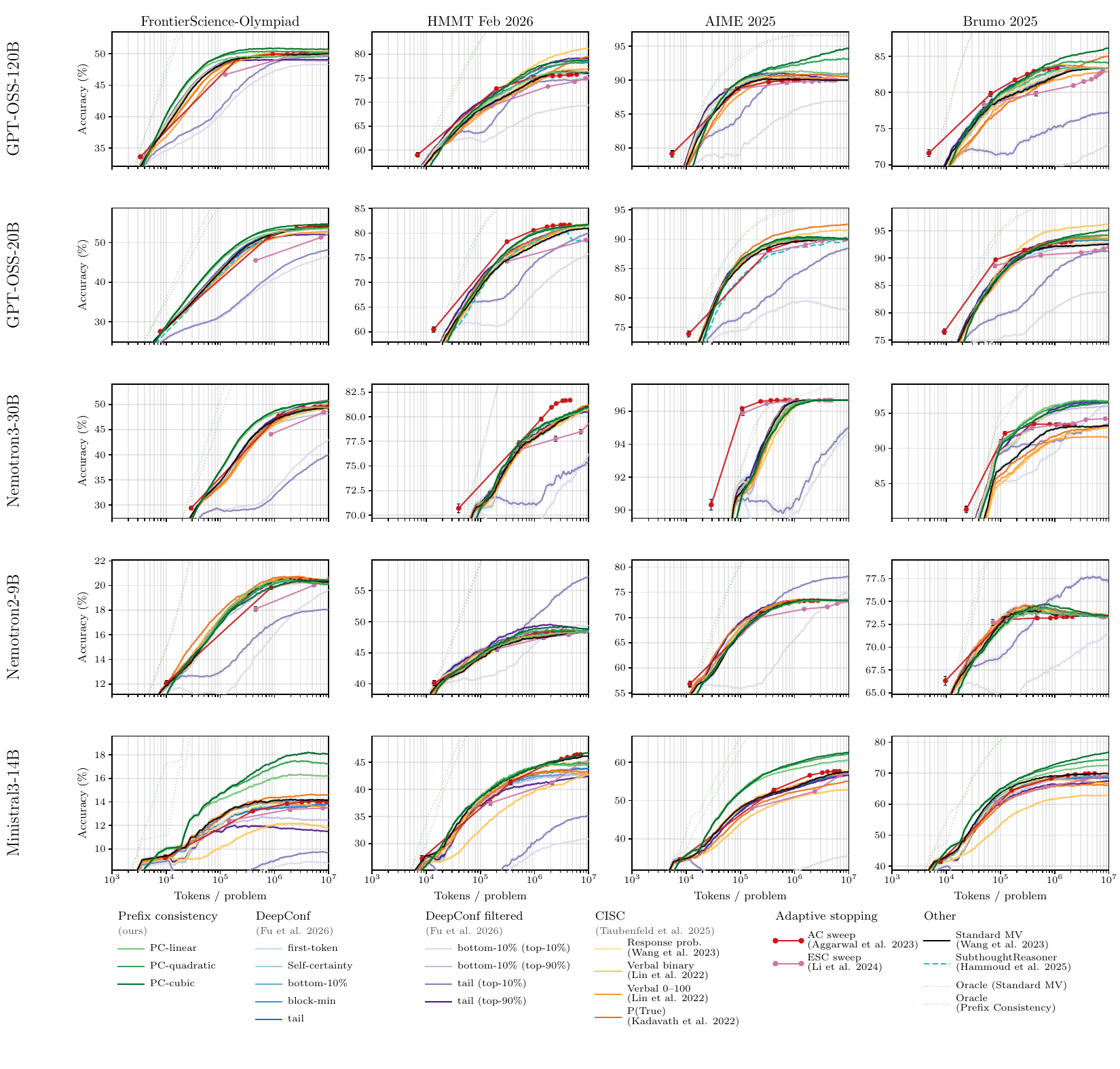}
\caption{\textbf{Cost--accuracy curves with the full baseline set.} Same layout as Figure~\ref{fig:cost_accuracy_all}, but additionally includes the oracle upper bounds alongside all evaluated baselines. Methods are colored by family: greens for PC, blues for DeepConf, purples for filtered DeepConf, browns/oranges for CISC, red for AC, mauve for ESC, cyan for SubthoughtReasoner, gray for Standard MV. Oracles are dotted. Shaded bands and AC/ESC error bars show $\pm 2\times$ the Monte Carlo standard error of accuracy (vertical) or cost (horizontal).}
\label{fig:cost_accuracy_full}
\end{figure}

Tables~\ref{tab:wmv_gpt_oss_120b_all}, \ref{tab:wmv_gpt_oss_20b_all}, \ref{tab:wmv_nemotron3_30b_all}, \ref{tab:wmv_nemotron2_9b_all}, and~\ref{tab:wmv_ministral3_14b_all} report fixed-budget accuracy at $B \in \{250\text{k}, 1\text{M}, 5\text{M}\}$ tokens for every evaluated baseline (Standard MV, all five DeepConf aggregation strategies~\citep{fu2025deepthinkconfidence} and their filtered variants, Response probability, verbalized confidence and P(True)~\citep{taubenfeld2025confidenceimprovesselfconsistency,kadavath2022languagemodelsknowthey}, and PC-linear/quadratic/cubic), one table per model. See Appendix~\ref{app:baseline_impl} for baseline definitions.
Tables~\ref{tab:token_savings_gpt_oss_120b_all}, \ref{tab:token_savings_gpt_oss_20b_all}, \ref{tab:token_savings_nemotron3_30b_all}, \ref{tab:token_savings_nemotron2_9b_all}, and~\ref{tab:token_savings_ministral3_14b_all} report the corresponding token-efficiency ratios at $\alpha \in \{75\%, 90\%, 99\%\}$.

\FloatBarrier

Two patterns emerge.
First, on the more difficult science benchmark (FrontierScience-Olympiad) PC-cubic is the best non-oracle method at $B{=}250\text{k}$ and $B{=}1\text{M}$ for the four models with non-trivial discrimination gap (GPT-OSS-120B, GPT-OSS-20B, Nemotron3-30B, Ministral3-14B): at $B{=}1\text{M}$ it reaches .508, .537, .486, and .174 respectively, against the strongest non-PC baseline at .503, .525, .472, and .143. At $B{=}5\text{M}$ PC-cubic remains best on three of these four, with PC-quadratic narrowly ahead on Nemotron3-30B (.504 vs.\ .503, within $2\sigma$). The PC-linear $\le$ PC-quadratic $\le$ PC-cubic ordering holds in 11 of the 12 (model, budget) FrontierScience-Olympiad cells over these four models, with the Nemotron3-30B 5M cell as the only inversion, consistent with the convex-weight analysis: in Eq.~\eqref{eq:exact_decomposition}, the coefficient $\lambda_w := w(1) - 2w(\tfrac12)$ on the self-reproduction term $\pi(a)\,T(a \rightarrow a)$ (the joint probability that $a$ is both the initial and the regenerated answer) takes values $0$, $\tfrac12$, $\tfrac34$ for PC-linear, PC-quadratic, PC-cubic respectively, so larger $n$ upweights self-reproducing candidates more strongly.
The exception is Nemotron2-9B, whose FrontierScience-Olympiad discrimination gap is the smallest in our suite ($D = 7.4\%$, Table~\ref{tab:signal}). PC-WMV there roughly matches Standard MV (PC-cubic .199 vs.\ Standard MV .203 at $B{=}1\text{M}$), consistent with Theorem~\ref{thm:improvement} predicting no advantage when $D$ is small.

Second, on the easier benchmarks (AIME~2025, Brumo~2025) where Standard MV plateau already sits close to its peak, verbalized confidence (Verbal binary, P(True)) and DeepConf tail (top-90\%) become competitive or take the lead in a few cells (e.g.\ Verbal binary on GPT-OSS-20B Brumo reaches .951 at $B{=}1\text{M}$ vs.\ .931 for PC-cubic). This matches the cost--quality trade-off discussed in Section~\ref{sec:experiments}, where PC's larger per-group cost is no longer amortized once Standard MV is near its plateau.

SubthoughtReasoner~\citep{hammoud2025beyondlastanswer} appears only in the GPT-OSS-20B tables, the only model for which it has been re-evaluated under the unified incremental path. SubthoughtReasoner does not consistently improve over Standard MV (vs.\ Standard MV at $B{=}1\text{M}$: .778 vs.\ .775 on HMMT, .518 vs.\ .523 on FrontierScience-Olympiad, .886 vs.\ .896 on AIME~2025), and the GPT-OSS-20B Brumo cell and other-model cells are not evaluated.

\begingroup
\setlength{\intextsep}{2pt plus 1pt minus 1pt}
\setlength{\floatsep}{2pt plus 1pt minus 1pt}
\setlength{\textfloatsep}{2pt plus 1pt minus 1pt}
\begin{table}[H]
\centering
\caption{\textbf{All baselines, GPT-OSS-120B (higher accuracy is better).} Bold marks the best non-oracle per column. Subscripts show $\pm 2\times$ the Monte Carlo standard error.}
\label{tab:wmv_gpt_oss_120b_all}
\scriptsize
\setlength{\tabcolsep}{1.5pt}
\renewcommand{\arraystretch}{0.8}
\resizebox{\textwidth}{!}{%
\begin{tabular}{l ccc ccc ccc ccc}
\toprule
 & \multicolumn{3}{c}{FrontierScience-Olympiad} & \multicolumn{3}{c}{HMMT Feb~2026} & \multicolumn{3}{c}{AIME~2025} & \multicolumn{3}{c}{Brumo~2025} \\
\cmidrule(lr){2-4} \cmidrule(lr){5-7} \cmidrule(lr){8-10} \cmidrule(lr){11-13}
Method & $B{=}$250k & $B{=}$1M & $B{=}$5M & $B{=}$250k & $B{=}$1M & $B{=}$5M & $B{=}$250k & $B{=}$1M & $B{=}$5M & $B{=}$250k & $B{=}$1M & $B{=}$5M \\
\midrule
\multicolumn{13}{l}{\textit{Baseline}} \\
\quad Standard MV & .493$_{\pm.001}$ & .495$_{\pm.001}$ & .499$_{\pm.001}$ & .708$_{\pm.003}$ & .745$_{\pm.003}$ & .763$_{\pm.001}$ & .901$_{\pm.002}$ & .901$_{\pm.001}$ & .900$_{\pm.000}$ & .801$_{\pm.003}$ & .821$_{\pm.002}$ & .833$_{\pm.000}$ \\
\midrule
\multicolumn{13}{l}{\textit{DeepConf}} \\
\quad DeepConf first-token & .493$_{\pm.001}$ & .495$_{\pm.001}$ & .499$_{\pm.001}$ & .708$_{\pm.003}$ & .745$_{\pm.003}$ & .763$_{\pm.001}$ & .901$_{\pm.002}$ & .901$_{\pm.001}$ & .900$_{\pm.000}$ & .801$_{\pm.003}$ & .821$_{\pm.002}$ & .833$_{\pm.000}$ \\
\quad Self-certainty & .494$_{\pm.001}$ & .495$_{\pm.001}$ & .497$_{\pm.001}$ & .713$_{\pm.003}$ & .749$_{\pm.003}$ & .768$_{\pm.002}$ & .904$_{\pm.002}$ & .903$_{\pm.001}$ & .900$_{\pm.000}$ & .805$_{\pm.003}$ & .826$_{\pm.002}$ & .833$_{\pm.000}$ \\
\quad DeepConf bottom-10\% & .493$_{\pm.001}$ & .495$_{\pm.001}$ & .496$_{\pm.001}$ & .710$_{\pm.003}$ & .747$_{\pm.003}$ & .764$_{\pm.001}$ & .902$_{\pm.002}$ & .902$_{\pm.001}$ & .900$_{\pm.000}$ & .802$_{\pm.003}$ & .823$_{\pm.002}$ & .833$_{\pm.000}$ \\
\bottomrule
\end{tabular}%
}%
\end{table}

\begin{table}[H]
\centering
\caption{\textbf{All baselines, GPT-OSS-20B (higher accuracy is better).} Bold marks the best non-oracle per column. Subscripts show $\pm 2\times$ the Monte Carlo standard error.}
\label{tab:wmv_gpt_oss_20b_all}
\scriptsize
\setlength{\tabcolsep}{1.5pt}
\renewcommand{\arraystretch}{0.8}
\resizebox{\textwidth}{!}{%
\begin{tabular}{l ccc ccc ccc ccc}
\toprule
 & \multicolumn{3}{c}{FrontierScience-Olympiad} & \multicolumn{3}{c}{HMMT Feb~2026} & \multicolumn{3}{c}{AIME~2025} & \multicolumn{3}{c}{Brumo~2025} \\
\cmidrule(lr){2-4} \cmidrule(lr){5-7} \cmidrule(lr){8-10} \cmidrule(lr){11-13}
Method & $B{=}$250k & $B{=}$1M & $B{=}$5M & $B{=}$250k & $B{=}$1M & $B{=}$5M & $B{=}$250k & $B{=}$1M & $B{=}$5M & $B{=}$250k & $B{=}$1M & $B{=}$5M \\
\midrule
\multicolumn{13}{l}{\textit{Baseline}} \\
\quad Standard MV & .482$_{\pm.002}$ & .523$_{\pm.001}$ & .537$_{\pm.001}$ & .736$_{\pm.003}$ & .775$_{\pm.002}$ & .807$_{\pm.001}$ & .881$_{\pm.002}$ & .896$_{\pm.001}$ & .900$_{\pm.000}$ & .901$_{\pm.003}$ & .922$_{\pm.002}$ & .924$_{\pm.001}$ \\
\midrule
\multicolumn{13}{l}{\textit{DeepConf}} \\
\quad DeepConf first-token & .482$_{\pm.002}$ & .523$_{\pm.001}$ & .537$_{\pm.001}$ & .736$_{\pm.003}$ & .775$_{\pm.002}$ & .807$_{\pm.001}$ & .881$_{\pm.002}$ & .896$_{\pm.001}$ & .900$_{\pm.000}$ & .901$_{\pm.003}$ & .922$_{\pm.002}$ & .924$_{\pm.001}$ \\
\quad Self-certainty & .482$_{\pm.002}$ & .522$_{\pm.001}$ & .535$_{\pm.001}$ & .738$_{\pm.003}$ & .778$_{\pm.003}$ & .808$_{\pm.001}$ & .885$_{\pm.002}$ & .898$_{\pm.001}$ & .900$_{\pm.000}$ & .903$_{\pm.003}$ & .924$_{\pm.002}$ & .927$_{\pm.001}$ \\
\quad DeepConf bottom-10\% & .481$_{\pm.002}$ & .521$_{\pm.001}$ & .536$_{\pm.001}$ & .737$_{\pm.003}$ & .777$_{\pm.002}$ & .808$_{\pm.001}$ & .882$_{\pm.002}$ & .896$_{\pm.001}$ & .900$_{\pm.000}$ & .900$_{\pm.003}$ & .921$_{\pm.002}$ & .924$_{\pm.001}$ \\
\bottomrule
\end{tabular}%
}%
\end{table}

\begin{table}[H]
\centering
\caption{\textbf{All baselines, Nemotron3-30B (higher accuracy is better).} Bold marks the best non-oracle per column. Subscripts show $\pm 2\times$ the Monte Carlo standard error.}
\label{tab:wmv_nemotron3_30b_all}
\scriptsize
\setlength{\tabcolsep}{1.5pt}
\renewcommand{\arraystretch}{0.8}
\resizebox{\textwidth}{!}{%
\begin{tabular}{l ccc ccc ccc ccc}
\toprule
 & \multicolumn{3}{c}{FrontierScience-Olympiad} & \multicolumn{3}{c}{HMMT Feb~2026} & \multicolumn{3}{c}{AIME~2025} & \multicolumn{3}{c}{Brumo~2025} \\
\cmidrule(lr){2-4} \cmidrule(lr){5-7} \cmidrule(lr){8-10} \cmidrule(lr){11-13}
Method & $B{=}$250k & $B{=}$1M & $B{=}$5M & $B{=}$250k & $B{=}$1M & $B{=}$5M & $B{=}$250k & $B{=}$1M & $B{=}$5M & $B{=}$250k & $B{=}$1M & $B{=}$5M \\
\midrule
\multicolumn{13}{l}{\textit{Baseline}} \\
\quad Standard MV & .404$_{\pm.002}$ & .468$_{\pm.002}$ & .490$_{\pm.001}$ & .736$_{\pm.004}$ & .777$_{\pm.003}$ & .802$_{\pm.002}$ & .935$_{\pm.002}$ & \textbf{.966}$_{\pm.000}$ & \textbf{.967}$_{\pm.000}$ & .896$_{\pm.003}$ & .928$_{\pm.002}$ & .931$_{\pm.001}$ \\
\midrule
\multicolumn{13}{l}{\textit{DeepConf}} \\
\quad DeepConf first-token & .403$_{\pm.002}$ & .467$_{\pm.002}$ & .490$_{\pm.001}$ & .736$_{\pm.004}$ & .777$_{\pm.003}$ & .802$_{\pm.002}$ & .934$_{\pm.002}$ & \textbf{.966}$_{\pm.001}$ & \textbf{.967}$_{\pm.000}$ & .895$_{\pm.003}$ & .927$_{\pm.002}$ & .931$_{\pm.001}$ \\
\quad Self-certainty & .402$_{\pm.002}$ & .468$_{\pm.002}$ & .491$_{\pm.001}$ & .734$_{\pm.004}$ & .779$_{\pm.003}$ & \textbf{.804}$_{\pm.002}$ & .935$_{\pm.002}$ & \textbf{.966}$_{\pm.000}$ & \textbf{.967}$_{\pm.000}$ & .918$_{\pm.003}$ & .954$_{\pm.002}$ & .963$_{\pm.001}$ \\
\quad DeepConf bottom-10\% & .401$_{\pm.002}$ & .466$_{\pm.002}$ & .489$_{\pm.001}$ & .732$_{\pm.004}$ & .779$_{\pm.003}$ & .802$_{\pm.002}$ & .930$_{\pm.003}$ & \textbf{.966}$_{\pm.000}$ & \textbf{.967}$_{\pm.000}$ & .918$_{\pm.003}$ & .953$_{\pm.002}$ & .963$_{\pm.001}$ \\
\bottomrule
\end{tabular}%
}%
\end{table}

\begin{table}[H]
\centering
\caption{\textbf{All baselines, Nemotron2-9B (higher accuracy is better).} Bold marks the best non-oracle per column. Subscripts show $\pm 2\times$ the Monte Carlo standard error.}
\label{tab:wmv_nemotron2_9b_all}
\scriptsize
\setlength{\tabcolsep}{1.5pt}
\renewcommand{\arraystretch}{0.8}
\resizebox{\textwidth}{!}{%
\begin{tabular}{l ccc ccc ccc ccc}
\toprule
 & \multicolumn{3}{c}{FrontierScience-Olympiad} & \multicolumn{3}{c}{HMMT Feb~2026} & \multicolumn{3}{c}{AIME~2025} & \multicolumn{3}{c}{Brumo~2025} \\
\cmidrule(lr){2-4} \cmidrule(lr){5-7} \cmidrule(lr){8-10} \cmidrule(lr){11-13}
Method & $B{=}$250k & $B{=}$1M & $B{=}$5M & $B{=}$250k & $B{=}$1M & $B{=}$5M & $B{=}$250k & $B{=}$1M & $B{=}$5M & $B{=}$250k & $B{=}$1M & $B{=}$5M \\
\midrule
\multicolumn{13}{l}{\textit{Baseline}} \\
\quad Standard MV & .187$_{\pm.001}$ & .203$_{\pm.001}$ & .203$_{\pm.001}$ & .461$_{\pm.004}$ & .475$_{\pm.002}$ & .482$_{\pm.001}$ & .708$_{\pm.003}$ & .731$_{\pm.002}$ & .734$_{\pm.000}$ & .737$_{\pm.003}$ & .736$_{\pm.002}$ & .734$_{\pm.001}$ \\
\midrule
\multicolumn{13}{l}{\textit{DeepConf}} \\
\quad DeepConf first-token & .187$_{\pm.001}$ & .203$_{\pm.001}$ & .203$_{\pm.001}$ & .460$_{\pm.004}$ & .475$_{\pm.002}$ & .482$_{\pm.001}$ & .706$_{\pm.003}$ & .730$_{\pm.002}$ & .734$_{\pm.000}$ & .738$_{\pm.003}$ & .736$_{\pm.002}$ & .734$_{\pm.001}$ \\
\quad Self-certainty & .187$_{\pm.001}$ & .202$_{\pm.001}$ & .203$_{\pm.001}$ & .468$_{\pm.004}$ & .478$_{\pm.002}$ & .483$_{\pm.001}$ & .717$_{\pm.003}$ & .733$_{\pm.002}$ & .734$_{\pm.001}$ & .743$_{\pm.003}$ & .737$_{\pm.002}$ & .734$_{\pm.001}$ \\
\quad DeepConf bottom-10\% & .186$_{\pm.001}$ & .202$_{\pm.001}$ & .202$_{\pm.001}$ & .463$_{\pm.004}$ & .476$_{\pm.002}$ & .483$_{\pm.001}$ & .709$_{\pm.003}$ & .731$_{\pm.002}$ & .734$_{\pm.001}$ & .739$_{\pm.003}$ & .735$_{\pm.002}$ & .734$_{\pm.001}$ \\
\bottomrule
\end{tabular}%
}%
\end{table}

\begin{table}[H]
\centering
\caption{\textbf{All baselines, Ministral3-14B (higher accuracy is better).} Bold marks the best non-oracle per column. Subscripts show $\pm 2\times$ the Monte Carlo standard error.}
\label{tab:wmv_ministral3_14b_all}
\scriptsize
\setlength{\tabcolsep}{1.5pt}
\renewcommand{\arraystretch}{0.8}
\resizebox{\textwidth}{!}{%
\begin{tabular}{l ccc ccc ccc ccc}
\toprule
 & \multicolumn{3}{c}{FrontierScience-Olympiad} & \multicolumn{3}{c}{HMMT Feb~2026} & \multicolumn{3}{c}{AIME~2025} & \multicolumn{3}{c}{Brumo~2025} \\
\cmidrule(lr){2-4} \cmidrule(lr){5-7} \cmidrule(lr){8-10} \cmidrule(lr){11-13}
Method & $B{=}$250k & $B{=}$1M & $B{=}$5M & $B{=}$250k & $B{=}$1M & $B{=}$5M & $B{=}$250k & $B{=}$1M & $B{=}$5M & $B{=}$250k & $B{=}$1M & $B{=}$5M \\
\midrule
\multicolumn{13}{l}{\textit{Baseline}} \\
\quad Standard MV & .137$_{\pm.001}$ & .141$_{\pm.001}$ & .142$_{\pm.001}$ & .413$_{\pm.004}$ & \textbf{.440}$_{\pm.003}$ & .458$_{\pm.002}$ & .508$_{\pm.003}$ & .535$_{\pm.003}$ & .569$_{\pm.002}$ & .669$_{\pm.004}$ & .691$_{\pm.003}$ & .699$_{\pm.003}$ \\
\midrule
\multicolumn{13}{l}{\textit{DeepConf}} \\
\quad DeepConf first-token & .135$_{\pm.001}$ & .140$_{\pm.001}$ & .142$_{\pm.001}$ & .411$_{\pm.004}$ & .438$_{\pm.003}$ & .458$_{\pm.002}$ & .505$_{\pm.003}$ & .532$_{\pm.003}$ & .569$_{\pm.002}$ & .668$_{\pm.004}$ & .690$_{\pm.003}$ & .698$_{\pm.003}$ \\
\quad Self-certainty & .130$_{\pm.001}$ & .135$_{\pm.001}$ & .137$_{\pm.001}$ & .409$_{\pm.004}$ & .429$_{\pm.003}$ & .438$_{\pm.002}$ & .505$_{\pm.003}$ & .533$_{\pm.003}$ & .565$_{\pm.002}$ & .659$_{\pm.004}$ & .682$_{\pm.003}$ & .686$_{\pm.002}$ \\
\quad DeepConf bottom-10\% & .131$_{\pm.001}$ & .136$_{\pm.001}$ & .138$_{\pm.001}$ & .409$_{\pm.004}$ & .430$_{\pm.003}$ & .438$_{\pm.002}$ & .506$_{\pm.003}$ & .531$_{\pm.003}$ & .564$_{\pm.002}$ & .660$_{\pm.004}$ & .683$_{\pm.003}$ & .686$_{\pm.002}$ \\
\bottomrule
\end{tabular}%
}%
\end{table}

\begin{table}[H]
\centering
\caption{\textbf{Token efficiency, GPT-OSS-120B (smaller is better).} Cells $<$1 indicate more cost-efficient than Standard MV. ``N/A'' indicates target unreachable. Bold marks the best per column.}
\label{tab:token_savings_gpt_oss_120b_all}
\scriptsize
\setlength{\tabcolsep}{1.5pt}
\renewcommand{\arraystretch}{0.8}
\resizebox{\textwidth}{!}{%
\begin{tabular}{l ccc ccc ccc ccc}
\toprule
 & \multicolumn{3}{c}{FrontierScience-Olympiad} & \multicolumn{3}{c}{HMMT Feb~2026} & \multicolumn{3}{c}{AIME~2025} & \multicolumn{3}{c}{Brumo~2025} \\
\cmidrule(lr){2-4} \cmidrule(lr){5-7} \cmidrule(lr){8-10} \cmidrule(lr){11-13}
 & $\alpha{=}75\%$ & $\alpha{=}90\%$ & $\alpha{=}99\%$ & $\alpha{=}75\%$ & $\alpha{=}90\%$ & $\alpha{=}99\%$ & $\alpha{=}75\%$ & $\alpha{=}90\%$ & $\alpha{=}99\%$ & $\alpha{=}75\%$ & $\alpha{=}90\%$ & $\alpha{=}99\%$ \\
\midrule
Pass@1 / Standard MV plateau & .338 / .500 &  &  & .589 / .760 &  &  & .789 / .900 &  &  & .713 / .833 &  &  \\
Standard MV budget ($B_{\mathrm{MV}}$) & 45k & 107k & 2.5M & 364k & 858k & 1.8M & 56k & 99k & 189k & 274k & 934k & 3.3M \\
\midrule
\multicolumn{13}{l}{\textit{DeepConf}} \\
\quad DeepConf first-token & 1.00$_{\pm0.03}$$\times$ & 1.00$_{\pm0.08}$$\times$ & 1.00$_{\pm0.34}$$\times$ & 1.00$_{\pm0.11}$$\times$ & 1.00$_{\pm0.08}$$\times$ & 1.00$_{\pm0.10}$$\times$ & 1.00$_{\pm0.05}$$\times$ & 1.00$_{\pm0.09}$$\times$ & 1.00$_{\pm0.19}$$\times$ & 1.00$_{\pm0.30}$$\times$ & 1.00$_{\pm0.12}$$\times$ & 1.00$_{\pm0.18}$$\times$ \\
\quad Self-certainty & 0.89$_{\pm0.04}$$\times$ & 0.87$_{\pm0.06}$$\times$ & N/A & 0.74$_{\pm0.07}$$\times$ & 0.85$_{\pm0.07}$$\times$ & 0.87$_{\pm0.11}$$\times$ & 0.81$_{\pm0.06}$$\times$ & 0.81$_{\pm0.07}$$\times$ & 0.66$_{\pm0.11}$$\times$ & 0.73$_{\pm0.18}$$\times$ & 0.76$_{\pm0.10}$$\times$ & 0.59$_{\pm0.09}$$\times$ \\
\quad DeepConf bottom-10\% & 0.97$_{\pm0.04}$$\times$ & 0.97$_{\pm0.08}$$\times$ & N/A & 0.86$_{\pm0.10}$$\times$ & 0.93$_{\pm0.08}$$\times$ & 0.93$_{\pm0.09}$$\times$ & 0.97$_{\pm0.06}$$\times$ & 0.90$_{\pm0.08}$$\times$ & 0.79$_{\pm0.13}$$\times$ & 0.95$_{\pm0.25}$$\times$ & 0.90$_{\pm0.10}$$\times$ & 0.88$_{\pm0.21}$$\times$ \\
\bottomrule
\end{tabular}%
}%
\end{table}

\begin{table}[H]
\centering
\caption{\textbf{Token efficiency, GPT-OSS-20B (smaller is better).} Cells $<$1 indicate more cost-efficient than Standard MV. ``N/A'' indicates target unreachable. Bold marks the best per column.}
\label{tab:token_savings_gpt_oss_20b_all}
\scriptsize
\setlength{\tabcolsep}{1.5pt}
\renewcommand{\arraystretch}{0.8}
\resizebox{\textwidth}{!}{%
\begin{tabular}{l ccc ccc ccc ccc}
\toprule
 & \multicolumn{3}{c}{FrontierScience-Olympiad} & \multicolumn{3}{c}{HMMT Feb~2026} & \multicolumn{3}{c}{AIME~2025} & \multicolumn{3}{c}{Brumo~2025} \\
\cmidrule(lr){2-4} \cmidrule(lr){5-7} \cmidrule(lr){8-10} \cmidrule(lr){11-13}
 & $\alpha{=}75\%$ & $\alpha{=}90\%$ & $\alpha{=}99\%$ & $\alpha{=}75\%$ & $\alpha{=}90\%$ & $\alpha{=}99\%$ & $\alpha{=}75\%$ & $\alpha{=}90\%$ & $\alpha{=}99\%$ & $\alpha{=}75\%$ & $\alpha{=}90\%$ & $\alpha{=}99\%$ \\
\midrule
Pass@1 / Standard MV plateau & .276 / .540 &  &  & .601 / .809 &  &  & .742 / .900 &  &  & .766 / .926 &  &  \\
Standard MV budget ($B_{\mathrm{MV}}$) & 215k & 640k & 5.0M & 517k & 1.8M & 5.3M & 128k & 312k & 2.1M & 150k & 348k & 5.3M \\
\midrule
\multicolumn{13}{l}{\textit{DeepConf}} \\
\quad DeepConf first-token & 1.00$_{\pm0.03}$$\times$ & 1.00$_{\pm0.07}$$\times$ & 1.00$_{\pm0.15}$$\times$ & 1.00$_{\pm0.11}$$\times$ & 1.00$_{\pm0.09}$$\times$ & 1.00$_{\pm0.18}$$\times$ & 1.00$_{\pm0.09}$$\times$ & 1.00$_{\pm0.16}$$\times$ & 1.00$_{\pm0.24}$$\times$ & 1.00$_{\pm0.06}$$\times$ & 1.00$_{\pm0.13}$$\times$ & 1.00$_{\pm1.33}$$\times$ \\
\quad Self-certainty & 0.97$_{\pm0.03}$$\times$ & 0.99$_{\pm0.07}$$\times$ & N/A & 0.85$_{\pm0.10}$$\times$ & 0.93$_{\pm0.09}$$\times$ & 0.87$_{\pm0.14}$$\times$ & 0.93$_{\pm0.07}$$\times$ & 0.78$_{\pm0.11}$$\times$ & 0.72$_{\pm0.18}$$\times$ & 0.95$_{\pm0.06}$$\times$ & 0.87$_{\pm0.10}$$\times$ & 0.18$_{\pm0.43}$$\times$ \\
\quad DeepConf bottom-10\% & 1.00$_{\pm0.04}$$\times$ & 1.07$_{\pm0.07}$$\times$ & 1.42$_{\pm0.29}$$\times$ & 0.88$_{\pm0.10}$$\times$ & 0.96$_{\pm0.08}$$\times$ & 0.89$_{\pm0.15}$$\times$ & 0.92$_{\pm0.07}$$\times$ & 0.91$_{\pm0.12}$$\times$ & 0.86$_{\pm0.24}$$\times$ & 1.03$_{\pm0.06}$$\times$ & 1.07$_{\pm0.15}$$\times$ & 0.91$_{\pm1.27}$$\times$ \\
\bottomrule
\end{tabular}%
}%
\end{table}

\begin{table}[H]
\centering
\caption{\textbf{Token efficiency, Nemotron3-30B (smaller is better).} Cells $<$1 indicate more cost-efficient than Standard MV. ``N/A'' indicates target unreachable. Bold marks the best per column.}
\label{tab:token_savings_nemotron3_30b_all}
\scriptsize
\setlength{\tabcolsep}{1.5pt}
\renewcommand{\arraystretch}{0.8}
\resizebox{\textwidth}{!}{%
\begin{tabular}{l ccc ccc ccc ccc}
\toprule
 & \multicolumn{3}{c}{FrontierScience-Olympiad} & \multicolumn{3}{c}{HMMT Feb~2026} & \multicolumn{3}{c}{AIME~2025} & \multicolumn{3}{c}{Brumo~2025} \\
\cmidrule(lr){2-4} \cmidrule(lr){5-7} \cmidrule(lr){8-10} \cmidrule(lr){11-13}
 & $\alpha{=}75\%$ & $\alpha{=}90\%$ & $\alpha{=}99\%$ & $\alpha{=}75\%$ & $\alpha{=}90\%$ & $\alpha{=}99\%$ & $\alpha{=}75\%$ & $\alpha{=}90\%$ & $\alpha{=}99\%$ & $\alpha{=}75\%$ & $\alpha{=}90\%$ & $\alpha{=}99\%$ \\
\midrule
Pass@1 / Standard MV plateau & .295 / .493 &  &  & .708 / .810 &  &  & .902 / .967 &  &  & .816 / .932 &  &  \\
Standard MV budget ($B_{\mathrm{MV}}$) & 486k & 1.3M & 6.0M & 1.6M & 4.2M & 8.5M & 396k & 589k & 1.2M & 342k & 651k & 2.5M \\
\midrule
\multicolumn{13}{l}{\textit{DeepConf}} \\
\quad DeepConf first-token & 1.03$_{\pm0.06}$$\times$ & 1.00$_{\pm0.10}$$\times$ & 0.92$_{\pm0.22}$$\times$ & 1.02$_{\pm0.20}$$\times$ & 1.00$_{\pm0.14}$$\times$ & 0.96$_{\pm0.12}$$\times$ & 1.03$_{\pm0.06}$$\times$ & 1.06$_{\pm0.08}$$\times$ & 0.94$_{\pm0.15}$$\times$ & 1.03$_{\pm0.13}$$\times$ & 1.00$_{\pm0.16}$$\times$ & 0.71$_{\pm0.39}$$\times$ \\
\quad Self-certainty & 1.04$_{\pm0.06}$$\times$ & 1.01$_{\pm0.11}$$\times$ & 0.85$_{\pm0.19}$$\times$ & 0.85$_{\pm0.15}$$\times$ & 0.89$_{\pm0.14}$$\times$ & 0.93$_{\pm0.12}$$\times$ & 0.97$_{\pm0.06}$$\times$ & 0.96$_{\pm0.06}$$\times$ & 0.84$_{\pm0.12}$$\times$ & 0.45$_{\pm0.05}$$\times$ & 0.41$_{\pm0.05}$$\times$ & 0.14$_{\pm0.07}$$\times$ \\
\quad DeepConf bottom-10\% & 1.07$_{\pm0.05}$$\times$ & 1.11$_{\pm0.11}$$\times$ & 1.36$_{\pm0.34}$$\times$ & 0.93$_{\pm0.18}$$\times$ & 1.00$_{\pm0.15}$$\times$ & 0.97$_{\pm0.13}$$\times$ & 1.01$_{\pm0.06}$$\times$ & 1.00$_{\pm0.06}$$\times$ & 0.90$_{\pm0.15}$$\times$ & 0.50$_{\pm0.06}$$\times$ & 0.41$_{\pm0.05}$$\times$ & 0.15$_{\pm0.08}$$\times$ \\
\bottomrule
\end{tabular}%
}%
\end{table}

\begin{table}[H]
\centering
\caption{\textbf{Token efficiency, Nemotron2-9B (smaller is better).} Cells $<$1 indicate more cost-efficient than Standard MV. ``N/A'' indicates target unreachable. Bold marks the best per column.}
\label{tab:token_savings_nemotron2_9b_all}
\scriptsize
\setlength{\tabcolsep}{1.5pt}
\renewcommand{\arraystretch}{0.8}
\resizebox{\textwidth}{!}{%
\begin{tabular}{l ccc ccc ccc ccc}
\toprule
 & \multicolumn{3}{c}{FrontierScience-Olympiad} & \multicolumn{3}{c}{HMMT Feb~2026} & \multicolumn{3}{c}{AIME~2025} & \multicolumn{3}{c}{Brumo~2025} \\
\cmidrule(lr){2-4} \cmidrule(lr){5-7} \cmidrule(lr){8-10} \cmidrule(lr){11-13}
 & $\alpha{=}75\%$ & $\alpha{=}90\%$ & $\alpha{=}99\%$ & $\alpha{=}75\%$ & $\alpha{=}90\%$ & $\alpha{=}99\%$ & $\alpha{=}75\%$ & $\alpha{=}90\%$ & $\alpha{=}99\%$ & $\alpha{=}75\%$ & $\alpha{=}90\%$ & $\alpha{=}99\%$ \\
\midrule
Pass@1 / Standard MV plateau & .120 / .203 &  &  & .401 / .484 &  &  & .569 / .733 &  &  & .660 / .733 &  &  \\
Standard MV budget ($B_{\mathrm{MV}}$) & 181k & 383k & 841k & 274k & 1.5M & 6.1M & 163k & 374k & 1.1M & 77k & 109k & 126k \\
\midrule
\multicolumn{13}{l}{\textit{DeepConf}} \\
\quad DeepConf first-token & 0.99$_{\pm0.10}$$\times$ & 1.01$_{\pm0.08}$$\times$ & 1.02$_{\pm0.17}$$\times$ & 1.01$_{\pm0.20}$$\times$ & 1.18$_{\pm0.74}$$\times$ & 1.03$_{\pm0.25}$$\times$ & 1.01$_{\pm0.08}$$\times$ & 1.11$_{\pm0.13}$$\times$ & 1.20$_{\pm0.31}$$\times$ & 0.99$_{\pm0.08}$$\times$ & 1.00$_{\pm0.13}$$\times$ & 1.00$_{\pm0.09}$$\times$ \\
\quad Self-certainty & 1.03$_{\pm0.10}$$\times$ & 1.04$_{\pm0.09}$$\times$ & 1.02$_{\pm0.17}$$\times$ & 0.61$_{\pm0.10}$$\times$ & 0.38$_{\pm0.24}$$\times$ & 0.60$_{\pm0.16}$$\times$ & 0.77$_{\pm0.06}$$\times$ & 0.69$_{\pm0.07}$$\times$ & 0.58$_{\pm0.13}$$\times$ & 0.94$_{\pm0.10}$$\times$ & 0.85$_{\pm0.09}$$\times$ & \textbf{0.90$_{\pm0.08}$$\times$} \\
\quad DeepConf bottom-10\% & 1.07$_{\pm0.11}$$\times$ & 1.12$_{\pm0.10}$$\times$ & 1.14$_{\pm0.18}$$\times$ & 0.82$_{\pm0.14}$$\times$ & 0.66$_{\pm0.37}$$\times$ & 0.91$_{\pm0.22}$$\times$ & 0.95$_{\pm0.07}$$\times$ & 0.90$_{\pm0.10}$$\times$ & 1.00$_{\pm0.26}$$\times$ & 1.02$_{\pm0.08}$$\times$ & 0.95$_{\pm0.11}$$\times$ & 1.04$_{\pm0.11}$$\times$ \\
\bottomrule
\end{tabular}%
}%
\end{table}

\begin{table}[H]
\centering
\caption{\textbf{Token efficiency, Ministral3-14B (smaller is better).} Cells $<$1 indicate more cost-efficient than Standard MV. ``N/A'' indicates target unreachable. Bold marks the best per column.}
\label{tab:token_savings_ministral3_14b_all}
\scriptsize
\setlength{\tabcolsep}{1.5pt}
\renewcommand{\arraystretch}{0.8}
\resizebox{\textwidth}{!}{%
\begin{tabular}{l ccc ccc ccc ccc}
\toprule
 & \multicolumn{3}{c}{FrontierScience-Olympiad} & \multicolumn{3}{c}{HMMT Feb~2026} & \multicolumn{3}{c}{AIME~2025} & \multicolumn{3}{c}{Brumo~2025} \\
\cmidrule(lr){2-4} \cmidrule(lr){5-7} \cmidrule(lr){8-10} \cmidrule(lr){11-13}
 & $\alpha{=}75\%$ & $\alpha{=}90\%$ & $\alpha{=}99\%$ & $\alpha{=}75\%$ & $\alpha{=}90\%$ & $\alpha{=}99\%$ & $\alpha{=}75\%$ & $\alpha{=}90\%$ & $\alpha{=}99\%$ & $\alpha{=}75\%$ & $\alpha{=}90\%$ & $\alpha{=}99\%$ \\
\midrule
Pass@1 / Standard MV plateau & .091 / .142 &  &  & .270 / .463 &  &  & .345 / .576 &  &  & .421 / .698 &  &  \\
Standard MV budget ($B_{\mathrm{MV}}$) & 118k & 249k & 924k & 258k & 1.2M & 7.4M & 371k & 2.0M & 6.9M & 92k & 265k & 1.8M \\
\midrule
\multicolumn{13}{l}{\textit{DeepConf}} \\
\quad DeepConf first-token & 1.10$_{\pm0.10}$$\times$ & 1.24$_{\pm0.37}$$\times$ & 1.27$_{\pm1.04}$$\times$ & 1.09$_{\pm0.13}$$\times$ & 1.24$_{\pm0.28}$$\times$ & 1.00$_{\pm0.30}$$\times$ & 1.24$_{\pm0.15}$$\times$ & 1.04$_{\pm0.11}$$\times$ & 1.02$_{\pm0.20}$$\times$ & 1.04$_{\pm0.08}$$\times$ & 1.01$_{\pm0.16}$$\times$ & 0.96$_{\pm0.48}$$\times$ \\
\quad Self-certainty & 1.62$_{\pm0.20}$$\times$ & $>$10$\times$ & N/A & 1.29$_{\pm0.17}$$\times$ & N/A & N/A & 1.21$_{\pm0.15}$$\times$ & 1.20$_{\pm0.12}$$\times$ & N/A & 1.35$_{\pm0.11}$$\times$ & 1.47$_{\pm0.22}$$\times$ & N/A \\
\quad DeepConf bottom-10\% & 1.49$_{\pm0.17}$$\times$ & 2.25$_{\pm1.08}$$\times$ & N/A & 1.23$_{\pm0.16}$$\times$ & N/A & N/A & 1.20$_{\pm0.15}$$\times$ & 1.24$_{\pm0.16}$$\times$ & N/A & 1.29$_{\pm0.12}$$\times$ & 1.42$_{\pm0.22}$$\times$ & N/A \\
\bottomrule
\end{tabular}%
}%
\end{table}

\endgroup

\subsection{Pool Coverage vs.\ Reweighting}
\label{app:oracle_decomposition}

We show that PC-WMV's advantage over the non-PC baselines comes from reweighting, not from an enlarged pool of correct candidates.

\paragraph{Oracle upper bounds.}
We define Oracle (Standard MV) as the accuracy of an ideal selector that, 
at each budget $B$, always picks the correct answer from the initial-sample 
pool whenever one exists, and Oracle (Prefix Consistency) as the same ideal 
selector applied to the combined pool of initial and regenerated answers.
Both oracles are unattainable in practice, since the correct answer is unknown 
at inference time; they serve as upper bounds on what Standard MV and PC-WMV 
can achieve, respectively.
Figure~\ref{fig:cost_accuracy_oracle} replots Figure~\ref{fig:cost_accuracy_full} restricted to Standard MV, PC-cubic, and the two oracle upper bounds, with the y-axis chosen so the oracles stay in view at high token budgets. Oracle (Standard MV) and Oracle (Prefix Consistency) coincide within $\pm 2\times$ the Monte Carlo standard error on most of the 20 cells, and the residual gap between them is small relative to the PC-cubic vs.\ Standard MV gap in the same cells. At a fixed token budget $B$, the union of correct candidates reachable from sampled groups is therefore essentially the same as that reachable from the equivalent number of initial samples drawn under the Standard MV protocol.
Three factors related to pool coverage explain the small residuals. First, at $\tau{=}0.75$, $K{=}1$ each PC group costs ${\approx}1.25{\times}$ a Standard MV sample,\footnote{The factor $1.25 = 1 + (1-\tau)$ is the truncation-implied expectation. Continuation/generation token ratios in Table~\ref{tab:token_stats} give empirical per-group cost factors of $1.25$--$1.34{\times}$ on four of the five models and $1.49$--$1.60{\times}$ on Ministral3-14B.} so at budget $B$ PC pools roughly $1.6{\times}$ as many candidate answers as Standard MV. Second, $\tilde a_i$ is correlated with $a_i$ through the shared prefix $y_i[{:}\lceil \tau |y_i| \rceil]$, so the effective number of independent answers in the group pool is well below $1.6{\times}$, nearly canceling the first effect. Third, the asymptotic pool of correct candidates is by construction weakly larger for the group pool than for the initial pool, which puts Oracle (Prefix Consistency) weakly above Oracle (Standard MV) at high $B$ in a subset of cells. None of these effects involves the reweighting, so they cannot account for the PC-cubic vs.\ Standard MV gap in the same panels. PC-WMV's advantage therefore comes from the weighting itself, with the prefix-consistency score $c_i(a)$ acting as a per-candidate reliability signal (Section~\ref{sec:prefix_consistency}).

\begin{figure}[!htbp]
\centering
\includegraphics[width=\textwidth]{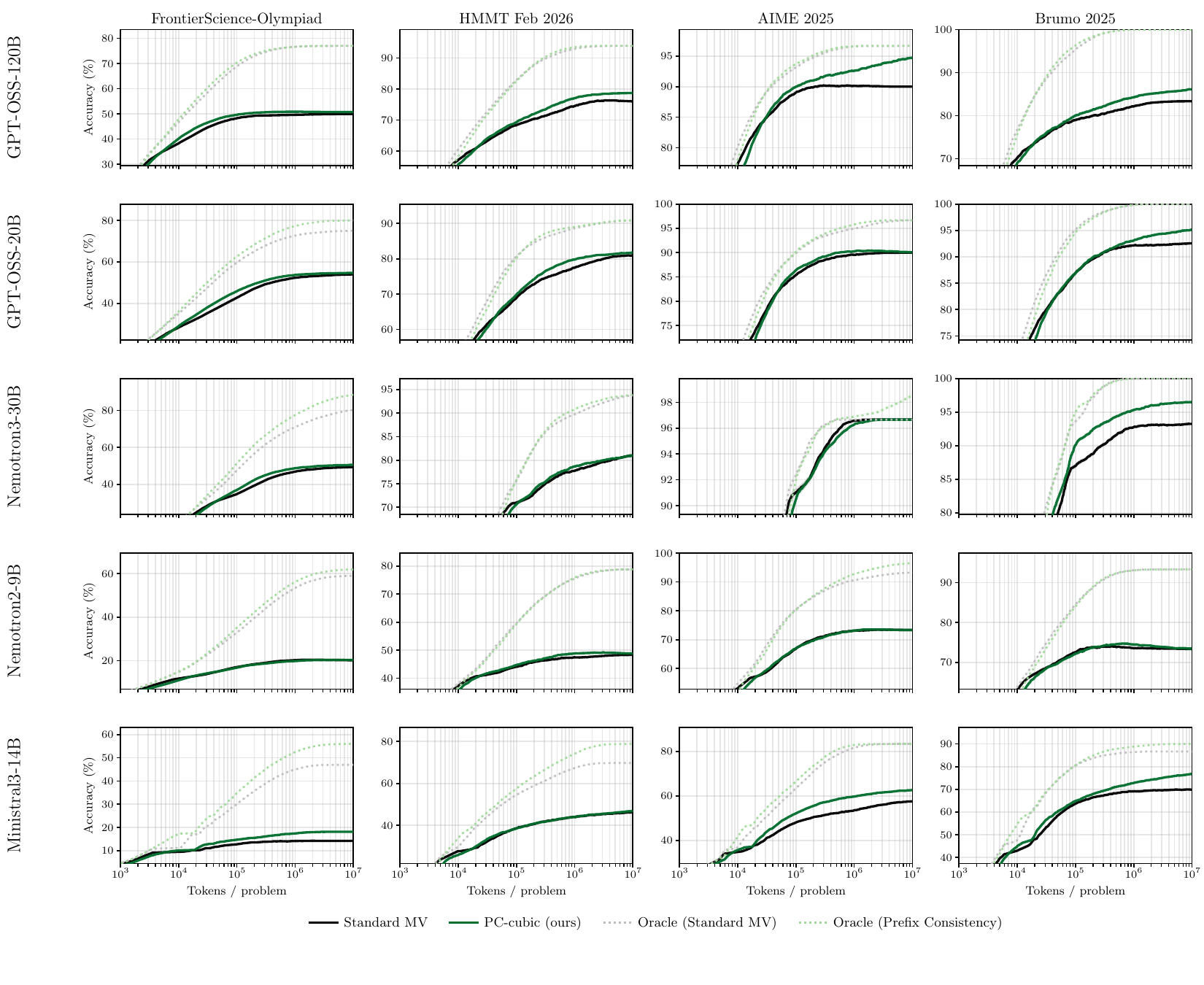}
\caption{\textbf{Oracle decomposition: pool coverage vs.\ reweighting.} Same 5x4 layout as Figure~\ref{fig:cost_accuracy_full}, restricted to Standard MV, PC-cubic, and the two oracle upper bounds. Y-axes include the oracle range so the high-budget tails stay in view. The two oracles nearly coincide in every cell, indicating that PC-WMV's advantage over the non-PC baselines does not come from regeneration adding new correct candidates to the pool. Shaded bands show $\pm 2\times$ the Monte Carlo standard error of accuracy.}
\label{fig:cost_accuracy_oracle}
\end{figure}

\paragraph{Per-problem marginal correctness.}
Figure~\ref{fig:pool_correctness_scatter} gives a complementary per-problem view. Each point is the empirical estimate of the per-problem marginals $\pi(a^\star) = \Pr[a_i = a^\star]$ over the initial answers and $\pi^{\rightarrow}(a^\star) = \Pr[\tilde a_i = a^\star]$ over the regenerations (notation as defined in Appendix~\ref{app:notation}, Eq.~\eqref{eq:pi_rho}). The plotted estimates $\hat\pi(a^\star)$ and $\widehat{\pi^{\rightarrow}}(a^\star)$ cluster tightly on the $y = x$ diagonal in every cell, and their cell-level means agree within a few percentage points on most cells. The largest cell-mean difference is for Nemotron3-30B on AIME~2025, where the cell mean of $\hat\pi(a^\star)$ is $0.90$ and that of $\widehat{\pi^{\rightarrow}}(a^\star)$ is $0.77$, and across the 20 cells the differences run in both directions (e.g., Ministral3-14B has the cell mean of $\widehat{\pi^{\rightarrow}}(a^\star)$ above that of $\hat\pi(a^\star)$ on AIME~2025 and Brumo~2025). A paired Wilcoxon signed-rank test on the per-problem differences $\hat\pi_q(a^\star) - \widehat{\pi^{\rightarrow}}_q(a^\star)$ rejects equality at the Bonferroni-corrected significance level $\alpha_{\mathrm{sig}} = 0.05 / 20$ on 5 of the 20 cells, with absolute cell-mean differences of at most $0.126$ (the Nemotron3-30B AIME~2025 outlier, with the other four cells within $0.061$). Among these 5 detected cells, 4 have the cell mean of $\hat\pi(a^\star)$ above the cell mean of $\widehat{\pi^{\rightarrow}}(a^\star)$, that is, regenerated answers are statistically less correct on average than initial answers (the only exception is Ministral3-14B FrontierScience-Olympiad). Across all 20 cells the cell-mean difference is positive on 13 cells and negative on 7, indistinguishable from a $50{/}50$ split (two-sided binomial sign test, $p = 0.26$). The marginal correctness rate is therefore comparable for initial and regenerated answers, and the cell-mean differences do not systematically favor the regenerations. This is a stronger statement than mere equality of rates: PC-WMV's advantage cannot come from regeneration raising the per-answer correctness rate, because in the cells where the rate differs the regenerations have the lower marginal correctness.

\begin{figure}[!htbp]
\centering
\includegraphics[width=0.75\textwidth]{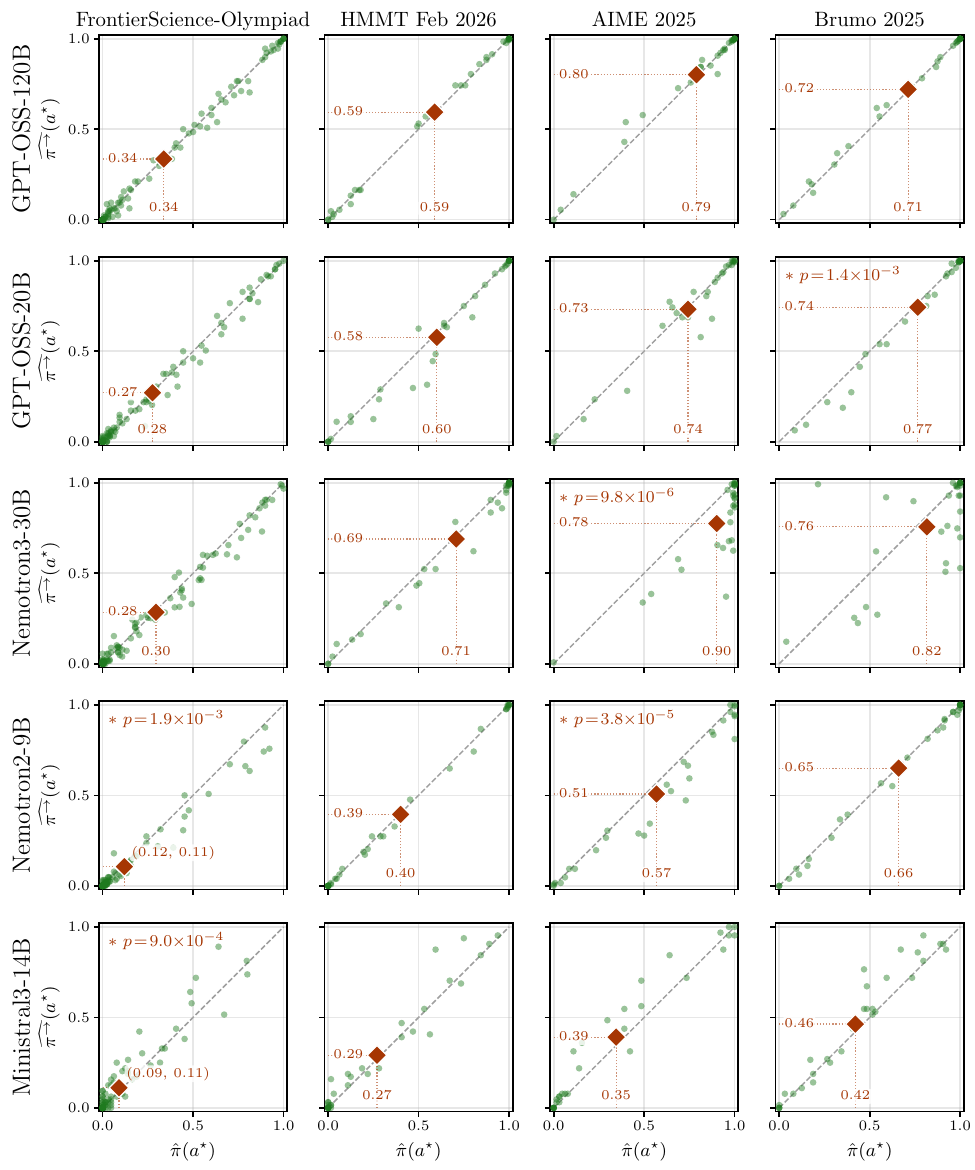}
\caption{\textbf{Per-problem $\hat\pi(a^\star)$ vs.\ $\widehat{\pi^{\rightarrow}}(a^\star)$.} Each point is one problem and the dashed line marks $y = x$. Red diamonds with dotted projection lines show each cell's mean and its coordinates. $\ast$ flags cells where a paired Wilcoxon test rejects equality (Bonferroni-corrected significance level $\alpha_{\mathrm{sig}} = 0.05 / 20$, not to be confused with the target-accuracy ratio $\alpha$ of Section~\ref{sec:token_efficiency}, with the panel-level $p$-value shown).}
\label{fig:pool_correctness_scatter}
\end{figure}

\FloatBarrier
\subsection{Sensitivity to \texorpdfstring{$\tau$}{tau} and \texorpdfstring{$K$}{K}}
\label{app:sensitivity}

Section~\ref{sec:prefix_consistency} fixes $K=1$ to keep notation succinct. Here we extend it to the general-$K$ multiset of Eq.~\eqref{eq:multiset} and the corresponding score $c_i^{(\tau,K)}(a) = |\{a' \in A_i^{(\tau,K)} : a' = a\}| / (K+1)$, sweeping $K \in \{1, 2, 3\}$ and $\tau \in \{0.25, 0.50, 0.75\}$.

Table~\ref{tab:transition_tau} sweeps the truncation fraction $\tau \in \{0.25, 0.50, 0.75\}$ at $K{=}1$ for GPT-OSS-20B. Both $r_C$ and $r_W$ decrease with deeper truncation, with $r_W$ generally decreasing faster (e.g., on AIME~2025, $r_C$ goes $87.8\% \to 73.0\%$ and $r_W$ goes $42.6\% \to 17.1\%$ as $\tau$ moves $0.75 \to 0.25$). The discrimination gap $D(\tau)$ varies by benchmark: it widens monotonically on AIME~2025 ($45.2 \to 50.0 \to 55.9$), peaks at $\tau{=}0.50$ on HMMT and Brumo~2025, and narrows on FrontierScience-Olympiad ($41.4 \to 39.5 \to 39.1$). Deeper truncation also costs more per group, since regeneration covers a longer suffix. Cost-equivalent accuracy across $(\tau, K)$ is reported in Table~\ref{tab:wmv_sensitivity}.

\begin{table}[!htbp]
\centering
\caption{\textbf{Reproduction rates (\%) across truncation fractions $\tau \in \{0.75, 0.50, 0.25\}$ at $K{=}1$, GPT-OSS-20B (larger $D$ is better).} Column symbols ($r_C$, $r_W$, $D$) follow Table~\ref{tab:signal}, with $\overline{\mathrm{AUROC}} = (1+D)/2$ at $K{=}1$. Macro-averaged over problems with at least one correct and one wrong initial sample.}
\label{tab:transition_tau}
\small
\begin{tabular}{@{}lllrrrr@{}}
\toprule
Model & Benchmark & $\tau$ & \textbf{$r_C$} & \textbf{$r_W$} & \textbf{$D$} & $\overline{\mathrm{AUROC}}$ \\
\midrule
GPT-OSS-20B & FrontierScience-Olympiad & 0.75 & \textbf{55.4} & \textbf{14.0} & \textbf{41.4} & 0.707 \\
 &  & 0.50 & \textbf{48.1} & \textbf{8.6} & \textbf{39.5} & 0.697 \\
 &  & 0.25 & \textbf{45.0} & \textbf{5.9} & \textbf{39.1} & 0.696 \\
\midrule
GPT-OSS-20B & HMMT Feb~2026 & 0.75 & \textbf{78.3} & \textbf{29.7} & \textbf{48.6} & 0.743 \\
 &  & 0.50 & \textbf{67.8} & \textbf{18.1} & \textbf{49.7} & 0.748 \\
 &  & 0.25 & \textbf{60.2} & \textbf{12.3} & \textbf{47.9} & 0.740 \\
\midrule
GPT-OSS-20B & AIME~2025 & 0.75 & \textbf{87.8} & \textbf{42.6} & \textbf{45.2} & 0.726 \\
 &  & 0.50 & \textbf{78.6} & \textbf{28.6} & \textbf{50.0} & 0.750 \\
 &  & 0.25 & \textbf{73.0} & \textbf{17.1} & \textbf{55.9} & 0.780 \\
\midrule
GPT-OSS-20B & Brumo~2025 & 0.75 & \textbf{82.8} & \textbf{39.0} & \textbf{43.7} & 0.719 \\
 &  & 0.50 & \textbf{74.8} & \textbf{18.3} & \textbf{56.5} & 0.782 \\
 &  & 0.25 & \textbf{67.6} & \textbf{16.5} & \textbf{51.1} & 0.756 \\
\bottomrule
\end{tabular}
\end{table}

Table~\ref{tab:wmv_sensitivity} reports PC-linear, PC-quadratic, and PC-cubic accuracy across all $(\tau, K)$ configurations on the four GPT-OSS-20B benchmarks at budgets $B \in \{250\text{k}, 1\text{M}, 5\text{M}\}$, with Standard MV at the top for reference.
The main paper fixes $\tau{=}0.75$, $K{=}1$ (Section~\ref{sec:experiments}) as a minimal default that preserves most of the trace and uses a single regeneration per sample. We use the sweep below to characterize how PC-WMV accuracy varies with $(\tau, K)$, not to select the operating point.

\begin{table}
\centering
\caption{\textbf{Sensitivity of PC-WMV accuracy to $(\tau, K)$ on GPT-OSS-20B (higher is better).} Accuracy of PC-linear, PC-quadratic, PC-cubic at budgets 250k / 1M / 5M for every $(\tau, K)$ in the sweep. Standard MV at the top for reference. Bold marks the best $(\tau, K, \text{weight})$ per benchmark at each budget.}
\label{tab:wmv_sensitivity}
\scriptsize
\setlength{\tabcolsep}{2.2pt}
\renewcommand{\arraystretch}{0.95}
\resizebox{\textwidth}{!}{%
\begin{tabular}{ccl *{4}{ccc}}
\toprule
& & & \multicolumn{12}{c}{GPT-OSS-20B} \\
\cmidrule(lr){4-15}
& & & \multicolumn{3}{c}{FrontierScience-Olympiad} & \multicolumn{3}{c}{HMMT Feb~2026} & \multicolumn{3}{c}{AIME~2025} & \multicolumn{3}{c}{Brumo~2025} \\
\cmidrule(lr){4-6} \cmidrule(lr){7-9} \cmidrule(lr){10-12} \cmidrule(lr){13-15}
$\tau$ & $K$ & Method & $B{=}$250k & $B{=}$1M & $B{=}$5M & $B{=}$250k & $B{=}$1M & $B{=}$5M & $B{=}$250k & $B{=}$1M & $B{=}$5M & $B{=}$250k & $B{=}$1M & $B{=}$5M \\
\midrule
\multicolumn{3}{c}{Standard MV} & .482$_{\pm.002}$ & .523$_{\pm.001}$ & .537$_{\pm.001}$ & .736$_{\pm.003}$ & .775$_{\pm.002}$ & .807$_{\pm.001}$ & .881$_{\pm.002}$ & .896$_{\pm.001}$ & .900$_{\pm.000}$ & .901$_{\pm.003}$ & .922$_{\pm.002}$ & .924$_{\pm.001}$ \\
\midrule
\multirow{9}{*}{0.75} & \multirow{3}{*}{1} & PC-linear & .495$_{\pm.002}$ & .524$_{\pm.001}$ & .535$_{\pm.001}$ & .751$_{\pm.003}$ & .797$_{\pm.002}$ & .812$_{\pm.001}$ & .887$_{\pm.002}$ & .899$_{\pm.001}$ & .900$_{\pm.000}$ & .903$_{\pm.003}$ & .928$_{\pm.002}$ & .934$_{\pm.001}$ \\
 & & PC-quadratic & .502$_{\pm.002}$ & .532$_{\pm.001}$ & .537$_{\pm.001}$ & .752$_{\pm.003}$ & \textbf{.800}$_{\pm.002}$ & .814$_{\pm.001}$ & \textbf{.888}$_{\pm.002}$ & .901$_{\pm.002}$ & .901$_{\pm.001}$ & .904$_{\pm.003}$ & .930$_{\pm.002}$ & .940$_{\pm.001}$ \\
 & & PC-cubic & .502$_{\pm.002}$ & .537$_{\pm.001}$ & .545$_{\pm.001}$ & \textbf{.753}$_{\pm.004}$ & .797$_{\pm.002}$ & .814$_{\pm.001}$ & .888$_{\pm.002}$ & \textbf{.902}$_{\pm.002}$ & .902$_{\pm.001}$ & .903$_{\pm.003}$ & .931$_{\pm.002}$ & \textbf{.947}$_{\pm.002}$ \\
\cmidrule(l){2-15}
 & \multirow{3}{*}{2} & PC-linear & .495$_{\pm.002}$ & .525$_{\pm.001}$ & .531$_{\pm.001}$ & .745$_{\pm.004}$ & .795$_{\pm.002}$ & .809$_{\pm.001}$ & .883$_{\pm.002}$ & .896$_{\pm.002}$ & .900$_{\pm.001}$ & .901$_{\pm.003}$ & .927$_{\pm.002}$ & .934$_{\pm.001}$ \\
 & & PC-quadratic & .500$_{\pm.002}$ & .532$_{\pm.001}$ & .537$_{\pm.001}$ & .745$_{\pm.004}$ & .799$_{\pm.002}$ & .815$_{\pm.001}$ & .883$_{\pm.003}$ & .899$_{\pm.002}$ & .901$_{\pm.001}$ & .900$_{\pm.003}$ & .929$_{\pm.002}$ & .942$_{\pm.002}$ \\
 & & PC-cubic & .498$_{\pm.002}$ & .530$_{\pm.001}$ & .540$_{\pm.001}$ & .745$_{\pm.004}$ & .798$_{\pm.003}$ & .817$_{\pm.001}$ & .882$_{\pm.003}$ & .901$_{\pm.002}$ & .905$_{\pm.001}$ & .900$_{\pm.003}$ & .929$_{\pm.002}$ & .946$_{\pm.002}$ \\
\cmidrule(l){2-15}
 & \multirow{3}{*}{3} & PC-linear & .500$_{\pm.002}$ & .528$_{\pm.001}$ & .534$_{\pm.001}$ & .740$_{\pm.004}$ & .793$_{\pm.002}$ & .811$_{\pm.001}$ & .881$_{\pm.002}$ & .896$_{\pm.002}$ & .900$_{\pm.001}$ & .899$_{\pm.003}$ & .926$_{\pm.002}$ & .936$_{\pm.001}$ \\
 & & PC-quadratic & \textbf{.504}$_{\pm.002}$ & .534$_{\pm.001}$ & .541$_{\pm.001}$ & .742$_{\pm.004}$ & .800$_{\pm.002}$ & .817$_{\pm.001}$ & .881$_{\pm.003}$ & .899$_{\pm.002}$ & .903$_{\pm.001}$ & .898$_{\pm.003}$ & .929$_{\pm.002}$ & .943$_{\pm.001}$ \\
 & & PC-cubic & .501$_{\pm.002}$ & .534$_{\pm.001}$ & .544$_{\pm.001}$ & .742$_{\pm.004}$ & .799$_{\pm.003}$ & \textbf{.822}$_{\pm.001}$ & .880$_{\pm.003}$ & .900$_{\pm.002}$ & \textbf{.905}$_{\pm.002}$ & .898$_{\pm.003}$ & .928$_{\pm.002}$ & .946$_{\pm.002}$ \\
\midrule
\multirow{9}{*}{0.50} & \multirow{3}{*}{1} & PC-linear & .489$_{\pm.002}$ & .524$_{\pm.001}$ & .540$_{\pm.001}$ & .728$_{\pm.003}$ & .773$_{\pm.003}$ & .805$_{\pm.002}$ & .886$_{\pm.002}$ & .900$_{\pm.001}$ & .900$_{\pm.000}$ & \textbf{.906}$_{\pm.003}$ & .931$_{\pm.002}$ & .935$_{\pm.001}$ \\
 & & PC-quadratic & .494$_{\pm.002}$ & .531$_{\pm.001}$ & .541$_{\pm.001}$ & .729$_{\pm.003}$ & .779$_{\pm.003}$ & .809$_{\pm.001}$ & .885$_{\pm.002}$ & .899$_{\pm.001}$ & .900$_{\pm.000}$ & .906$_{\pm.003}$ & \textbf{.932}$_{\pm.002}$ & .938$_{\pm.001}$ \\
 & & PC-cubic & .493$_{\pm.002}$ & .532$_{\pm.001}$ & .541$_{\pm.001}$ & .729$_{\pm.004}$ & .780$_{\pm.003}$ & .810$_{\pm.001}$ & .885$_{\pm.002}$ & .898$_{\pm.002}$ & .900$_{\pm.001}$ & .904$_{\pm.003}$ & .931$_{\pm.002}$ & .939$_{\pm.001}$ \\
\cmidrule(l){2-15}
 & \multirow{3}{*}{2} & PC-linear & .492$_{\pm.002}$ & .527$_{\pm.001}$ & .535$_{\pm.001}$ & .727$_{\pm.003}$ & .776$_{\pm.003}$ & .805$_{\pm.002}$ & .883$_{\pm.002}$ & .900$_{\pm.001}$ & .900$_{\pm.000}$ & .891$_{\pm.003}$ & .924$_{\pm.002}$ & .931$_{\pm.001}$ \\
 & & PC-quadratic & .496$_{\pm.002}$ & .533$_{\pm.001}$ & .542$_{\pm.001}$ & .727$_{\pm.004}$ & .780$_{\pm.003}$ & .809$_{\pm.001}$ & .880$_{\pm.003}$ & .900$_{\pm.002}$ & .900$_{\pm.000}$ & .888$_{\pm.003}$ & .924$_{\pm.002}$ & .933$_{\pm.001}$ \\
 & & PC-cubic & .495$_{\pm.002}$ & .532$_{\pm.001}$ & .545$_{\pm.001}$ & .726$_{\pm.004}$ & .779$_{\pm.003}$ & .807$_{\pm.001}$ & .880$_{\pm.003}$ & .898$_{\pm.002}$ & .901$_{\pm.001}$ & .885$_{\pm.003}$ & .922$_{\pm.002}$ & .935$_{\pm.001}$ \\
\cmidrule(l){2-15}
 & \multirow{3}{*}{3} & PC-linear & .493$_{\pm.002}$ & .530$_{\pm.001}$ & .541$_{\pm.001}$ & .721$_{\pm.004}$ & .773$_{\pm.003}$ & .801$_{\pm.002}$ & .881$_{\pm.003}$ & .900$_{\pm.001}$ & .900$_{\pm.000}$ & .889$_{\pm.003}$ & .924$_{\pm.002}$ & .932$_{\pm.001}$ \\
 & & PC-quadratic & .499$_{\pm.002}$ & .535$_{\pm.001}$ & .539$_{\pm.001}$ & .720$_{\pm.004}$ & .779$_{\pm.003}$ & .807$_{\pm.001}$ & .880$_{\pm.003}$ & .899$_{\pm.002}$ & .900$_{\pm.001}$ & .884$_{\pm.003}$ & .923$_{\pm.002}$ & .935$_{\pm.001}$ \\
 & & PC-cubic & .500$_{\pm.002}$ & .537$_{\pm.001}$ & .548$_{\pm.001}$ & .719$_{\pm.004}$ & .778$_{\pm.003}$ & .806$_{\pm.002}$ & .879$_{\pm.003}$ & .895$_{\pm.002}$ & .900$_{\pm.001}$ & .883$_{\pm.003}$ & .919$_{\pm.002}$ & .937$_{\pm.001}$ \\
\midrule
\multirow{9}{*}{0.25} & \multirow{3}{*}{1} & PC-linear & .487$_{\pm.002}$ & .527$_{\pm.001}$ & .538$_{\pm.001}$ & .721$_{\pm.003}$ & .765$_{\pm.003}$ & .798$_{\pm.002}$ & .880$_{\pm.002}$ & .896$_{\pm.001}$ & .900$_{\pm.000}$ & .897$_{\pm.003}$ & .925$_{\pm.002}$ & .933$_{\pm.001}$ \\
 & & PC-quadratic & .494$_{\pm.002}$ & .537$_{\pm.001}$ & .547$_{\pm.001}$ & .723$_{\pm.004}$ & .773$_{\pm.003}$ & .807$_{\pm.002}$ & .879$_{\pm.002}$ & .894$_{\pm.001}$ & .899$_{\pm.000}$ & .899$_{\pm.003}$ & .927$_{\pm.002}$ & .934$_{\pm.001}$ \\
 & & PC-cubic & .496$_{\pm.002}$ & .542$_{\pm.001}$ & .551$_{\pm.001}$ & .723$_{\pm.004}$ & .774$_{\pm.003}$ & .812$_{\pm.001}$ & .879$_{\pm.002}$ & .891$_{\pm.002}$ & .897$_{\pm.001}$ & .898$_{\pm.003}$ & .927$_{\pm.002}$ & .935$_{\pm.001}$ \\
\cmidrule(l){2-15}
 & \multirow{3}{*}{2} & PC-linear & .495$_{\pm.002}$ & .538$_{\pm.001}$ & .549$_{\pm.001}$ & .714$_{\pm.004}$ & .761$_{\pm.003}$ & .790$_{\pm.002}$ & .879$_{\pm.002}$ & .894$_{\pm.001}$ & .900$_{\pm.000}$ & .897$_{\pm.003}$ & .928$_{\pm.002}$ & .933$_{\pm.001}$ \\
 & & PC-quadratic & .500$_{\pm.002}$ & .544$_{\pm.001}$ & .550$_{\pm.001}$ & .715$_{\pm.004}$ & .762$_{\pm.003}$ & .798$_{\pm.002}$ & .877$_{\pm.002}$ & .892$_{\pm.001}$ & .899$_{\pm.001}$ & .893$_{\pm.003}$ & .927$_{\pm.002}$ & .933$_{\pm.001}$ \\
 & & PC-cubic & .500$_{\pm.002}$ & \textbf{.548}$_{\pm.001}$ & .554$_{\pm.001}$ & .716$_{\pm.004}$ & .762$_{\pm.003}$ & .802$_{\pm.002}$ & .877$_{\pm.002}$ & .888$_{\pm.002}$ & .896$_{\pm.001}$ & .891$_{\pm.003}$ & .926$_{\pm.002}$ & .933$_{\pm.000}$ \\
\cmidrule(l){2-15}
 & \multirow{3}{*}{3} & PC-linear & .488$_{\pm.002}$ & .535$_{\pm.001}$ & .549$_{\pm.001}$ & .707$_{\pm.004}$ & .761$_{\pm.003}$ & .793$_{\pm.002}$ & .876$_{\pm.003}$ & .895$_{\pm.001}$ & .900$_{\pm.000}$ & .887$_{\pm.003}$ & .924$_{\pm.002}$ & .930$_{\pm.001}$ \\
 & & PC-quadratic & .493$_{\pm.002}$ & .544$_{\pm.001}$ & .553$_{\pm.001}$ & .708$_{\pm.004}$ & .761$_{\pm.003}$ & .795$_{\pm.002}$ & .875$_{\pm.003}$ & .893$_{\pm.001}$ & .899$_{\pm.001}$ & .883$_{\pm.003}$ & .923$_{\pm.002}$ & .931$_{\pm.001}$ \\
 & & PC-cubic & .493$_{\pm.002}$ & .547$_{\pm.001}$ & \textbf{.562}$_{\pm.001}$ & .709$_{\pm.004}$ & .761$_{\pm.003}$ & .794$_{\pm.002}$ & .874$_{\pm.003}$ & .890$_{\pm.002}$ & .896$_{\pm.001}$ & .882$_{\pm.003}$ & .921$_{\pm.002}$ & .931$_{\pm.001}$ \\
\bottomrule
\end{tabular}%
}%
\end{table}

Two observations follow. At 250k tokens, high-cost configurations (large $K$ and small $\tau$) underperform Standard MV because each regeneration eats into the budget for new groups: $\tau{=}0.25$, $K{=}3$ PC-cubic (cost ${\approx}3.25{\times}$) reaches .493, .709, .874, and .882 on FrontierScience-Olympiad, HMMT, AIME, and Brumo, underperforming Standard MV on the latter three (.736, .881, .901), while the low-cost $\tau{=}0.75$, $K{=}1$ (cost ${\approx}1.25{\times}$) outperforms Standard MV on all four. At 1M--5M budgets, the default $\tau{=}0.75$, $K{=}1$ stays close to the per-cell best on the three math benchmarks (HMMT, AIME, Brumo). FrontierScience-Olympiad is the exception, where some higher-cost configurations improve accuracy (PC-cubic at $\tau{=}0.25$, $K{=}3$ reaches .547 at 1M and .562 at 5M on FrontierScience-Olympiad, against .537 and .545 at $\tau{=}0.75$, $K{=}1$). $K{=}1$ is sufficient on the math benchmarks at 1M, while $K{=}3$ yields modest improvements on HMMT at 5M (PC-cubic .822 at $K{=}3$ vs.\ .814 at $K{=}1$, both at $\tau{=}0.75$). Overall, this GPT-OSS-20B sweep supports the effectiveness of the default $\tau{=}0.75$, $K{=}1$ on the four benchmarks, but does not establish insensitivity to $(\tau,K)$ for the other models.

\subsection{Robustness to Judge Choice}
\label{app:judge_robustness}

We show that PC's per-cell advantage over the best baseline holds regardless of whether answer equivalence is evaluated by the model itself or by a stronger external grader.

The main paper uses each evaluated model itself to determine answer equivalence (Appendix~\ref{app:answer_equivalence_judging}).
One possible concern is that the evaluation may be biased by the limited capability of the model itself. As a robustness check, this section reports results using a stronger external grader. Namely, we re-scored the released pool with Claude Sonnet 4.6~\citep{anthropic2026sonnet46} on four models (GPT-OSS-120B, GPT-OSS-20B, Nemotron3-30B, Nemotron2-9B) across the three LLM-judged benchmarks (Brumo~2025, HMMT Feb~2026, FrontierScience-Olympiad), giving 12 cells. AIME~2025 is exact-match and is therefore judge-free.

We re-cluster the released generation pool using the external judge's YES/NO equivalence judgments instead of the model's own, then recompute per-sample correctness labels along with $\overline{\mathrm{AUROC}}$, PC-WMV accuracy, and the baseline accuracies. No new samples are drawn from the model.

\paragraph{PC's \texorpdfstring{$\overline{\mathrm{AUROC}}$}{AUROC} advantage is preserved on every cell.}
Figure~\ref{fig:judge_robustness_lead} plots, per cell, the gap $\Delta\overline{\mathrm{AUROC}} := \overline{\mathrm{AUROC}}_{\mathrm{PC}} - \overline{\mathrm{AUROC}}_{\mathrm{best\;baseline}}$ under self-judge ($x$) against the same quantity under the external judge ($y$). All 12 points fall in the same quadrant under both judges: 9 cells with $\Delta\overline{\mathrm{AUROC}} > 0$ (PC ahead of the best baseline) and 3 cells with $\Delta\overline{\mathrm{AUROC}} < 0$. No cell flips. Per-cell numbers are in Table~\ref{tab:judge_robustness}.

\begin{figure}[htbp]
\centering
\includegraphics[width=0.55\linewidth]{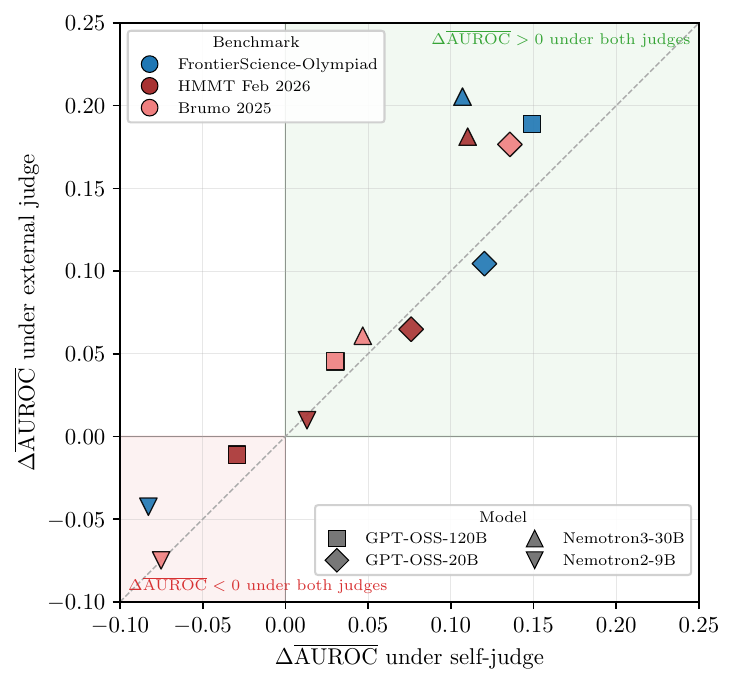}
\caption{\textbf{PC's $\overline{\mathrm{AUROC}}$ advantage is preserved across judges.} Each axis is $\Delta\overline{\mathrm{AUROC}} := \overline{\mathrm{AUROC}}_{\mathrm{PC}} - \overline{\mathrm{AUROC}}_{\mathrm{best\;baseline}}$, under self-judge ($x$) and an external judge (Claude Sonnet 4.6, $y$); one point per cell (12 cells), and the dashed line is $y{=}x$. All 12 points lie in the top-right or bottom-left quadrant; no cell crosses an axis.}
\label{fig:judge_robustness_lead}
\end{figure}

\begin{table}[htbp]
\centering\small
\caption{\textbf{Per-cell PC and best-baseline $\overline{\mathrm{AUROC}}$ under self-judge vs.\ an external judge (Claude Sonnet 4.6).} Bold marks the higher of PC and the best baseline within each cell, with the best baseline's identity in parentheses. AIME~2025 is exact-match (judge-free) and is omitted.}
\label{tab:judge_robustness}
\resizebox{\textwidth}{!}{%
\begin{tabular}{ll cc cc}
\toprule
\multirow{2}{*}{Model} & \multirow{2}{*}{Benchmark} & \multicolumn{2}{c}{PC} & \multicolumn{2}{c}{best baseline} \\
\cmidrule(lr){3-4}\cmidrule(lr){5-6}
 & & self & external & self & external \\
\midrule
\multirow{3}{*}{GPT-OSS-120B} & FrontierScience-Olympiad & \textbf{0.719} & \textbf{0.778} & 0.570\,(Self-certainty) & 0.589\,(Self-certainty) \\
 & HMMT Feb~2026 & 0.698 & 0.711 & \textbf{0.728}\,(DeepConf tail) & \textbf{0.722}\,(DeepConf tail) \\
 & Brumo~2025 & \textbf{0.636} & \textbf{0.660} & 0.606\,(DeepConf tail) & 0.615\,(DeepConf tail) \\
\addlinespace
\multirow{3}{*}{GPT-OSS-20B} & FrontierScience-Olympiad & \textbf{0.707} & \textbf{0.718} & 0.587\,(P(True)) & 0.613\,(P(True)) \\
 & HMMT Feb~2026 & \textbf{0.743} & \textbf{0.740} & 0.667\,(P(True)) & 0.675\,(P(True)) \\
 & Brumo~2025 & \textbf{0.719} & \textbf{0.736} & 0.583\,(P(True)) & 0.559\,(P(True)) \\
\addlinespace
\multirow{3}{*}{Nemotron3-30B} & FrontierScience-Olympiad & \textbf{0.631} & \textbf{0.733} & 0.524\,(Verbal 0--100) & 0.527\,(Verbal 0--100) \\
 & HMMT Feb~2026 & \textbf{0.698} & \textbf{0.766} & 0.588\,(DeepConf tail) & 0.584\,(P(True)) \\
 & Brumo~2025 & \textbf{0.801} & \textbf{0.823} & 0.755\,(DeepConf tail) & 0.762\,(DeepConf tail) \\
\addlinespace
\multirow{3}{*}{Nemotron2-9B} & FrontierScience-Olympiad & 0.537 & 0.603 & \textbf{0.620}\,(P(True)) & \textbf{0.646}\,(P(True)) \\
 & HMMT Feb~2026 & \textbf{0.648} & \textbf{0.648} & 0.635\,(P(True)) & 0.639\,(P(True)) \\
 & Brumo~2025 & 0.633 & 0.633 & \textbf{0.708}\,(P(True)) & \textbf{0.708}\,(P(True)) \\
\bottomrule
\end{tabular}%
}
\end{table}

\FloatBarrier

\paragraph{WMV cost--accuracy curves under both judges nearly overlap.}
Figure~\ref{fig:judge_robustness_grid} overlays the WMV cost--accuracy curves under self-judge (faded) and the external judge (saturated) for the four primary methods (PC-cubic, Standard MV, DeepConf tail, P(True)). On Brumo and HMMT, the two layers are essentially indistinguishable. On FrontierScience-Olympiad, the external-judge curves sit above their self-judge counterparts by a roughly uniform vertical offset; PC-cubic stays the top (or tied-top) curve in every panel where it leads under self-judge, and stays below where it trails. The curves show broadly similar cost--accuracy trade-offs under both judges in this tested subset.

\begin{figure}[htbp]
\centering
\includegraphics[width=0.9\linewidth]{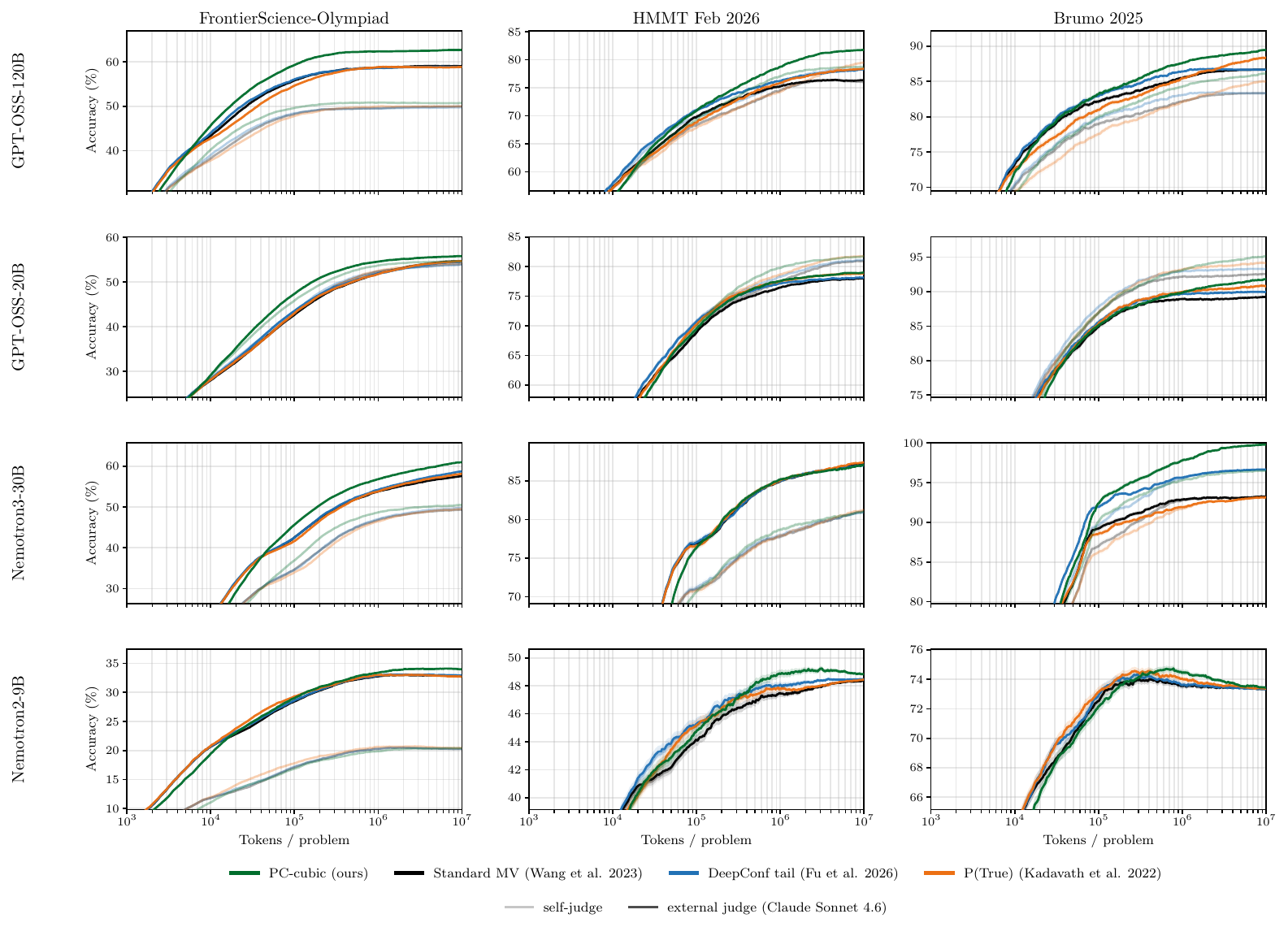}
\caption{\textbf{WMV cost--accuracy curves are stable across judges.} 4 models $\times$ 3 benchmarks; faded curves use the model's own judge, saturated curves use an external judge (Claude Sonnet 4.6). The four methods (PC-cubic, Standard MV, DeepConf tail, P(True)) match Figure~\ref{fig:cost_accuracy_all}. The broad relative performance of the methods is similar across the panels; the largest absolute shifts between the two judges are in the FrontierScience-Olympiad column.}
\label{fig:judge_robustness_grid}
\end{figure}

\paragraph{Pool-level shifts are concentrated on FrontierScience-Olympiad.}
On Brumo and HMMT, the two judges disagree on fewer than $4\%$ of per-sample correctness labels per cell, and Pass@1 differs by at most $0.06$. On FrontierScience-Olympiad the per-sample flip rate is higher ($9\%$ to $17\%$) and Pass@1 rises by up to $0.09$ on three of the four cells; the fourth (GPT-OSS-20B) has compensating up- and down-flips that leave Pass@1 within $0.01$ despite an $8.6\%$ flip rate. The shift reflects the external judge recognizing more equivalent surface forms than the model's own. The added equivalences strengthen PC's relative position rather than weaken it: in Figure~\ref{fig:judge_robustness_lead}, three of the four FrontierScience-Olympiad points (blue) sit above the $y{=}x$ diagonal, and the leftmost column of Figure~\ref{fig:judge_robustness_grid} shows PC's gap over the baselines widening on those cells under the external judge.

\FloatBarrier

\subsection{Reliability Signals vs.\ Pass@1}
\label{app:reliability_vs_pass1}

This subsection backs the analysis of Section~\ref{sec:what_drives} (rising $r_C$ and a smaller, problem-dependent $r_W$ slope within each category) with full panels and slope tables for the reproduction rates $r_C, r_W$ under the GLM estimator (Appendix~\ref{app:pass1_vs_rates_glm}), verifies that the qualitative pattern survives per-benchmark refits (Appendix~\ref{app:pass1_vs_rates_per_benchmark}) and two alternative estimators (Appendix~\ref{app:pass1_vs_rates_binned}), and reports the corresponding Pass@1 view for the per-sample baseline signals (Appendix~\ref{app:pass1_vs_signals}).

\subsubsection{GLM Panels and Slopes}
\label{app:pass1_vs_rates_glm}

Figure~\ref{fig:pass1_vs_rates_remaining} adds GLM panels for the two models not shown in Figure~\ref{fig:pass1_vs_rates} (GPT-OSS-20B, Nemotron2-9B), and Table~\ref{tab:glm_beta} reports the slopes for all five. The qualitative pattern is consistent: $r_C$ rises with Pass@1 on both categories, while $r_W$ has a smaller, problem-dependent slope and sits at a lower level on Science than on Math. The GLM fit kernel and cluster-bootstrap inference protocol below are shared with the alternative estimators of Appendix~\ref{app:pass1_vs_rates_binned}.

\begin{figure}[!htbp]
\centering
\resizebox{0.667\textwidth}{!}{%
  \includegraphics{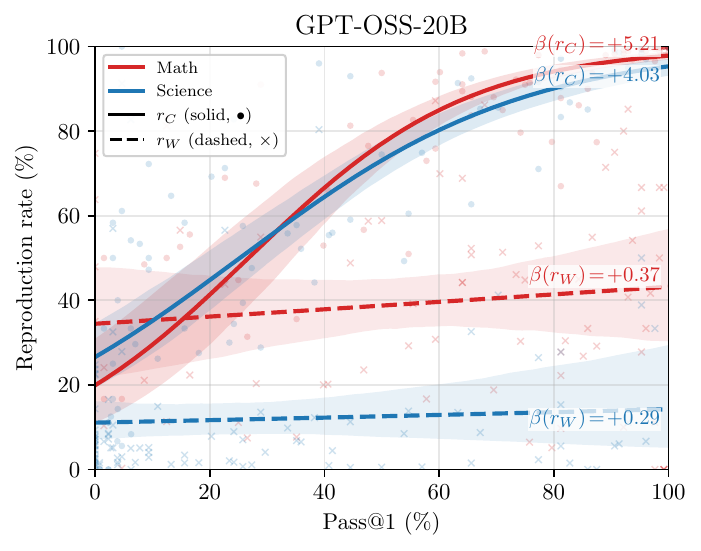}%
  \includegraphics{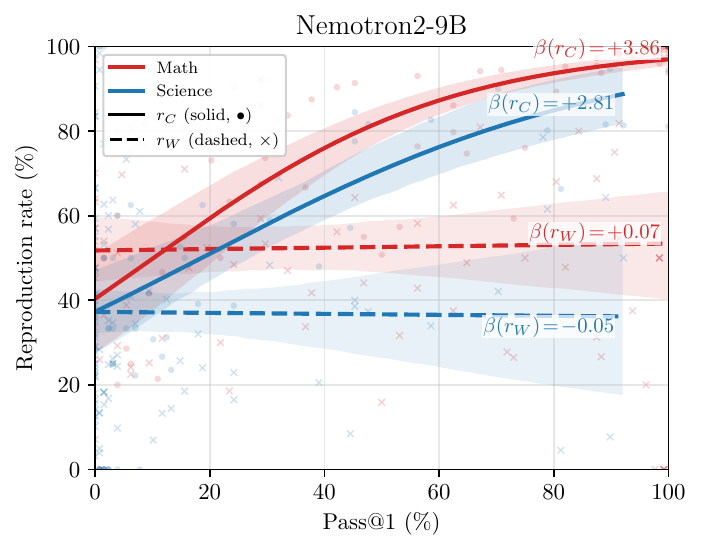}%
}
\caption{\textbf{Per-problem reproduction rates vs.\ Pass@1, remaining models.} GLM fits for GPT-OSS-20B and Nemotron2-9B, complementing Figure~\ref{fig:pass1_vs_rates}. Curves and overlays follow the same convention.}
\label{fig:pass1_vs_rates_remaining}
\end{figure}

\begin{table}[!htbp]
\centering
\caption{\textbf{Logistic GLM slope estimates on $\textnormal{logit}(r) = \beta_0 + \beta \cdot$ \textnormal{Pass@1} per (model, category).} Bootstrap percentile CIs are from cluster bootstrap over problems (1000 replicates). The $p$-column gives the two-sided bootstrap $p$-value for $H_0: \beta = 0$, with ``$<$.001'' indicating that no replicate crossed zero.}
\label{tab:glm_beta}
\footnotesize
\begin{tabular}{llccccc}
\toprule
\multirow{2}{*}{Model} & \multirow{2}{*}{Category} & \multicolumn{2}{c}{$\beta(r_C)$} & \multicolumn{2}{c}{$\beta(r_W)$} \\
\cmidrule(lr){3-4}\cmidrule(lr){5-6}
 & & $\hat\beta$ [CI] & $p$   & $\hat\beta$ [CI] & $p$ \\
\midrule
\multirow{2}{*}{GPT-OSS-120B} & Science & $+4.28$ $[+3.51, +5.14]$ & $<$.001 & $-0.06$ $[-1.20, +0.99]$ & $0.96$ \\
 & Math & $+5.10$ $[+4.05, +6.26]$ & $<$.001 & $+1.14$ $[+0.04, +2.30]$ & $0.034$ \\
\addlinespace
\multirow{2}{*}{GPT-OSS-20B} & Science & $+4.03$ $[+3.36, +4.77]$ & $<$.001 & $+0.29$ $[-1.05, +1.41]$ & $0.66$ \\
 & Math & $+5.21$ $[+4.38, +6.25]$ & $<$.001 & $+0.37$ $[-0.67, +1.56]$ & $0.48$ \\
\addlinespace
\multirow{2}{*}{Nemotron3-30B} & Science & $+3.67$ $[+2.90, +4.37]$ & $<$.001 & $+0.14$ $[-1.03, +1.20]$ & $0.8$ \\
 & Math & $+3.18$ $[+2.17, +4.37]$ & $<$.001 & $-0.20$ $[-1.40, +0.79]$ & $0.67$ \\
\addlinespace
\multirow{2}{*}{Nemotron2-9B} & Science & $+2.81$ $[+1.89, +4.07]$ & $<$.001 & $-0.05$ $[-1.25, +0.90]$ & $0.89$ \\
 & Math & $+3.86$ $[+3.21, +4.69]$ & $<$.001 & $+0.07$ $[-0.69, +0.74]$ & $0.85$ \\
\addlinespace
\multirow{2}{*}{Ministral3-14B} & Science & $+2.66$ $[+1.27, +4.09]$ & $0.002$ & $-0.61$ $[-1.98, +0.55]$ & $0.27$ \\
 & Math & $+2.94$ $[+2.22, +3.67]$ & $<$.001 & $-0.75$ $[-1.45, -0.08]$ & $0.022$ \\
\bottomrule
\end{tabular}
\end{table}

\paragraph{GLM fit kernel.}
For each outcome $Y \in \{C, W\}$ (regenerations seeded from a correct ($C$) or wrong ($W$) trace, matching $r_C$ and $r_W$ in Section~\ref{sec:prefix_consistency}), fit a logistic regression $\Pr(Y_j = 1 \mid \text{Pass@1}_j) = \sigma(\beta_{0,Y} + \beta_Y \cdot \text{Pass@1}_j)$ on trial-level data, where each trial $j$ inherits the Pass@1 of its problem, by maximum likelihood (\texttt{statsmodels.GLM}, Binomial family, default iteratively reweighted least squares solver). Reported quantities are the slope $\beta_Y$ and the predicted curve on a fixed Pass@1 grid, masked to the cell's observed support.

\paragraph{Inference protocol.}
A (model, category) cell consists of $N$ per-problem records, each carrying its Pass@1 and two Bernoulli trial streams: $r_{C,i,j} \in \{0,1\}$ for the $n_{C,i}$ regenerations seeded from a correct trace and $r_{W,i,j}$ for the $n_{W,i}$ from a wrong trace ($r = 1$ if the regenerated answer matches the seed, $0$ otherwise). Each problem $i$ contributes $n_{C,i}$ trials $\bigl(\text{Pass@1}_i,\, r_{C,i,j}\bigr)$ to the $r_C$ regression (and $n_{W,i}$ trials $\bigl(\text{Pass@1}_i,\, r_{W,i,j}\bigr)$ to $r_W$), where the same Pass@1 is shared across all trials from the same problem. We refer to this trial-level dataset as the expanded trials. Point estimates apply each estimator's fit kernel $\hat f_Y$ ($Y \in \{C, W\}$) to the expanded trials.

Confidence intervals are derived using $B = 1000$ cluster-bootstrap iterations. At iteration $b$, we resample $N$ problem indices with replacement, then refit $\hat f_C^{(b)}$ and $\hat f_W^{(b)}$ on the corresponding trial set. 
Within-problem correlation among regenerations that share a prefix is corrected for by resampling problems rather than trials; trial-level Bernoulli standard errors would otherwise be underestimated by a factor of $2$ to $4$ on these data. Because $\hat f_C^{(b)}$ and $\hat f_W^{(b)}$ share the same resample within an iteration, the difference $D^{(b)} = \hat f_C^{(b)} - \hat f_W^{(b)}$ inherits the correct joint distribution, and CIs for $D$ are read off $\{D^{(b)}\}$ directly rather than summed from per-fit half-widths.

We report $2\sigma$-equivalent bootstrap percentile CIs: their endpoints are the $\Phi(-2)$ and $\Phi(2)$ quantiles of the bootstrap distribution, matching the central probability mass within two standard deviations of a normal distribution. This specifies the nominal coverage, not a point estimate $\pm 2$ bootstrap standard deviations. Curve intervals are pointwise, taken per Pass@1 grid point for continuous estimators and per bin for the binned estimator. The two-sided bootstrap $p$-value for $H_0: \beta = 0$ in Table~\ref{tab:glm_beta} is $p = \min\{1,\, 2\min(\Pr_b[\beta^{(b)} \le 0],\, \Pr_b[\beta^{(b)} \ge 0])\}$, with ``$<$.001'' indicating no replicate crossed zero. The seed is fixed at $0$. A small convergence check (3 seeds, several $B$, GPT-OSS-120B Math $r_W$) confirmed stability from $B = 1000$.

\FloatBarrier
\subsubsection{Robustness under Per-Benchmark Pooling}
\label{app:pass1_vs_rates_per_benchmark}

Section~\ref{sec:what_drives} pools the three Math benchmarks (HMMT Feb~2026, AIME~2025, Brumo~2025) into a single Math curve per model. Since the three span different Pass@1 ranges within a fixed model, the pooled slope could in principle reflect between-benchmark Pass@1 differences rather than a within-benchmark Pass@1 effect.

\begin{figure}[!htbp]
\centering
\resizebox{0.8\textwidth}{!}{%
  \includegraphics{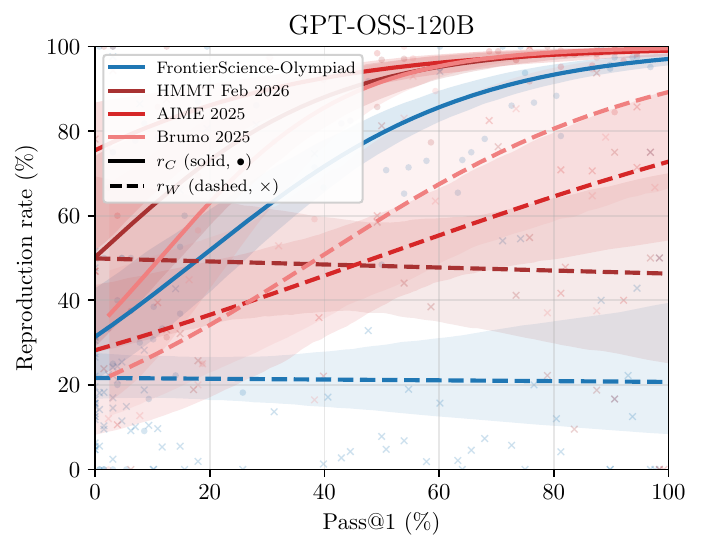}%
  \includegraphics{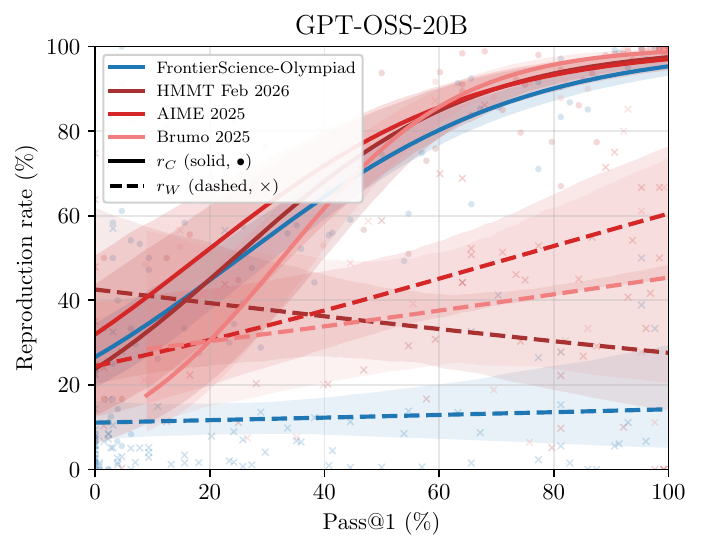}%
  \includegraphics{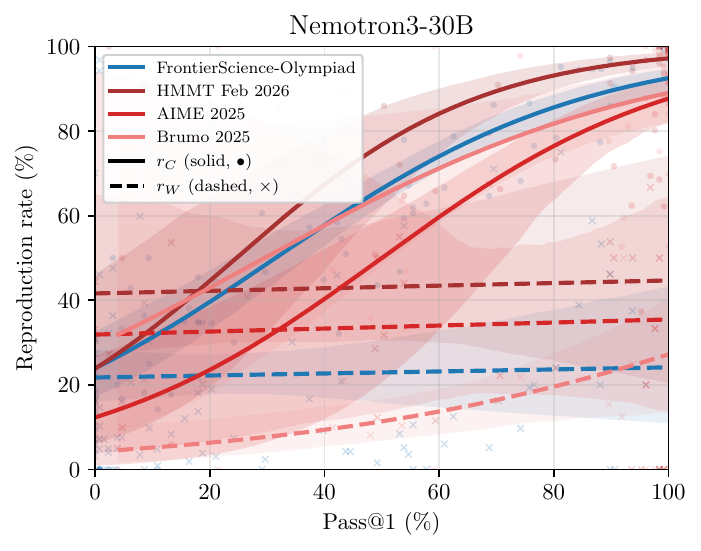}%
}\\[0.5em]
\resizebox{0.533\textwidth}{!}{%
  \includegraphics{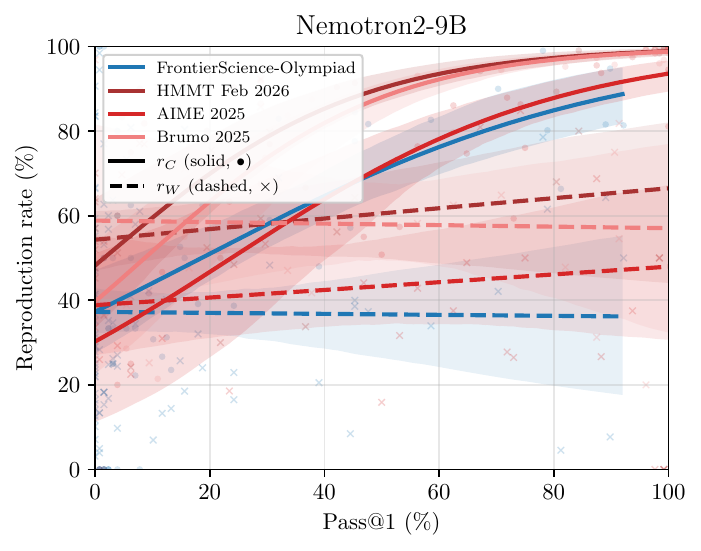}%
  \includegraphics{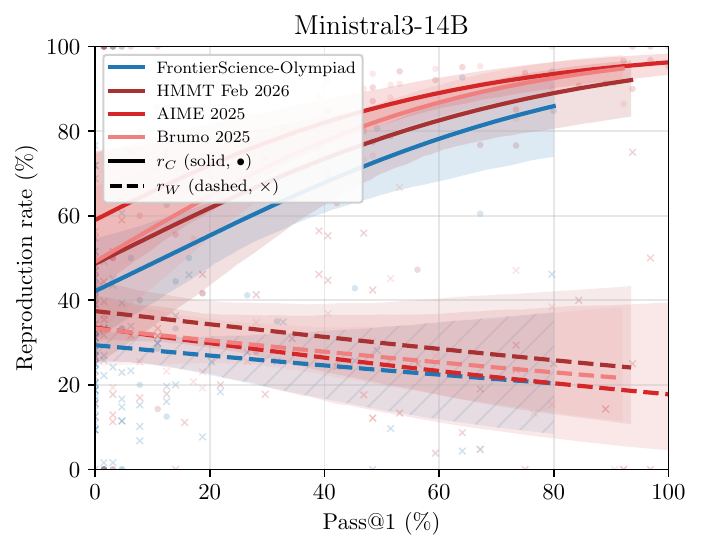}%
}
\caption{\textbf{Per-benchmark reproduction rates vs.\ Pass@1.} One panel per model. Solid ($\bullet$): $r_C$. Dashed ($\times$): $r_W$. Lines are logistic GLM fits with pointwise cluster-bootstrap CIs. Scatter shows per-problem rates. Unlike Figure~\ref{fig:pass1_vs_rates}, the Math benchmarks are fitted separately.}
\label{fig:pass1_vs_rates_per_benchmark}
\end{figure}

\FloatBarrier
For the per-benchmark analysis, we refit the same logistic GLM on each (model, benchmark) cell separately, with the fit kernel and cluster-bootstrap protocol of Appendix~\ref{app:pass1_vs_rates_glm}. Figure~\ref{fig:pass1_vs_rates_per_benchmark} plots the four per-benchmark curves per model in a single panel, and Table~\ref{tab:glm_beta_per_benchmark} reports the slopes.

\begin{table}[!htbp]
\centering
\caption{\textbf{Per-(model, benchmark) GLM slope estimates on $\textnormal{logit}(r) = \beta_0 + \beta \cdot$ \textnormal{Pass@1}.} Per-benchmark variant of Table~\ref{tab:glm_beta}, fit on each (model, benchmark) cell separately. Bootstrap percentile CIs from a cluster bootstrap over problems (1000 replicates), and two-sided bootstrap $p$-values for $H_0: \beta = 0$ (``$<$.001'' indicates that no replicate crossed zero).}
\label{tab:glm_beta_per_benchmark}
\footnotesize
\begin{tabular}{llccccc}
\toprule
\multirow{2}{*}{Model} & \multirow{2}{*}{Benchmark} & \multicolumn{2}{c}{$\beta(r_C)$} & \multicolumn{2}{c}{$\beta(r_W)$} \\
\cmidrule(lr){3-4}\cmidrule(lr){5-6}
 & & $\hat\beta$ [CI] & $p$   & $\hat\beta$ [CI] & $p$ \\
\midrule
\multirow{4}{*}{GPT-OSS-120B} & FrontierScience-Olympiad & $+4.28$ $[+3.51, +5.14]$ & $<$.001 & $-0.06$ $[-1.20, +0.99]$ & $0.96$ \\
 & HMMT Feb~2026 & $+4.97$ $[+3.76, +6.45]$ & $<$.001 & $-0.14$ $[-1.58, +1.65]$ & $0.83$ \\
 & AIME~2025 & $+3.50$ $[+2.01, +5.80]$ & $<$.001 & $+1.92$ $[+0.67, +4.50]$ & $0.002$ \\
 & Brumo~2025 & $+6.21$ $[+4.53, +8.34]$ & $<$.001 & $+3.46$ $[+1.25, +5.75]$ & $0.004$ \\
\addlinespace
\multirow{4}{*}{GPT-OSS-20B} & FrontierScience-Olympiad & $+4.03$ $[+3.36, +4.77]$ & $<$.001 & $+0.29$ $[-1.05, +1.41]$ & $0.66$ \\
 & HMMT Feb~2026 & $+4.82$ $[+3.61, +6.83]$ & $<$.001 & $-0.66$ $[-2.01, +1.17]$ & $0.4$ \\
 & AIME~2025 & $+4.24$ $[+3.05, +6.08]$ & $<$.001 & $+1.55$ $[+0.45, +2.91]$ & $0.006$ \\
 & Brumo~2025 & $+6.55$ $[+5.28, +8.25]$ & $<$.001 & $+0.81$ $[-1.34, +3.16]$ & $0.42$ \\
\addlinespace
\multirow{4}{*}{Nemotron3-30B} & FrontierScience-Olympiad & $+3.67$ $[+2.90, +4.37]$ & $<$.001 & $+0.14$ $[-1.03, +1.20]$ & $0.8$ \\
 & HMMT Feb~2026 & $+4.72$ $[+3.48, +6.71]$ & $<$.001 & $+0.13$ $[-1.75, +1.68]$ & $0.85$ \\
 & AIME~2025 & $+3.93$ $[+2.59, +6.97]$ & $<$.001 & $+0.16$ $[-4.10, +3.47]$ & $0.86$ \\
 & Brumo~2025 & $+2.97$ $[+0.40, +5.26]$ & $0.016$ & $+2.14$ $[+0.68, +3.59]$ & $0.018$ \\
\addlinespace
\multirow{4}{*}{Nemotron2-9B} & FrontierScience-Olympiad & $+2.81$ $[+1.89, +4.07]$ & $<$.001 & $-0.05$ $[-1.25, +0.90]$ & $0.89$ \\
 & HMMT Feb~2026 & $+4.57$ $[+3.43, +5.93]$ & $<$.001 & $+0.51$ $[-0.49, +1.45]$ & $0.26$ \\
 & AIME~2025 & $+3.52$ $[+2.30, +5.31]$ & $<$.001 & $+0.37$ $[-0.74, +1.40]$ & $0.42$ \\
 & Brumo~2025 & $+4.83$ $[+3.96, +6.00]$ & $<$.001 & $-0.07$ $[-2.18, +1.47]$ & $1$ \\
\addlinespace
\multirow{4}{*}{Ministral3-14B} & FrontierScience-Olympiad & $+2.66$ $[+1.27, +4.09]$ & $0.002$ & $-0.61$ $[-1.98, +0.55]$ & $0.27$ \\
 & HMMT Feb~2026 & $+2.68$ $[+0.87, +4.48]$ & $0.006$ & $-0.68$ $[-1.84, +0.49]$ & $0.2$ \\
 & AIME~2025 & $+2.88$ $[+1.71, +4.13]$ & $<$.001 & $-0.85$ $[-2.44, +0.42]$ & $0.22$ \\
 & Brumo~2025 & $+3.19$ $[+2.01, +4.98]$ & $<$.001 & $-0.65$ $[-1.69, +0.56]$ & $0.28$ \\
\bottomrule
\end{tabular}
\end{table}

The two claims in Section~\ref{sec:what_drives} both survive the per-benchmark refit. (i)~$\beta(r_C) > 0$ on all $20$ (model, benchmark) cells with $p < 0.05$ (in fact $p < 0.001$ on $17$ of $20$ cells), so the rising-$r_C$ pattern is within-benchmark and not a pooling artefact. (ii)~$|\beta(r_W)| < \beta(r_C)$ on all $20$ cells, so the smaller-magnitude statement is also within-benchmark.

The per-benchmark view further clarifies the two non-zero pooled $\beta(r_W)$ values from Section~\ref{sec:what_drives}. The GPT-OSS-120B Math value $+1.14$ is driven by AIME~2025 ($+1.92$, $p = 0.002$) and Brumo~2025 ($+3.46$, $p = 0.004$), with HMMT Feb~2026 separately null ($-0.14$, $p = 0.83$): the pooled value reflects a real within-benchmark Pass@1 effect on AIME and Brumo, but the effect is benchmark-specific rather than uniform across Math. The Ministral3-14B Math value $-0.75$ is uniform across all three benchmarks ($-0.85$, $-0.68$, $-0.65$), i.e.\ a within-benchmark effect that holds Math-wide. In neither case is the pooled slope a between-benchmark artefact.

\FloatBarrier
\subsubsection{Robustness under Alternative Estimators}
\label{app:pass1_vs_rates_binned}

In Figure~\ref{fig:pass1_vs_rates}, we fit a logit-linear model. 
To verify that the conclusions in the main paper (that is, rising $r_C$, a smaller and problem-dependent $r_W$ slope within each category, and lower $r_W$ on Science than on Math) do not depend on that form, we rerun the same plot for all five models under two alternative estimators, sharing the cluster-bootstrap protocol of Appendix~\ref{app:pass1_vs_rates_glm}.

\paragraph{Trial-pooled binned estimator (Figure~\ref{fig:pass1_vs_rates_binned}).}
This estimator reads reproduction rates directly off the data, removing the GLM's logit-linear assumption at the cost of bin-edge sensitivity. Bin edges come from the full-data Pass@1 quantiles, with $5$ equal-count bins per category deduplicated by \texttt{np.unique}. A Pass@1 mass at zero can collapse adjacent edges, e.g.\ Ministral3-14B Science with $53$ problems at zero resolves to $3$ effective bins. Each bootstrap iteration assigns problems to those fixed edges and reports the trial-pooled rate $\sum_{i \in \mathrm{bin}} k_{Y,i} \big/ \sum_{i \in \mathrm{bin}} n_{Y,i}$, where $k_{Y,i}$ and $n_{Y,i}$ count reproductions and trials.

\begin{figure}[!htbp]
\centering
\resizebox{0.8\textwidth}{!}{%
  \includegraphics{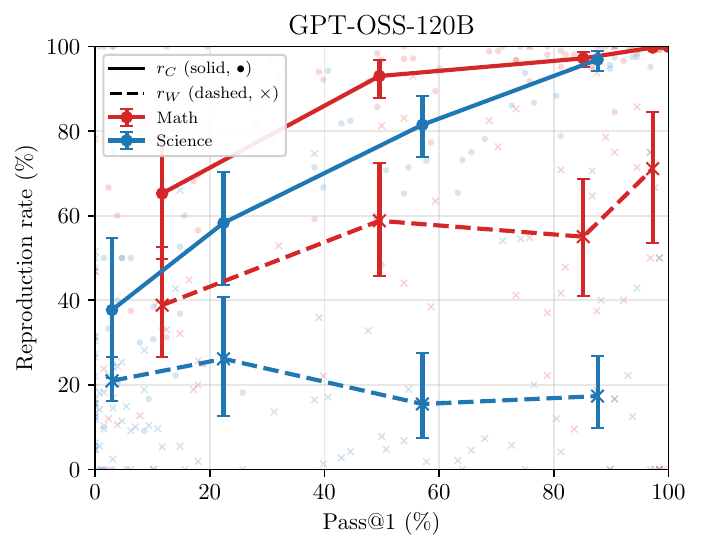}%
  \includegraphics{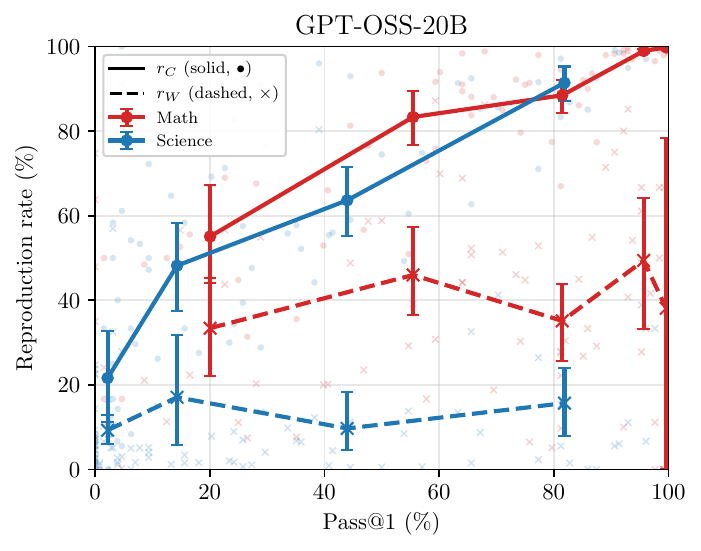}%
  \includegraphics{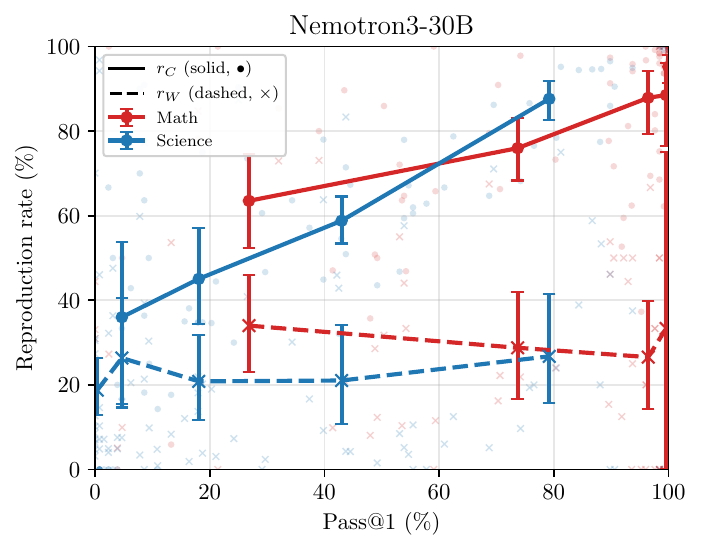}%
}\\[0.5em]
\resizebox{0.533\textwidth}{!}{%
  \includegraphics{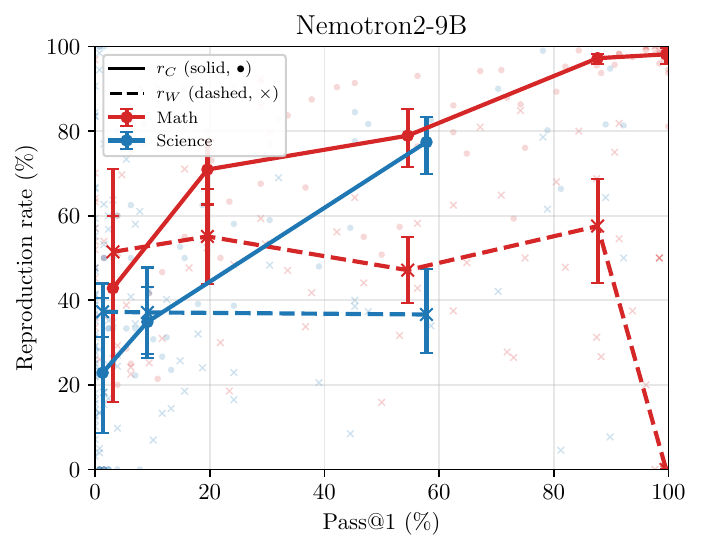}%
  \includegraphics{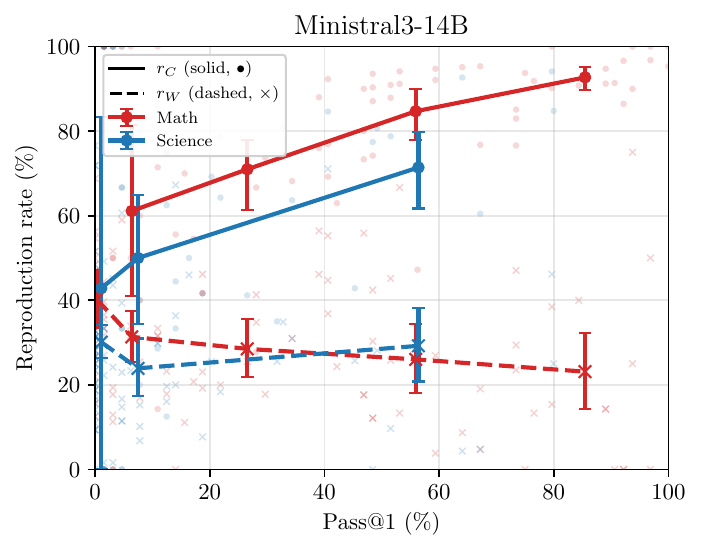}%
}
\caption{\textbf{Binned view of Figure~\ref{fig:pass1_vs_rates}, all five models.} Trial-pooled $r_C$ (solid, $\bullet$) and $r_W$ (dashed, $\times$) over five Pass@1 quantile bins per category, with cluster-bootstrap CIs (1000 replicates, same inference protocol as the GLM in the main paper).}
\label{fig:pass1_vs_rates_binned}
\end{figure}

\paragraph{LOWESS smoother (Figure~\ref{fig:pass1_vs_rates_lowess}).}
Because this estimator does not assume a monotone or parametric relationship, it can reveal any non-monotonic structure in $r_W$ that a logit-linear model would inadvertently hide.

\begin{figure}[!htbp]
\centering
\resizebox{0.8\textwidth}{!}{%
  \includegraphics{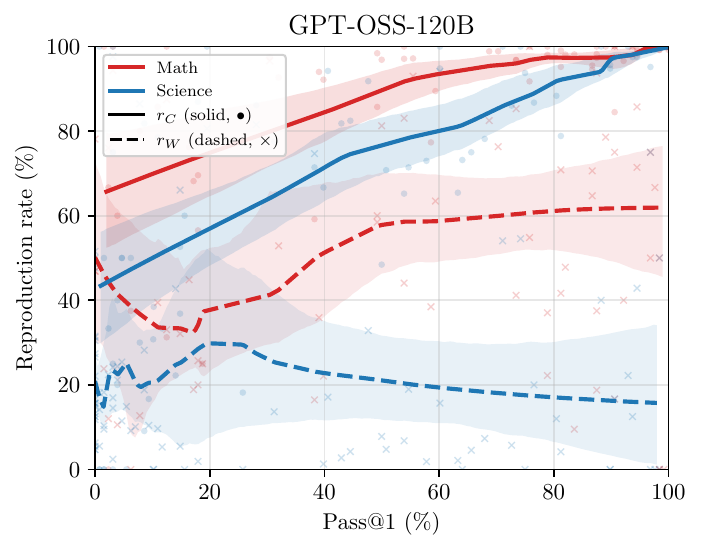}%
  \includegraphics{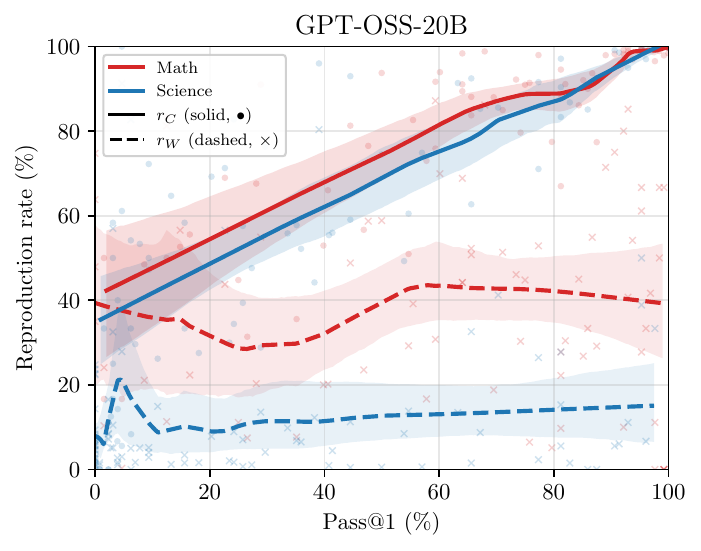}%
  \includegraphics{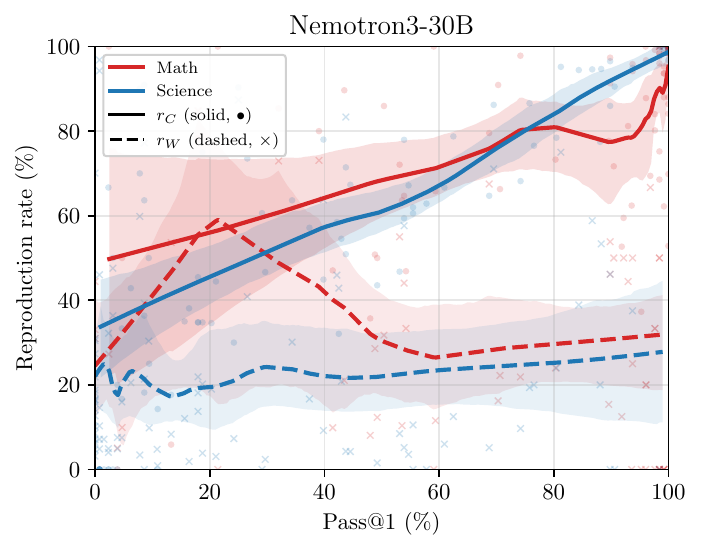}%
}\\[0.5em]
\resizebox{0.533\textwidth}{!}{%
  \includegraphics{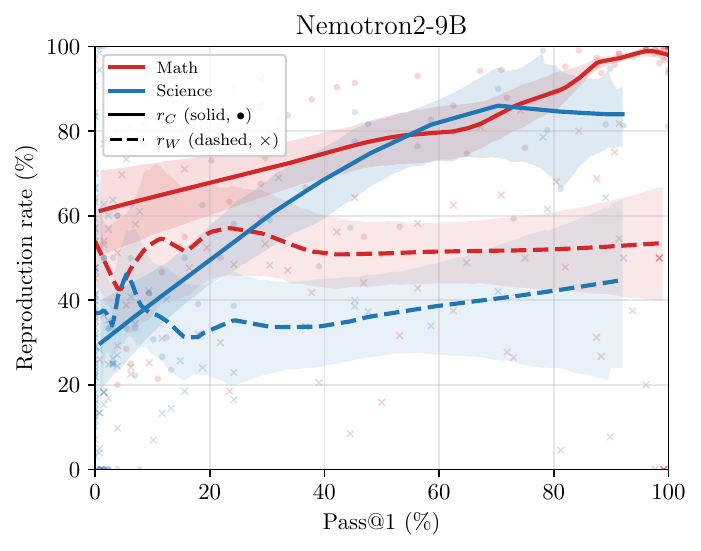}%
  \includegraphics{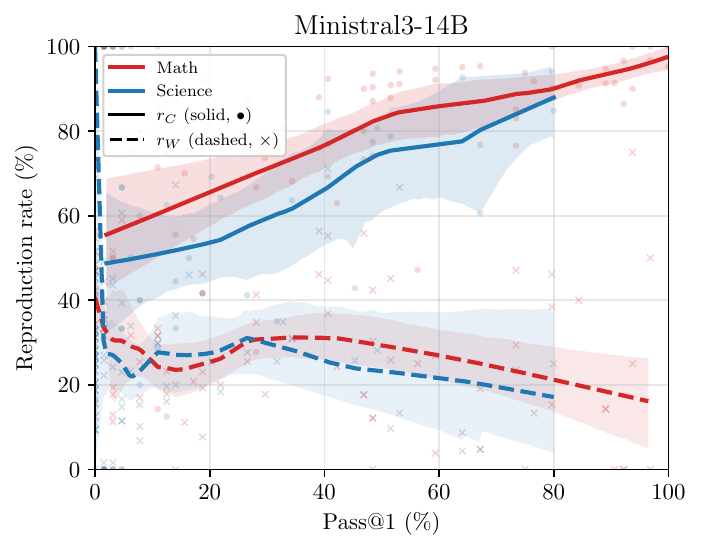}%
}
\caption{\textbf{LOWESS smoother view of Figure~\ref{fig:pass1_vs_rates}, all five models.} $r_C$ (solid) and $r_W$ (dashed) from trial-level LOWESS (span $=0.5$, no robustness reweightings) per category, with pointwise bootstrap CIs from cluster bootstrap over problems (1000 replicates).}
\label{fig:pass1_vs_rates_lowess}
\end{figure}

For the LOWESS fits, we use \texttt{statsmodels.nonparametric.smoothers\_lowess} on trial-level $(\text{Pass@1}_j, Y_j)$ pairs with span $\mathrm{frac} = 0.5$ and robustness iterations $\mathtt{it} = 0$. The default $\mathtt{it} = 3$ uses bisquare residual weights designed for continuous residuals, which collapse Bernoulli outcomes to degenerate $0$ or $100\%$ curves (verified empirically on Science). The fit is linearly interpolated onto the GLM's Pass@1 grid and masked to the observed support.

All three estimators agree qualitatively across all five models: $r_C$ rises across Pass@1 on every (model, category) pair, and $r_W$ has a smaller, problem-dependent slope with the category-specific levels described in Section~\ref{sec:what_drives}.

\FloatBarrier
\subsubsection{Baseline Confidence Signals vs.\ Pass@1}
\label{app:pass1_vs_signals}

For comparison with Figure~\ref{fig:pass1_vs_rates}, we repeat the per-class, Pass@1-conditioned view on the two baseline confidence signals named in the caption of Figure~\ref{fig:signal_distributions}: DeepConf tail (Figure~\ref{fig:pass1_vs_signals_tail}) and P(True) (Figure~\ref{fig:pass1_vs_signals_ptrue}). For each signal, we split per-trial values by whether the trace's answer matches gold (Correct, solid) or not (Wrong, dashed), and fit each class per category with a Gaussian linear model $s = \beta_0 + \beta \cdot \text{Pass@1}$, replacing the Bernoulli logistic GLM of Appendix~\ref{app:pass1_vs_rates_glm} since the response is now a continuous score. The cluster-bootstrap protocol ($B = 1000$ resamples over problems, $2\sigma$ percentile CIs) is unchanged. Scatter overlays show per-problem mean confidence (one dot per problem per class), the continuous-signal analogue of the per-problem rate $k/n$ in Figure~\ref{fig:pass1_vs_rates}.

\begin{figure}[!htbp]
\centering
\resizebox{0.8\textwidth}{!}{%
  \includegraphics{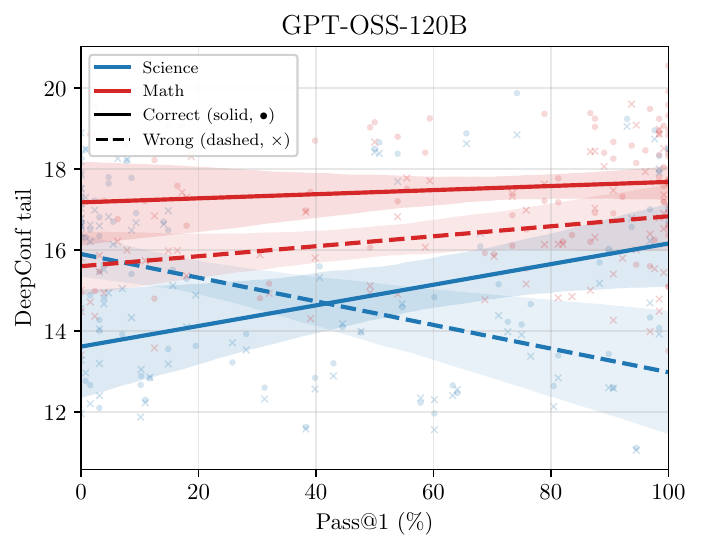}%
  \includegraphics{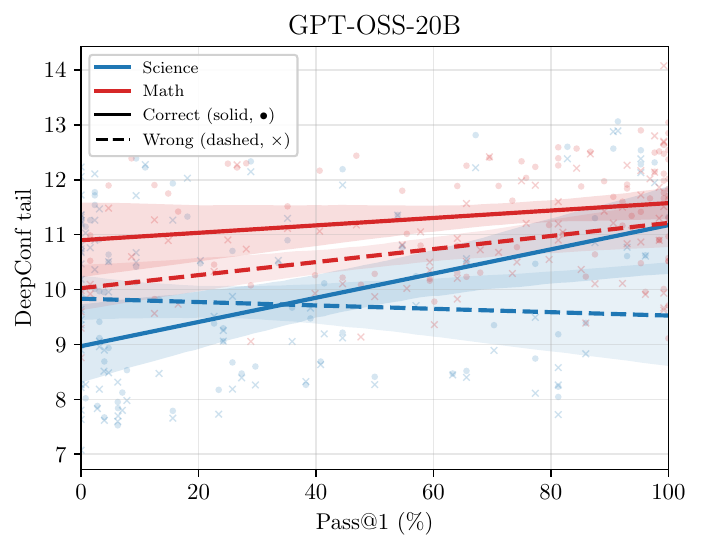}%
  \includegraphics{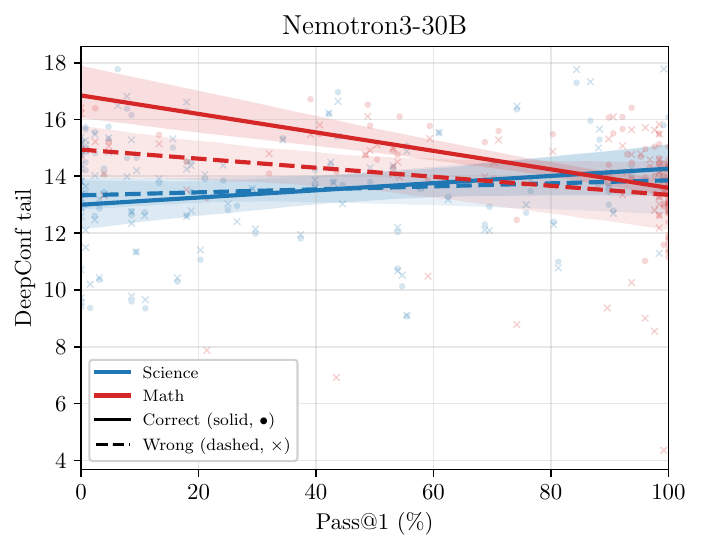}%
}\\[0.5em]
\resizebox{0.533\textwidth}{!}{%
  \includegraphics{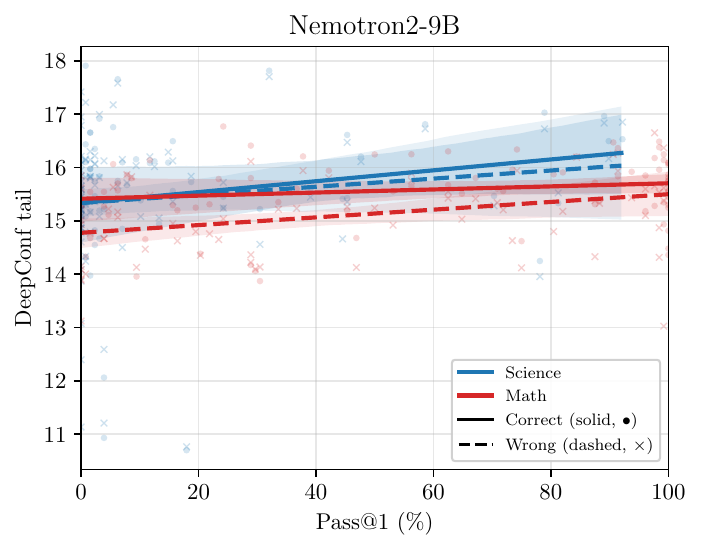}%
  \includegraphics{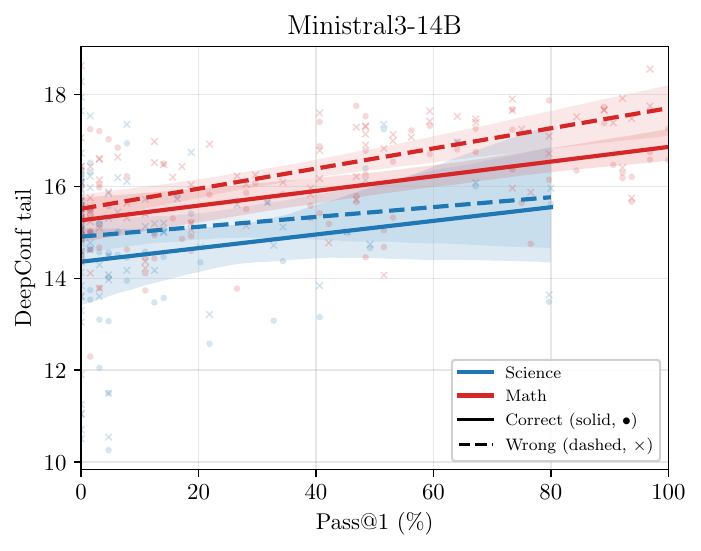}%
}
\caption{\textbf{DeepConf tail vs.\ Pass@1, all five models.} Per-class Gaussian linear fits with cluster-bootstrap CIs over problems. Scatter overlays are per-problem mean confidence.}
\label{fig:pass1_vs_signals_tail}
\end{figure}

\begin{figure}[!htbp]
\centering
\resizebox{0.8\textwidth}{!}{%
  \includegraphics{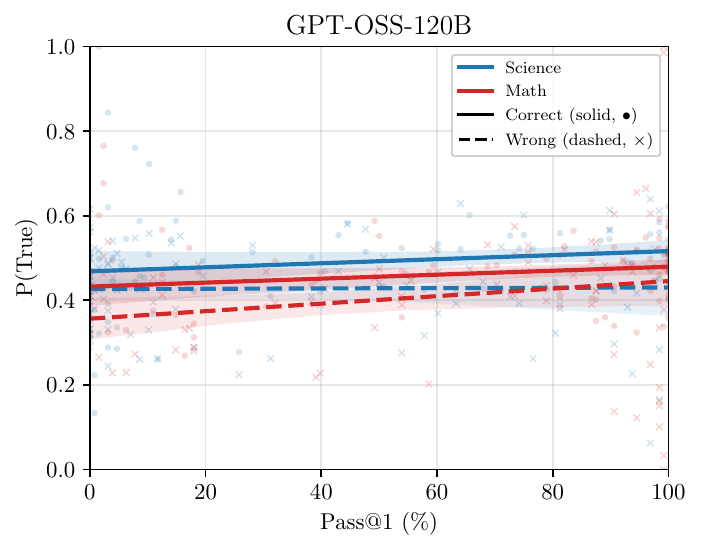}%
  \includegraphics{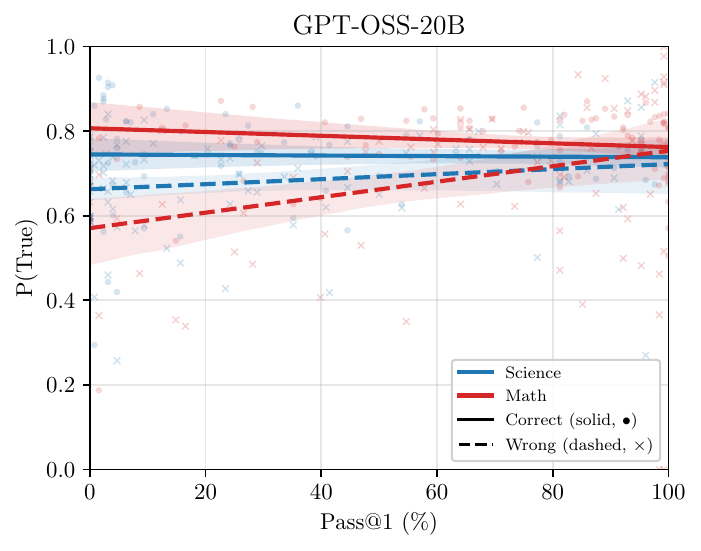}%
  \includegraphics{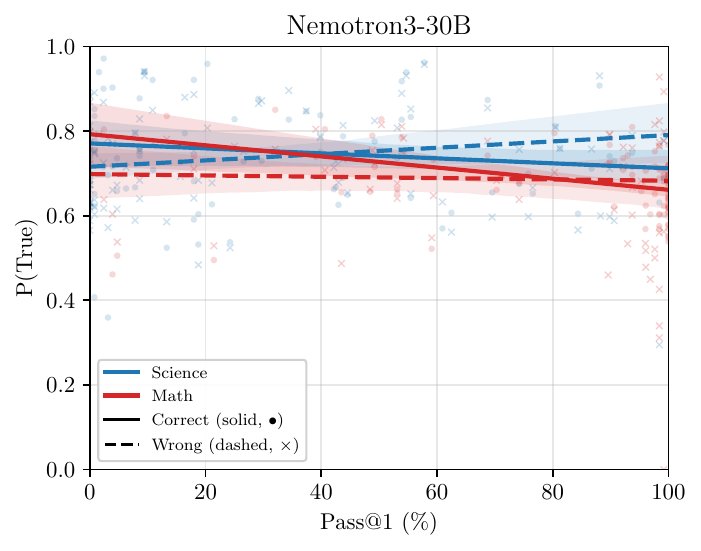}%
}\\[0.5em]
\resizebox{0.533\textwidth}{!}{%
  \includegraphics{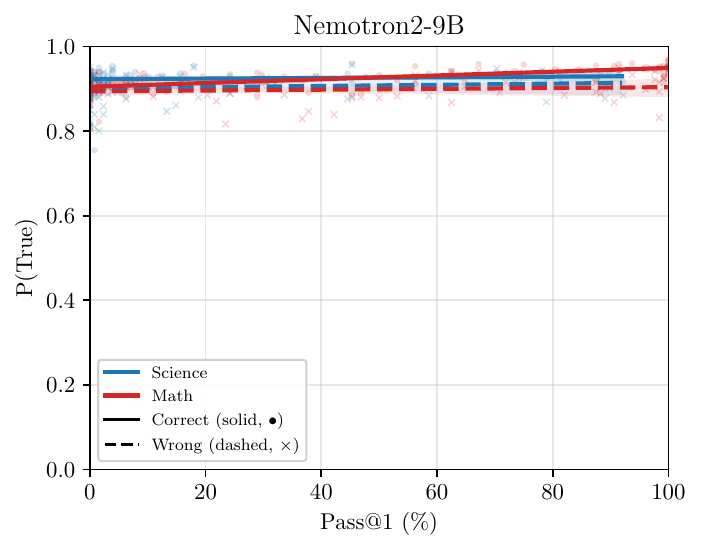}%
  \includegraphics{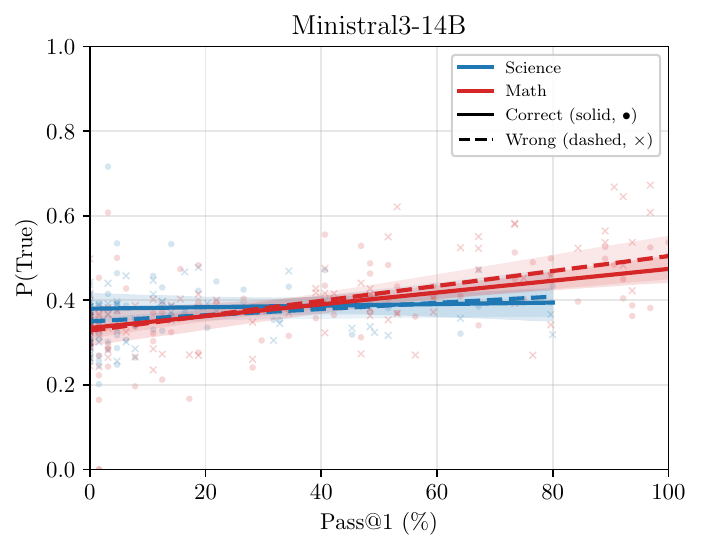}%
}
\caption{\textbf{P(True) vs.\ Pass@1, all five models.} Same protocol as Figure~\ref{fig:pass1_vs_signals_tail}.}
\label{fig:pass1_vs_signals_ptrue}
\end{figure}

\FloatBarrier

\section{Baseline Implementation Details}
\label{app:baseline_impl}

We document here the exact confidence scores and voting rules used for each baseline, since different papers use slightly different weighting conventions.

\paragraph{Glossary of baseline labels.}
The baselines summarized in Section~\ref{sec:experiments} appear under the following labels in our tables and figures (e.g.\ Table~\ref{tab:wmv} and the per-model tables of Appendix~\ref{app:all_baselines}). The per-baseline paragraphs below give the full implementation details.
\begin{itemize}[leftmargin=15pt,itemsep=1pt]
\item \emph{DeepConf}~\citep{fu2025deepthinkconfidence}: a family of trace-level log-probability scores (\emph{first-token}, \emph{Mean} (the \emph{Self-certainty} signal of \citet{kang2025selfcertainty}), \emph{bottom-10\%}, \emph{block-min}, \emph{tail}). Each is plugged into Eq.~\eqref{eq:wmv_prior} with $w$ as the identity. Filtered variants (\emph{top-10\%}, \emph{top-90\%}) keep only the top-$\eta\%$ of traces by the same score before voting.
\item \emph{CISC}~\citep{taubenfeld2025confidenceimprovesselfconsistency}: Confidence-Informed Self-Consistency, a per-sample weighting scheme that draws its raw confidence from one of four sources: \emph{Response probability}~\citep{wang2023selfconsistency} (length-normalized geometric mean of per-token probabilities), \emph{Verbal binary}~\citep{lin2022teaching} (a 0/1 self-rating parsed from a follow-up call), \emph{Verbal 0--100}~\citep{lin2022teaching} (a percentage self-rating parsed from a follow-up call), and \emph{P(True)}~\citep{kadavath2022languagemodelsknowthey} (the renormalized softmax probability of the ``$1$'' token versus ``$0$'' at the rating-call confidence position).
\item \emph{Adaptive stopping}: cost--accuracy curves traced by sweeping the early-stopping hyperparameter of Adaptive Consistency (\emph{AC sweep}~\citep{aggarwal-etal-2023-lets}) and Early-Stopping Self-Consistency (\emph{ESC sweep}~\citep{li2024escape}) over the same initial pool.
\item \emph{SubthoughtReasoner}~\citep{hammoud2025beyondlastanswer}: a single-trace refinement baseline that segments each trace at linguistic cues and votes over per-subthought regenerations. It appears only in the GPT-OSS-20B all-baselines tables (Tables~\ref{tab:wmv_gpt_oss_20b_all} and~\ref{tab:token_savings_gpt_oss_20b_all}).
\item \emph{PC-linear, PC-quadratic, PC-cubic}: PC-WMV (Algorithm~\ref{alg:regen_wmv}) instantiated with the power-family weighting $w^{(n)}(c) = c^n$ at $n = 1, 2, 3$. We use ``PC-cubic'' as shorthand for PC-WMV with $n{=}3$.
\end{itemize}

\paragraph{Deep Think with Confidence (DeepConf)~\citep{fu2025deepthinkconfidence}.}
We compute five trace-level confidence scores from the top-$20$ token log-probabilities saved at generation, following the original implementation.\footnote{\url{https://github.com/facebookresearch/deepconf}}
For a trace $y_i$, let $P_{t,j}$ be the $j$-th largest token probability at position $t \leq |y_i|$. The readout $\ell_i = \{P_{t,j}\}_{t \leq |y_i|,\, j \leq 20}$ is the top-$20$ log-probability record along $y_i$. Write $C_t = -\frac{1}{20}\sum_{j=1}^{20} \log P_{t,j}$ for the negative top-$20$ mean at position $t$. The five trace-level signals $s(y_i, \ell_i)$ are:
\begin{table}[H]
\centering\small
\begin{tabularx}{\linewidth}{@{}lX@{}}
\toprule
Score & Definition \\
\midrule
First-token & $D_{\mathrm{KL}}(\widehat P_1\Vert U)$, where $\widehat P_1$ renormalizes the available finite top-$20$ probabilities at the first generated token, and $U$ is uniform on that support. \\
Mean & $\tfrac{1}{|y_i|}\sum_t C_t$, the average token-level confidence over the trace. \\
Bottom-10\% & Mean of the lowest $\lceil0.1m\rceil$ of $m$ stride-$1$ moving averages over $C_t$. The window length is $1024$, reduced to $\lfloor |y_i|/2\rfloor$ for shorter traces. \\
Block-min & Minimum block-level mean $C_t$, after splitting the trace at double-newline boundaries. \\
Tail & Mean $C_t$ over the last $\min(|y_i|,2024)$ tokens. The value $2024$ is the implementation's literal constant. \\
\bottomrule
\end{tabularx}
\end{table}
The Mean score above is a top-$20$ implementation of the self-certainty signal originally proposed by~\citet{kang2025selfcertainty} on the full vocabulary. We adopt it as our \emph{Self-certainty} baseline and report it under that label throughout the paper.
DeepConf is the case of Eq.~\eqref{eq:wmv_prior} where $s(y_i, \ell_i)$ is one of the five signals above and $w$ is the identity, giving $v_i(a) = s(y_i, \ell_i) \cdot \mathbf{1}[a_i = a]$, where $a_i$ is the answer extracted from trace $i$.
Confidence filtering (Section~3.2 of their paper) optionally restricts the sum to the top-$\eta\%$ of traces by $s(y_i, \ell_i)$ before aggregation.
The unfiltered variant (\texttt{DeepConf-<score>}) sums $v_i(a)$ over all traces, for each of the five scores above.
In addition, we evaluate four filtered variants that combine the two stronger scores with the two retention levels proposed in the paper:
\begin{table}[H]
\centering\small
\begin{tabularx}{\linewidth}{@{}llX@{}}
\toprule
Score & $\eta$ & Aggregation \\
\midrule
Bottom-10\% & $10\%$ & Keep top-$10\%$ by bottom-$10\%$ confidence, then sum $s(y_i, \ell_i) \cdot \mathbf{1}[a_i = a]$ over retained. \\
Bottom-10\% & $90\%$ & Same with $\eta = 90\%$. \\
Tail        & $10\%$ & Keep top-$10\%$ of traces by tail confidence, then sum $s(y_i, \ell_i) \cdot \mathbf{1}[a_i = a]$ over retained. \\
Tail        & $90\%$ & Same with $\eta = 90\%$. \\
\bottomrule
\end{tabularx}
\end{table}
The DeepConf online early-termination mechanism is not used: our evaluation is offline over a pre-generated pool under a shared token budget.

\paragraph{Confidence-Informed Self-Consistency (CISC)~\citep{taubenfeld2025confidenceimprovesselfconsistency}.}
CISC (Definition~3.1 of their paper) extracts a per-sample raw confidence $s_i = s(y_i, \ell_i)$ from one of four sources drawn from prior work, then applies a softmax weighting $w(s_i) = \exp(s_i / T) / \sum_j \exp(s_j / T)$ (over samples) with tunable temperature $T$:
\begin{table}[H]
\centering
\footnotesize
\renewcommand{\arraystretch}{1.15}
\begin{tabularx}{\linewidth}{@{}>{\raggedright\arraybackslash}p{0.25\linewidth}>{\raggedright\arraybackslash}p{0.18\linewidth}X@{}}
\toprule
Source & $\ell_i$ & Raw confidence $s_i$ \\
\midrule
Response probability~\citep{wang2023selfconsistency} & per-token log-probabilities of $y_i$ & Length-normalized geometric mean of token probabilities, $\exp\bigl(\tfrac{1}{|y_i|}\sum_t \log P_t\bigr)$ where $P_t$ is the probability of the generated token at position $t$. \\
Verbal binary~\citep{lin2022teaching} & $\emptyset$ (verbalized) & $\{0, 1\}$ self-rating parsed from text, a simplification of the verbalized paradigm. \\
Verbal 0--100~\citep{lin2022teaching} & $\emptyset$ (verbalized) & Post-hoc percentage-confidence self-rating parsed from text. \\
P(True)~\citep{kadavath2022languagemodelsknowthey} & top-$20$ logprobs at the rating-call confidence token & Renormalized softmax probability of ``$1$'' versus ``$0$'' computed from the top-$20$ logprobs at the rating-call confidence token, $\exp(\ell_1) / (\exp(\ell_0) + \exp(\ell_1))$ when both are present (with single-side fallbacks otherwise). \\
\bottomrule
\end{tabularx}
\end{table}
The original CISC release\footnote{\url{https://github.com/google-research/google-research/tree/master/cisc}} targets non-reasoning models and forces a single-token response (\texttt{temperature=0}, \texttt{max\_tokens=1}). Our reasoning models may emit channel-control tokens or reasoning before the confidence digit because the saved trace leaves the final channel open (e.g., Harmony for GPT-OSS). We therefore allow a separate completion from each saved trace $y_i$, using the model's recommended sampling parameters.

Verbal binary and P(True) share the suffix ``\textit{Now I will rate my confidence in the proposed answer as either 0 or 1. Proposed confidence: (}'', while Verbal 0--100 uses ``\textit{Now I will rate my confidence in the proposed answer on a scale of 0--100. Proposed confidence: (}'' (with leading newline). The wording mirrors the CISC release; one binary call supplies both Verbal binary and P(True). Appendix~\ref{app:answer_confidence_extraction} gives the extraction and missing-score rules.
To avoid a baseline-specific $T$ search that could over-fit our benchmarks, we report the untuned linear weighting $w(s) = s$ for each source (verbal $0$--$100$ rescaled to $[0, 1]$), so $v_i(a) = s(y_i, \ell_i) \cdot \mathbf{1}[a_i = a]$.
Softmax with $T{=}1$ is also computed by our pipeline but omitted from the tables because it consistently underperforms the linear variant. A full $T$ sweep is left to future work.

Verbal baselines are charged $|y_i|$ plus the stored minimum-extraction token count, rather than the full secondary-call output (Appendix~\ref{app:answer_confidence_extraction}). This favors the verbal baselines and is not their full incurred cost. The secondary call's maximum output length matches the initial generation budget (Appendix~\ref{app:reprod}). Response probability has no secondary call and reads the saved log-probabilities for free. PC-cubic sharpens its signal, while these confidence baselines use untuned linear weights; the AUROC comparison is invariant to strictly increasing score transformations.

\paragraph{Adaptive Consistency (AC)~\citep{aggarwal-etal-2023-lets}.}
AC consumes the initial pool one sample at a time and stops at the first $k$ for which a Beta-binomial posterior favors the running top answer over the runner-up by at least $C_\mathrm{thresh}$. Let $n_1$ and $n_2$ be the top-$1$ and top-$2$ answer counts after $k$ samples. The closed-form stopping criterion is
\begin{equation*}
1 - I_{1/2}(n_1 + 1,\, n_2 + 1) \;\geq\; C_\mathrm{thresh},
\end{equation*}
where $I_x(\alpha, \beta)$ is the regularized incomplete beta function. This is the analytic equivalent of the Monte Carlo integral implemented in the official release.\footnote{\url{https://github.com/Pranjal2041/AdaptiveConsistency}}
At the stop time, the prediction is the running mode, so AC adds no per-sample weighting: it is Standard MV truncated at a natural stopping point.
The original paper reports $C_\mathrm{thresh} = 0.95$ as the default and sweeps $[0.5, 1)$ in Figure~2 to trace a cost--accuracy frontier.
We sweep $C_\mathrm{thresh} \in \{0.5, 0.6, 0.7, 0.8, 0.9, 0.95, 0.97, 0.99, 0.995, 0.999\}$. Each value yields a single (cost, accuracy) point at its natural stopping cost, and we label the resulting curve ``AC sweep''.

\paragraph{Early-Stopping Self-Consistency (ESC)~\citep{li2024escape}.}
ESC consumes the initial pool in fixed-size windows of $W$ samples. When a window is unanimous, ESC locks that answer as the final prediction and stops. Otherwise, the window's votes are added to a running counter and a new window is drawn. Before any lock, the intermediate prediction is the mode of the running counter. The paper recommends $W = 5$ for most tasks and $W = 8$ for MATH.
Our streaming reformulation reproduces the lock semantics of the official batch release.\footnote{\url{https://github.com/Yiwei98/ESC}}
We sweep $W \in \{2, 3, \ldots, 10\}$ and label the resulting natural-stopping curve ``ESC sweep''. Like AC, ESC is Standard MV with a window-based stopping rule and adds no per-sample weight.

\paragraph{Cost accounting for AC and ESC.}
Both methods consume only initial samples. Their per-trial budget is the cumulative generated-token length of the consumed initial samples, the same cost accounting used by Standard MV, Self-certainty, DeepConf, and Response probability.
Because each hyperparameter value yields one (cost, accuracy) point at its natural stopping cost rather than a curve over a shared budget grid, we treat the points across the swept hyperparameter as the method's cost--accuracy curve and interpolate it for any required budget.
Sampling is with replacement, so $N$ does not cap the number of consumed samples or bound accuracy by $N$-sample Standard MV. The sampled sequence length is the larger of the maximum evaluated sample count and the maximum token budget divided by the cheapest pool sample's token cost, rounded down. If no stop occurs, the natural-stopping evaluation uses the vote and cost at the end of that finite sequence.

\paragraph{SubthoughtReasoner~\citep{hammoud2025beyondlastanswer}.}
At the time of writing, no official implementation was available, so we reimplemented the method based on the description provided in the original paper.
Our implementation segments each reasoning trace into sequential subthoughts at the linguistic cues described in Section~3 of their paper, regenerates a continuation from the end of each subthought, and takes an unweighted majority vote over the resulting pool of answers.
To bound the per-problem cost of constructing this pool, we process initial samples in order, accumulating each sample's initial tokens together with its subthought-regenerated continuations, and stop once the cumulative cost reaches that of the $N$ initial samples for that problem. Otherwise, the per-problem cost of regenerating subthought continuations from all $N$ initial samples would be several times that of the initial generation alone.
We evaluated this baseline on only a small subset of (model, benchmark) conditions.

\paragraph{Prefix consistency (ours).}
For prefix consistency, $\ell_i = \emptyset$: it reads only generated tokens and regenerated answers, with no log-probability access.
Unlike the per-sample signal $s(y_i, \ell_i)$ used by the baselines above, our signal $c_i^{(\tau)}(a)$ is defined per distinct candidate $a \in A_i^{(\tau)}$ (Section~\ref{sec:prefix_consistency}). Moreover, we use the power-family weighting $w^{(n)}(c) = c^n$ for $n \in \{1, 2, 3\}$ (PC-linear, PC-quadratic, PC-cubic).
PC-linear ($n{=}1$) admits a simple interpretation: substituting $w(c) = c$ and Eq.~\eqref{eq:ci_def} into the PC-WMV vote gives $\sum_i v_i(a) = \tfrac{1}{K+1} \sum_i |\{a' \in A_i^{(\tau)} : a' = a\}|$, which is proportional to the count of $a$ in the combined pool of $N$ initial answers and $N K$ regenerations. PC-linear's argmax therefore coincides with unweighted majority voting on this $(K{+}1)N$ pool, treating each regeneration as one additional vote of equal weight. PC-quadratic and PC-cubic depart from this baseline by giving super-linear weight to candidates that the same group reproduces.

\section{Examples of Prefix Consistency}
Figure~\ref{fig:cot_regen_trajectories} shows two representative examples of how regeneration behaves after truncation.

When the initial answer is correct, the regenerated continuation often recovers the same core reasoning structure and reproduces the correct answer. In contrast, when the initial reasoning is already flawed, regeneration typically does not repair it. Instead, it continues along a similar erroneous line of reasoning and may even produce a different wrong answer. These examples help illustrate why regeneration amplifies the consistency of correct traces while still failing to escape wrong ones.

\begin{figure*}[htbp]
\centering
\scriptsize
\setlength{\abovedisplayskip}{3pt}
\setlength{\belowdisplayskip}{3pt}
\setlength{\abovedisplayshortskip}{3pt}
\setlength{\belowdisplayshortskip}{3pt}
\setlength{\abovecaptionskip}{6pt}
\includegraphics[alt={A correct AIME solution returns 70 before and after regeneration. A wrong solution changes from 46 to 38. Both retain the first 75 percent of the generated solution.}]{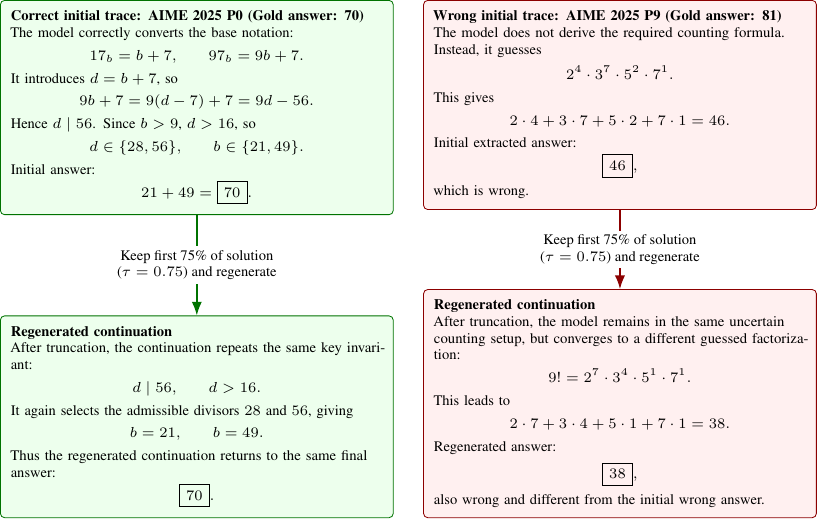}

\caption{
\textbf{Qualitative examples of CoT regeneration after truncation.} A correct initial trace (green, left) and a wrong initial trace (red, right) on AIME~2025. For each trace, the first 75\% ($\tau=0.75$) of the full generated solution (by token count, including the final answer) is retained and the remainder is regenerated.
}
\label{fig:cot_regen_trajectories}
\end{figure*}

\FloatBarrier

\section{Evaluation Protocol}
\label{app:eval_protocol}

This appendix consolidates the analysis pipeline shared across the experimental sections of the main paper. Implementation choices specific to each baseline are in Appendix~\ref{app:baseline_impl}. Hardware and inference settings are in Appendix~\ref{app:reprod}.

\subsection{Setup}
\label{app:eval_setup}

\paragraph{Benchmarks.}
We evaluate on one science benchmark (FrontierScience-Olympiad) and three math benchmarks (HMMT Feb~2026, AIME~2025, Brumo~2025). Table~\ref{tab:datasets} lists the problem count, answer format, and scoring rule of each. We use the full released test split for every benchmark. Answers are extracted from the boxed expression, then math-normalized for the math benchmarks (fraction unification, degree-marker removal) or used with minimal normalization for FrontierScience-Olympiad. Equivalence to the gold answer is decided by exact match for AIME~2025 and by an LLM judge for the other three benchmarks (Appendix~\ref{app:answer_equivalence_judging}). HuggingFace URLs and licenses are in Table~\ref{tab:links}.

\begin{table}[htbp]
    \centering
    \caption{\textbf{Evaluation benchmarks.}}
    \label{tab:datasets}
    \resizebox{\textwidth}{!}{%
    \begin{tabular}{lccll}
    \toprule
        Dataset & \# Problems & Category & Answer format & Scoring \\
    \midrule
        FrontierScience-Olympiad & 100 & Science & Free-form (expression, number, or short text) & LLM judge \\
        HMMT Feb~2026            & 33  & Math    & Integer or closed-form expression             & LLM judge \\
        AIME~2025                & 30  & Math    & Integer in $[0, 999]$                          & Exact match \\
        Brumo~2025               & 30  & Math    & Integer or closed-form expression             & LLM judge \\
        \bottomrule
    \end{tabular}%
    }
\end{table}

\paragraph{Pre-generated pool and cost vectors.}
For each (model, benchmark) cell we draw $N{=}128$ initial generations per problem ($N{=}64$ for Ministral3-14B, where the full $N{=}128$ run had not completed at the time of writing), and for prefix-consistency methods we additionally draw $K$ regenerations per initial sample at the chosen truncation fraction $\tau$. All voting and analysis read from this fixed pool, so the same generated tokens back every comparison and the only randomness in the reported numbers is from the analysis-side resampling described below.
Each generation carries its actual recorded token count. For each method, the cost of a vote is the cumulative generated-token length of every sample it actually reads. Standard MV, Self-certainty, DeepConf, Response probability, AC, and ESC consume only the initial samples. CISC's verbalized sources (Verbal binary, Verbal 0--100) and P(True) additionally consume the secondary verbal-rating completions call per sample described in Appendix~\ref{app:baseline_impl}. PC additionally reads the regenerated continuations from each sample's prefix, charging only newly generated continuation tokens, not the kept prefix again. SubthoughtReasoner reads the continuations regenerated from each subthought boundary.

\paragraph{Hyperparameter defaults.}
Unless otherwise noted, every result in the main paper uses $\tau{=}0.75$, $K{=}1$, and the power-family weighting $w^{(n)}(c) = c^n$ for $n \in \{1, 2, 3\}$. Sensitivity to alternative choices of $\tau$ and $K$ is reported in Appendix~\ref{app:sensitivity}.

\paragraph{Confidence intervals.}
For accuracy and mean cost, error bars show $\pm 2\times$ the Monte Carlo standard error. Token-efficiency ratio intervals are the point estimate $\pm 2$ standard deviations of the parametric-bootstrap ratio distribution. Regression coefficients and fitted curves use $2\sigma$-equivalent percentile intervals from a cluster bootstrap over problems. Their endpoints are the $\Phi(-2)$ and $\Phi(2)$ bootstrap quantiles, matching the nominal coverage of a normal $2\sigma$ interval. The trial Monte Carlo (Sections~\ref{sec:wmv_results} and~\ref{sec:token_efficiency}) draws $M = 500$ replicates per (method, benchmark, budget) cell. The parametric ratio bootstrap (Section~\ref{sec:token_efficiency}) and the cluster bootstrap over problems (Section~\ref{sec:what_drives}) each use $1{,}000$ resamples.

\subsection{\texorpdfstring{$\overline{\mathrm{AUROC}}$}{AUROC} for the Correctness-Predictor Evaluation (Section~\ref{sec:signal})}
\label{app:auroc_eval}

\paragraph{Definition of \texorpdfstring{$\overline{\mathrm{AUROC}}$}{AUROC}.}
AUROC denotes the Area Under the Receiver Operating Characteristic curve. Let $\mathcal{Q}' := \{q \in \mathcal{Q} : 0 < \mathrm{Pass@1}_q < 1\}$ be the subset of problems on which $\mathrm{AUROC}_q$ and the per-problem rates $r_{C,q}, r_{W,q}$ are defined, and set $s_i = c_i(a_i)$ for prefix consistency and $s_i = s(y_i, \ell_i)$ for the baselines (Appendix~\ref{app:baseline_impl}). We score each signal by a per-problem AUROC and macro-average over $\mathcal{Q}'$:
\begin{equation}
\label{eq:auroc_per_problem}
\begin{aligned}
\mathrm{AUROC}_q &:= \Pr\!\left[s_i > s_j \,\middle|\, a_i = a^\star_q,\, a_j \neq a^\star_q\right] \\
&\quad + \tfrac12\Pr\!\left[s_i = s_j \,\middle|\, a_i = a^\star_q,\, a_j \neq a^\star_q\right], \\
\overline{\mathrm{AUROC}} &:= \frac{1}{|\mathcal{Q}'|} \sum_{q \in \mathcal{Q}'} \mathrm{AUROC}_q.
\end{aligned}
\end{equation}
where $i, j$ are independent draws from the $N$ initial samples on problem $q$, and the reported $r_C, r_W, D$ are means of $r_{C,q}, r_{W,q}, D_q$ over $\mathcal{Q}'$.

\paragraph{AUROC computation.}
We compute Eq.~\eqref{eq:auroc_per_problem} as the trapezoidal area under the empirical ROC on each problem's $N$ initial samples, with ties contributing $1/2$. Only problems containing both correct and wrong samples enter $\mathcal Q'$. Tables~\ref{tab:auroc} and~\ref{tab:signal} average AUROC and reproduction rates over this same subset, weighting problems equally rather than by their numbers of cross-class pairs.

For prefix consistency at $K=1$, the binary score evaluated at the initial answer, $c_i(a_i) \in \{1/2, 1\}$, has only one non-trivial operating point on the false-positive/true-positive rate plane, $(\mathrm{FPR}, \mathrm{TPR}) = (r_{W,q}, r_{C,q})$, so the trapezoidal area gives $\mathrm{AUROC}_q = (1 + D_q) / 2$, with $D_q$ expressed as a probability difference rather than percentage points. Tables~\ref{tab:signal} and~\ref{tab:auroc} therefore encode the same information for prefix consistency, and the $\overline{\mathrm{AUROC}}$ ordering across benchmarks matches the $D$ ordering.

\subsection{Cost-Accuracy Evaluation (Sections~\ref{sec:wmv_results} and~\ref{sec:token_efficiency})}
\label{app:cost_accuracy_eval}

\paragraph{Sample-with-replacement trials.}
At each token budget $B$, we draw samples from the pool with replacement and retain the longest prefix of draws whose cumulative cost is at most $B$. The next draw is excluded if it would exceed $B$. Initial samples and PC groups are indivisible sampling units. If no unit fits, the method returns no answer and is scored wrong. Exact vote ties are resolved by first insertion into the vote accumulator, hence by answer encounter order in the sampled sequence. We report mean accuracy over $M$ trials. The trial-mean variance used to derive the CI is $\sigma^2 = \frac{1}{M |\mathcal Q|^2} \sum_q \hat p_q (1 - \hat p_q)$, where $\hat p_q$ is the trial-mean correctness on problem $q$ and $|\mathcal Q|$ is the number of problems.
Here, $\sigma$ is the Monte Carlo standard error of the estimated mean accuracy, and the reported interval is $\pm 2\sigma$. These intervals are conditional on the fixed benchmark questions and generated sample pools: they quantify analysis-side Monte Carlo error, not uncertainty from sampling new questions or generating new pools. This interval is what is shown as ``$\pm$'' in Table~\ref{tab:wmv}.

\paragraph{Dense token-budget grid.}
Token-efficiency ratios in Table~\ref{tab:token_savings} require evaluating the cost--accuracy curve at arbitrary target budgets, so we evaluate every method on the log-uniform grid
\begin{equation*}
    \mathcal{B} = \bigl\{10^{3 + k/100} \;:\; k = 0, 1, \ldots, 400\bigr\}.
\end{equation*}
The dense grid is used for all methods that admit a continuous budget (Standard MV, PC variants, DeepConf, CISC, P(True), Response probability). AC and ESC instead contribute their natural-stopping points (one per swept hyperparameter) as described in Appendix~\ref{app:baseline_impl}.

\paragraph{Standard MV plateau and Pass@1.}
The Standard MV plateau is Standard MV's bootstrap-saturated accuracy on the $N$-sample initial pool, taken as Standard MV's stored accuracy at $\max\mathcal{B} = 10^7$ tokens. We use this finite-budget anchor rather than the unbounded i.i.d.\ asymptote (which is pool-determined but, on slow-converging problems, can sit above what any finite budget reaches) so that the target stays reachable by Standard MV at every $\alpha \in [0, 1]$. Pass@1 is the closed-form expected accuracy of a single uniformly drawn pool sample, computed directly from the pool.

\paragraph{Monotone envelope and log-budget interpolation.}
Each method's grid of (budget, accuracy) points is reduced to its running-max envelope along sorted budget (so accuracy is non-decreasing in budget). This gives the ``min budget at which the method has ever reached this accuracy'' semantics that the token-efficiency ratio is meant to compare. To read off the budget at a target accuracy $\mathrm{acc}_{\mathrm{tgt}}$, we find the first envelope segment $(B^{(j-1)}, \mathrm{acc}^{(j-1)}) \to (B^{(j)}, \mathrm{acc}^{(j)})$ that brackets $\mathrm{acc}_{\mathrm{tgt}}$ and linearly interpolate in (accuracy, log-budget):
\begin{equation*}
    \log B_{\mathrm{tgt}} = \log B^{(j-1)} + \frac{\mathrm{acc}_{\mathrm{tgt}} - \mathrm{acc}^{(j-1)}}{\mathrm{acc}^{(j)} - \mathrm{acc}^{(j-1)}} \cdot (\log B^{(j)} - \log B^{(j-1)}).
\end{equation*}
The same rule applies to AC and ESC, whose curves are the natural-stopping point lists.

\paragraph{Parametric bootstrap for ratio CIs.}
Confidence intervals on $B_{\mathrm{method}} / B_{\mathrm{MV}}$ are the point estimate $\pm 2$ standard deviations of the ratio distribution from a parametric bootstrap with the same trial Monte Carlo noise model used by Table~\ref{tab:wmv}. Each draw perturbs every fixed-budget accuracy entry by independent $\mathcal{N}(0, \sigma_{\mathrm{acc}}^2)$ noise with $\sigma_{\mathrm{acc}} = \mathrm{CI}_{\mathrm{acc}} / 2$ (the Monte Carlo standard error from the stored $\pm 2$ standard-error interval), and additionally perturbs each natural-stopping operating point's cost by $\mathcal{N}(0, \sigma_{\mathrm{cost}}^2)$ noise with $\sigma_{\mathrm{cost}} = \mathrm{CI}_{\mathrm{cost}} / 2$ (the trial-MC standard error of the natural-stopping cost, which varies across operating points unlike the fixed-budget cost). The plateau is perturbed by an analogous draw using its own CI. Each replicate then recomputes the envelope, the target $\mathrm{acc}_{\mathrm{tgt}} = \mathrm{Pass@1} + \alpha \cdot (\mathrm{plateau}' - \mathrm{Pass@1})$ with the perturbed plateau $\mathrm{plateau}'$, and the ratio. Pass@1 is closed-form over the pool (deterministic) and is held fixed across draws. Anchoring the plateau to Standard MV's stored accuracy at $\max\mathcal{B}$ avoids defining that anchor by a maximum over noisy points. It does not remove selection bias in the running-max crossing curves or preserve cross-budget and cross-method dependence. The ratio intervals are therefore approximate conditional Monte Carlo summaries, without validated nominal coverage. Cells where the method's monotone envelope does not reach the target on the point estimate are reported as ``N/A''. The CI subscript is additionally suppressed (point estimate shown alone) when fewer than $50\%$ of bootstrap draws reach the target.

\subsection{Reproduction-Rate GLM (Section~\ref{sec:what_drives})}

\paragraph{Logistic GLM and cluster bootstrap.}
The slopes $\beta(r_C)$ and $\beta(r_W)$ in Section~\ref{sec:what_drives} and Table~\ref{tab:glm_beta} are the Pass@1 coefficients of a per-(model, category) logistic GLM fit on the expanded per-trial Bernoulli outcomes: $\mathrm{logit}\,r_C$ (and separately $\mathrm{logit}\,r_W$) is fit as an affine function of Pass@1, where each (correct initial sample, regeneration) pair contributes one trial to the $r_C$ fit, each (wrong initial sample, regeneration) pair contributes one trial to the $r_W$ fit, and every trial carries the source problem's Pass@1 as its covariate. The fit uses statsmodels' \texttt{GLM} with a binomial family. CIs and $p$-values come from a joint cluster bootstrap over whole problems: each resample draws problems with replacement (preserving each problem's full set of trials) and refits the GLM. We report $2\sigma$-equivalent percentile intervals for coefficients and pointwise percentile bands for fitted curves, and two-sided $p$-values from the empirical sign distribution.

\section{Model Output Postprocessing}
\label{app:output_postprocessing}

We extract final answers from \texttt{\textbackslash boxed\{\}}, normalize them, and then assess answer equivalence. Because AIME~2025 answers are integers, we use exact matching. For HMMT Feb~2026, Brumo~2025, and FrontierScience-Olympiad, the same answer can be expressed in different forms, so we use an LLM judge to assess equivalence. Table~\ref{tab:datasets} in Appendix~\ref{app:eval_setup} summarizes the answer formats and scoring rules. For verbal confidence, we elicit an additional response and parse the model's self-reported score on a 0--100 or binary scale. For P(True), we instead use token probabilities from the binary confidence response.

\subsection{Answer and Verbal Confidence Extraction}
\label{app:answer_confidence_extraction}

\paragraph{Boxed-answer parses.}
Failed boxed responses receive no vote in any aggregator (Standard MV, PC, DeepConf, CISC). Excluding them reduces the usable pool and can change its answer distribution if failures correlate with answer identity or correctness.
Across all 20 cells, the pooled extraction-failure rates are $0.25\%$ for initial generations and $0.42\%$ for regenerated continuations.
Seventeen of the twenty cells have failure rates below $1\%$ for both initial generations and continuations. The exceptions are the three Nemotron3-30B cells on FrontierScience-Olympiad, AIME~2025, and Brumo~2025. Only Brumo also exceeds $1\%$ for initial generations (Table~\ref{tab:format_failure}).
On cells with multiple truncation fractions, boxed-extraction failure rates vary by at most $0.08$ percentage points across $\tau \in \{0.25, 0.50, 0.75\}$.
We treat boxed parse failures as missing votes rather than as a separate signal because the rate is small enough that it does not move the relative comparisons reported in the paper.

\begin{table}[!htbp]
\centering
\caption{\textbf{Answer- and confidence-extraction failure rates (\%).} Generations and Continuations report failures to extract either a boxed answer or a usable numeric fallback from initial generations and regenerated continuations, respectively ($\tau{=}0.75$, $K{=}1$). Rates are pooled over problems and samples. Verbal columns report missing confidence values from secondary calls: 0--100 is the CISC score, Binary the 0/1 correctness verdict, and P(True) the log-probability of the binary call's ``1'' token. Binary and P(True) share the same call.}
\label{tab:format_failure}
\footnotesize
\setlength{\tabcolsep}{4pt}
\begin{tabular}{ll r rr rrr}
\toprule
& & & \multicolumn{2}{c}{Boxed extraction (\%)} & \multicolumn{3}{c}{Verbal queries (\%)} \\
\cmidrule(lr){4-5} \cmidrule(lr){6-8}
Model & Dataset & $N$ & Generations & Continuations & 0--100 & Binary & P(True) \\
\midrule
\multirow{4}{*}{GPT-OSS-120B} & FrontierScience-Olympiad & 128 & 0.08 & 0.05 & 42.99 & 12.04 & 0.00 \\
 & HMMT Feb~2026 & 128 & 0.00 & 0.02 & 47.61 & 25.14 & 0.00 \\
 & AIME~2025 & 128 & 0.00 & 0.00 & 48.07 & 21.59 & 0.00 \\
 & Brumo~2025 & 128 & 0.00 & 0.00 & 44.87 & 19.92 & 0.00 \\
\midrule
\multirow{4}{*}{GPT-OSS-20B} & FrontierScience-Olympiad & 128 & 0.31 & 0.35 & 19.38 & 8.21 & 0.00 \\
 & HMMT Feb~2026 & 128 & 0.00 & 0.05 & 26.07 & 24.76 & 0.00 \\
 & AIME~2025 & 128 & 0.03 & 0.03 & 55.65 & 30.60 & 0.00 \\
 & Brumo~2025 & 128 & 0.00 & 0.00 & 51.67 & 29.19 & 0.00 \\
\midrule
\multirow{4}{*}{Nemotron3-30B} & FrontierScience-Olympiad & 128 & 0.17 & 1.62 & 48.77 & 18.92 & 0.00 \\
 & HMMT Feb~2026 & 128 & 0.54 & 0.59 & 44.06 & 18.73 & 0.00 \\
 & AIME~2025 & 128 & 0.00 & 1.17 & 47.53 & 20.34 & 0.00 \\
 & Brumo~2025 & 128 & 4.14 & 1.25 & 42.57 & 18.90 & 0.00 \\
\midrule
\multirow{4}{*}{Nemotron2-9B} & FrontierScience-Olympiad & 128 & 0.03 & 0.38 & 0.00 & 0.03 & 0.00 \\
 & HMMT Feb~2026 & 128 & 0.00 & 0.33 & 0.00 & 0.24 & 0.00 \\
 & AIME~2025 & 128 & 0.00 & 0.21 & 0.00 & 0.00 & 0.00 \\
 & Brumo~2025 & 128 & 0.16 & 0.08 & 0.16 & 0.23 & 0.16 \\
\midrule
\multirow{4}{*}{Ministral3-14B} & FrontierScience-Olympiad & 64 & 0.28 & 0.27 & 0.02 & 7.41 & 0.00 \\
 & HMMT Feb~2026 & 64 & 0.05 & 0.05 & 0.28 & 6.01 & 0.00 \\
 & AIME~2025 & 64 & 0.00 & 0.00 & 0.21 & 10.10 & 0.00 \\
 & Brumo~2025 & 64 & 0.00 & 0.05 & 0.05 & 6.98 & 0.00 \\
\bottomrule
\end{tabular}
\end{table}

\paragraph{Verbal score handling.}
Verbal 0--100 scans the full secondary completion before its first ``)'' for the first integer and clips it to $[0,100]$; extraction fails if no integer is found. Verbal binary parses the first ``0'' or ``1'' before the first ``)'' in the final portion when a CoT-to-final delimiter is present, otherwise in the full response; it fails if neither digit is found.
P(True) selects the first post-delimiter position whose top-$20$ candidates contain ``0'' or ``1'', falling back to the first generated position. It normalizes the probabilities of ``1'' and ``0'' when both are present, with single-side fallbacks.
Pooled across all 20 cells, confidence-extraction failure rates are $27.31\%$ for Verbal 0--100, $12.82\%$ for Verbal binary, and $0.01\%$ for P(True).
Verbal failure rates show wide model spread, under $1\%$ for Nemotron2-9B and Ministral3-14B versus $19\%$--$56\%$ for GPT-OSS and Nemotron3-30B on Verbal 0--100 (Table~\ref{tab:format_failure}).
The AUROC table (Table~\ref{tab:auroc}) drops missing scores, while the WMV pipeline replaces them by fixed values of $50$ for Verbal 0--100 and $0.5$ for Verbal binary and P(True) (Appendix~\ref{app:baseline_impl}).
Retaining these samples preserves their observed answer frequencies and parse-error rate during resampling; dropping them could change the answer distribution if failures correlate with the boxed answer.
Table~\ref{tab:auroc_verbal_sensitivity} separately compares dropping missing scores with per-problem median imputation. This sensitivity analysis does not reproduce WMV's fixed-default policy. Its gap is at most $0.031$ on every cell, and neither sensitivity mode changes the best signal on any cell.
\begin{table}[!htbp]
\centering
\caption{\textbf{Verbal-confidence $\overline{\mathrm{AUROC}}$ sensitivity to failure handling.} For Verbal 0--100 and Verbal binary, \emph{drop} excludes samples with unparseable scores, as in Table~\ref{tab:auroc}. \emph{Imputed} replaces missing scores with the per-problem median of non-missing scores. This sensitivity variant differs from WMV's fixed defaults (Appendix~\ref{app:baseline_impl}). Rows labeled ``fail'' give missing-score percentages among samples with usable extracted answers.}
\label{tab:auroc_verbal_sensitivity}
\footnotesize
\setlength{\tabcolsep}{1.5pt}
\resizebox{\textwidth}{!}{%
\begin{tabular}{l cccc cccc cccc cccc cccc}
\toprule
& \multicolumn{4}{c}{GPT-OSS-120B} & \multicolumn{4}{c}{GPT-OSS-20B} & \multicolumn{4}{c}{Nemotron3-30B} & \multicolumn{4}{c}{Nemotron2-9B} & \multicolumn{4}{c}{Ministral3-14B} \\
\cmidrule(lr){2-5} \cmidrule(lr){6-9} \cmidrule(lr){10-13} \cmidrule(lr){14-17} \cmidrule(lr){18-21}
Signal / mode & FSci & HMMT & AIME & Brumo & FSci & HMMT & AIME & Brumo & FSci & HMMT & AIME & Brumo & FSci & HMMT & AIME & Brumo & FSci & HMMT & AIME & Brumo \\
\midrule
\multicolumn{21}{l}{\emph{Verbal 0--100}} \\
\quad drop & .516 & .552 & .505 & .436 & .516 & .574 & .561 & .511 & .524 & .489 & .606 & .532 & .571 & .563 & .584 & .548 & .505 & .525 & .525 & .450 \\
\quad imputed & .510 & .536 & .501 & .448 & .514 & .559 & .550 & .506 & .520 & .491 & .579 & .534 & .571 & .563 & .584 & .548 & .505 & .525 & .525 & .450 \\
\quad fail (\%) & 43.0 & 47.6 & 48.1 & 44.9 & 19.3 & 26.1 & 55.6 & 51.7 & 48.7 & 44.1 & 47.5 & 42.4 & 0.0 & 0.0 & 0.0 & 0.0 & 0.0 & 0.3 & 0.2 & 0.1 \\
\midrule
\multicolumn{21}{l}{\emph{Verbal binary}} \\
\quad drop & .530 & .591 & .583 & .554 & .542 & .573 & .561 & .526 & .515 & .524 & .536 & .477 & .496 & .504 & .505 & .504 & .470 & .457 & .474 & .495 \\
\quad imputed & .527 & .560 & .555 & .529 & .542 & .564 & .549 & .523 & .515 & .520 & .544 & .478 & .496 & .504 & .505 & .504 & .475 & .458 & .471 & .496 \\
\quad fail (\%) & 12.0 & 25.1 & 21.6 & 19.9 & 8.2 & 24.8 & 30.6 & 29.2 & 18.9 & 18.7 & 20.3 & 18.3 & 0.0 & 0.2 & 0.0 & 0.1 & 7.4 & 6.0 & 10.1 & 7.0 \\
\bottomrule
\end{tabular}%
}%
\vspace{2pt}
\par\raggedright\footnotesize Abbreviations: FSci = FrontierScience-Olympiad, HMMT = HMMT Feb~2026, AIME = AIME~2025, Brumo = Brumo~2025.
\end{table}

\paragraph{Verbal-query token costs.}
In addition to the initial trace length, the evaluation charges the stored minimum-extraction count. For Verbal 0--100, a recorded count runs through the first closing ``)'' and can be missing even when a score parses if that delimiter is absent. For the shared binary call, it runs through the confidence-token position selected above and may remain nonzero when text parsing fails. Missing counts are treated as zero, so failed Verbal 0--100 extraction incurs no secondary-call overhead in the reported budget. Output after the extraction point is not charged; this accounting favors the verbal baselines rather than measuring their full incurred generation cost.

\paragraph{Verbal 0--100 parser fragility.}

The dominant failure pattern is a parenthetical or commentary token landing before the digit, for example ``\texttt{score) 70}'', ``\texttt{analysis) 87 (commentary?) (final) 93}'', ``\texttt{rating)\textbackslash n\textbackslash n**70**.}'', ``\texttt{(C) 80\%}'', and ``\texttt{value).\textbackslash n(If you think...)}''. A smaller share of failures is completions with no digit at all, such as ``\texttt{score). That is this answer: **safe**.}''.
The model that does not fail at all, Nemotron2-9B, has a median successful response length of just $5$ tokens, indicating that it emits ``\texttt{<digit>)}'' and stops. The failing models have median successful lengths of $30$--$200$ tokens, meaning even their successful completions wander past the digit and leave many opportunities for parenthetical content to land before any number.
These failures illustrate the difficulty of applying a fixed confidence-response template to reasoning models that interleave thinking, commentary, and LaTeX. More robust elicitation and parsing in this setting remain future work.

\FloatBarrier

\subsection{LLM-Based Answer Equivalence Judging}
\label{app:answer_equivalence_judging}

\paragraph{Judge prompts and settings.}
The FrontierScience-Olympiad prompt follows the official grader of the benchmark~\citep[Appendix~B]{openai2025frontierscience}. The prompt for HMMT Feb~2026 and Brumo~2025 follows MathArena~\citep{balunovic2025matharena}. Both prompts require a ``YES'' or ``NO'' verdict on whether two answers are equivalent. We use the evaluated model as its own judge, with the same sampling parameters and chat template as for generation.

Using the evaluated model as the judge is a known limitation. A stronger external grader could change absolute accuracies. Appendix~\ref{app:judge_robustness} examines whether our relative comparisons are robust to re-scoring the same generation pool with Claude Sonnet 4.6.

The full prompts used are:

\paragraph{FrontierScience-Olympiad.}
\begin{list}{}{\setlength{\leftmargin}{1em}\setlength{\rightmargin}{1em}\setlength{\topsep}{3pt}\setlength{\parsep}{0pt}}\item\relax
\small\ttfamily
You are grading an attempted answer to a science olympiad problem. You will be given the attempted answer and the reference answer. Evaluate strictly, but fairly. The reference answer is either a single number or expression in latex formatting, a chemical formula, a compound name, or a phrase referring to a specific name, entity, or method. Mark the attempted answer as correct if it fully matches the reference answer or is otherwise equivalent (e.g., an equivalent algebraic expression, a numerical number within 1 decimal place rounding of the reference answer (e.g., 6.69$\approx$6.7), an equivalent name for a compound/formula, equivalent when accounting for units, etc.). Mark it as incorrect if it is not equivalent to the reference answer.\\[0.3em]
Attempted answer: \{answer1\}\\
Reference answer: \{answer2\}\\[0.3em]
Answer only ``YES'' if correct or ``NO'' if incorrect.
\end{list}

\paragraph{HMMT Feb~2026, Brumo~2025.}
\begin{list}{}{\setlength{\leftmargin}{1em}\setlength{\rightmargin}{1em}\setlength{\topsep}{3pt}\setlength{\parsep}{0pt}}\item\relax
\small\ttfamily
Determine whether two mathematical answers are numerically identical.\\[0.3em]
Answer 1: \{answer1\}\\
Answer 2: \{answer2\}\\[0.3em]
Criteria:
\begin{itemize}[topsep=2pt,itemsep=0pt,parsep=0pt,partopsep=0pt]
\item If they represent exactly the same numerical value, respond ``YES''
\item If they represent different values or one is not numerical, respond ``NO''
\item Different formats (fractions, decimals, radicals, exponential notation) are acceptable if numerically equivalent
\end{itemize}
Examples: ``1/2'' and ``0.5'' $\to$ YES; ``$\sqrt{4}$'' and ``2'' $\to$ YES; ``$2^3$'' and ``8'' $\to$ YES; ``3.14'' and ``$\pi$'' $\to$ NO; ``x+1'' and ``1+x'' $\to$ YES; ``$x^2$'' and ``2x'' $\to$ NO.\\[0.3em]
Answer only ``YES'' or ``NO''.
\end{list}

\section{Reproducibility Statement}
\label{app:reprod}

\paragraph{Code and data.}
\makeatletter
\if@anonymous
  The analysis pipeline, reproduction scripts, and the answer pool used to reproduce all reported numbers are attached as anonymized supplemental material.
\else
  The analysis pipeline and reproduction scripts are released at \url{\codeurl}, and the answer pool used to reproduce all reported numbers is released at \url{\dataurl}.
  Raw model outputs, including sampled text and log-probabilities of the sampled token and the top 20 candidate tokens at each generation step, are also available. See the code repository above for access details.
\fi
\makeatother
Models and datasets are publicly available (Table~\ref{tab:links}).
Initial generation is stochastic (no vLLM seed), but aggregation and voting use seed $42$, so every reported number is bitwise reproducible from the pool.
\makeatletter
\if@anonymous
  All figures and tables in this paper, except for the schematic and conceptual figures, can be regenerated from the supplemental code and data.
\fi
\makeatother

\begin{table}[htbp]
    \centering
    \caption{\textbf{Models and datasets used in the experiments.} License names follow the original release pages. Please verify before redistribution.}\label{tab:links}
    \scriptsize
    \begin{tabularx}{\textwidth}{>{\raggedright\arraybackslash}p{0.27\textwidth}X>{\raggedright\arraybackslash}p{0.14\textwidth}}
    \toprule
        Name & Reference & License \\
    \midrule
        \multicolumn{3}{l}{\textit{Models}} \\ \addlinespace
        GPT-OSS-120B \citep{openai2025gptoss} & \url{https://huggingface.co/openai/gpt-oss-120b} & Apache 2.0 \\ \addlinespace
        GPT-OSS-20B \citep{openai2025gptoss} & \url{https://huggingface.co/openai/gpt-oss-20b} & Apache 2.0 \\ \addlinespace
        Nemotron3-30B \citep{nvidia2025nemotron3nano} & \url{https://huggingface.co/nvidia/NVIDIA-Nemotron-3-Nano-30B-A3B-BF16} & NVIDIA Open Model License \\ \addlinespace
        Nemotron2-9B \citep{nvidia2025nemotronnano2} & \url{https://huggingface.co/nvidia/NVIDIA-Nemotron-Nano-9B-v2} & NVIDIA Open Model License \\ \addlinespace
        Ministral3-14B \citep{mistral2026ministral3} & \url{https://huggingface.co/mistralai/Ministral-3-14B-Reasoning-2512} & Apache 2.0 \\
    \midrule
        \multicolumn{3}{l}{\textit{Datasets}} \\ \addlinespace
        AIME~2025 \citep{balunovic2025matharena} & \url{https://huggingface.co/datasets/MathArena/aime_2025} & CC BY-NC-SA 4.0 \\ \addlinespace
        HMMT Feb~2026 \citep{balunovic2025matharena} & \url{https://huggingface.co/datasets/MathArena/hmmt_feb_2026} & CC BY-NC-SA 4.0 \\ \addlinespace
        Brumo~2025 \citep{balunovic2025matharena} & \url{https://huggingface.co/datasets/MathArena/brumo_2025} & CC BY-NC-SA 4.0 \\ \addlinespace
        FrontierScience-Olympiad \citep{openai2025frontierscience} & \url{https://huggingface.co/datasets/openai/frontierscience} & Apache 2.0 \\
        \bottomrule
    \end{tabularx}
\end{table}

\paragraph{Inference and sampling.}
Each model is served through a vLLM OpenAI-compatible endpoint on $4{\times}$ NVIDIA A100 80GB GPUs, with context window $131{,}072$ and a maximum output length of $100{,}000$ tokens.
Sampling uses each model's recommended settings (Table~\ref{tab:sampling_settings}).

\begin{table}[htbp]
    \centering
    \caption{\textbf{Sampling settings per model.} Values follow each model's recommended configuration on its release page.}
    \label{tab:sampling_settings}
    \small
    \begin{tabular}{lcccc}
    \toprule
        Model & \texttt{temperature} & \texttt{top\_p} & \texttt{top\_k} & \texttt{reasoning\_effort} \\
    \midrule
        GPT-OSS-120B & 1.0 & 1.0 & 40 & medium \\
        GPT-OSS-20B & 1.0 & 1.0 & 40 & medium \\
        Nemotron3-30B & 1.0 & 1.0 & 20 & -- \\
        Nemotron2-9B & 0.6 & 0.95 & 20 & -- \\
        Ministral3-14B & 0.7 & 0.95 & 20 & -- \\
        \bottomrule
    \end{tabular}
\end{table}
Table~\ref{tab:token_stats} reports per-generation token counts at $\tau{=}0.75$ for all models and at $\tau \in \{0.50, 0.25\}$ for GPT-OSS-20B, used in the $\tau$-sensitivity sweep (Appendix~\ref{app:sensitivity}).

\begin{table}[htbp]
\centering
\caption{\textbf{Average output tokens per generation.} Generations is the original full generation. Continuations are completions from the truncation point onward ($K{=}1$). Verbal queries are secondary completions on each init sample: 0--100 elicits the confidence value, Binary serves Verbal binary and P(True). Values are means of stored actual generated-token counts, excluding missing counts. Verbal-query lengths include output after the extraction point and differ from the minimum-extraction costs charged in the budget evaluation (Appendix~\ref{app:baseline_impl}). $N$ is the number of answers per problem.}
\label{tab:token_stats}
\scriptsize
\setlength{\tabcolsep}{4pt}
\begin{tabular}{ll r r rrr rr}
\toprule
& & & & \multicolumn{3}{c}{Continuations} & \multicolumn{2}{c}{Verbal queries} \\
\cmidrule(lr){5-7} \cmidrule(lr){8-9}
Model & Dataset & $N$ & Generations & $\tau{=}0.75$ & $\tau{=}0.50$ & $\tau{=}0.25$ & 0--100 & Binary \\
\midrule
\multirow{4}{*}{GPT-OSS-120B} & FrontierScience-Olympiad & 128 & 3{,}341 & 945 & -- & -- & 322 & 196 \\
 & HMMT Feb~2026 & 128 & 6{,}925 & 2{,}199 & -- & -- & 246 & 182 \\
 & AIME~2025 & 128 & 5{,}442 & 1{,}668 & -- & -- & 262 & 220 \\
 & Brumo~2025 & 128 & 4{,}782 & 1{,}493 & -- & -- & 246 & 190 \\
\midrule
\multirow{4}{*}{GPT-OSS-20B} & FrontierScience-Olympiad & 128 & 7{,}813 & 2{,}631 & 4{,}426 & 5{,}906 & 230 & 210 \\
 & HMMT Feb~2026 & 128 & 13{,}707 & 4{,}621 & 7{,}992 & 10{,}941 & 295 & 290 \\
 & AIME~2025 & 128 & 11{,}159 & 3{,}598 & 6{,}232 & 8{,}564 & 273 & 388 \\
 & Brumo~2025 & 128 & 9{,}208 & 3{,}056 & 5{,}323 & 7{,}151 & 277 & 308 \\
\midrule
\multirow{4}{*}{Nemotron3-30B} & FrontierScience-Olympiad & 128 & 29{,}064 & 8{,}333 & -- & -- & 1{,}419 & 103 \\
 & HMMT Feb~2026 & 128 & 39{,}816 & 12{,}618 & -- & -- & 308 & 149 \\
 & AIME~2025 & 128 & 28{,}941 & 7{,}313 & -- & -- & 281 & 119 \\
 & Brumo~2025 & 128 & 22{,}243 & 5{,}859 & -- & -- & 1{,}158 & 274 \\
\midrule
\multirow{4}{*}{Nemotron2-9B} & FrontierScience-Olympiad & 128 & 10{,}124 & 3{,}016 & -- & -- & 13 & 3 \\
 & HMMT Feb~2026 & 128 & 14{,}216 & 4{,}172 & -- & -- & 8 & 4 \\
 & AIME~2025 & 128 & 11{,}715 & 3{,}498 & -- & -- & 10 & 3 \\
 & Brumo~2025 & 128 & 9{,}634 & 3{,}125 & -- & -- & 10 & 5 \\
\midrule
\multirow{4}{*}{Ministral3-14B} & FrontierScience-Olympiad & 64 & 9{,}671 & 5{,}834 & -- & -- & 1{,}940 & 1{,}778 \\
 & HMMT Feb~2026 & 64 & 8{,}388 & 4{,}560 & -- & -- & 2{,}161 & 1{,}768 \\
 & AIME~2025 & 64 & 7{,}590 & 3{,}735 & -- & -- & 1{,}735 & 1{,}451 \\
 & Brumo~2025 & 64 & 7{,}739 & 3{,}994 & -- & -- & 1{,}802 & 1{,}277 \\
\bottomrule
\end{tabular}
\end{table}

\paragraph{Generation prompts.}
For the three math benchmarks (HMMT Feb~2026, AIME~2025, Brumo~2025) the system prompt is \textit{``You are a helpful assistant specialized in solving mathematical problems.''} For FrontierScience-Olympiad it is \textit{``You are an expert scientist solving olympiad-level problems in physics, chemistry, and biology.''} Both append \textit{``Please reason step by step, and put your final answer within \texttt{\textbackslash boxed\{\}}.''} to each problem statement. Initial samples and the regenerated continuations from each truncation point share the same prompt.

\end{document}